\documentclass[11pt]{article}
\IfFormatAtLeastTF{2026-06-01}{%
\RemoveFromHook{package/amsthm/after}[firstaid/aliascounter]%
}{}

\usepackage[utf8]{inputenc} 
\usepackage[T1]{fontenc}    

\usepackage{url}            
\usepackage[authoryear,round]{natbib} 
\usepackage{booktabs}       
\usepackage{pifont}         
\usepackage{amsfonts}       
\usepackage{nicefrac}       
\usepackage{microtype}      
\usepackage{xcolor}         
\usepackage{amsmath, amsthm, amssymb} 
\usepackage{mathtools}
\usepackage{pc2report} 
\usepackage{algorithm}
\usepackage{algpseudocode}
\makeatletter
\let\pc@float@makebox\float@makebox
\renewcommand{\float@makebox}[1]{%
  \color@vbox\normalcolor\pc@float@makebox{#1}\color@endbox}
\makeatother
\usepackage{bm}
\usepackage{graphicx,placeins}
\usepackage[font=small]{caption}
\usepackage{needspace}
\usepackage{wrapfig,lipsum}
\usepackage{thmtools, thm-restate}
\usepackage{mathrsfs} 
\usepackage{bbold}

\usepackage[varqu,varl,scaled=0.95]{zi4} 
\usepackage{array}
\usepackage{tabularx}
\newcommand{\pcVisualFont}{\PCVisualFont}
\newcommand{\pcTableStyle}{\PCTableStyle}
\usepackage{enumitem}
\usepackage{tikz}
\usetikzlibrary{arrows.meta,positioning,fit,calc,patterns.meta}
\usepackage{fancyvrb}
\usepackage[most]{tcolorbox}
\usepackage{etoc}

\usepackage{algorithm}
\usepackage{algpseudocode}

\usepackage{cancel}

\definecolor{dgreen}{rgb}{0.00,0.49,0.00}
\definecolor{dblue}{rgb}{0,0.08,0.75}

\RequirePackage[colorlinks,citecolor=dgreen,urlcolor=dblue,linkcolor=dblue]{hyperref}

\renewcommand{\paragraph}[1]{{\bfseries #1}.}

\makeatletter
\def\@endtheorem{\endtrivlist}
\makeatother

\declaretheorem[name=Theorem,refname=Thm.]{theorem}
\declaretheorem[name=Lemma]{lemma}
\declaretheorem[name=Proposition,refname=Prop.]{proposition}

\declaretheorem[name=Corollary,refname=Cor.,sibling=theorem]{corollary}
\declaretheorem[name=Definition,refname=Def.]{definition}

\declaretheorem[name=Assumption]{assumption}

\declaretheorem[name=Example]{example}

\usepackage[capitalize]{cleveref}

\crefname{assumption}{}{}
\Crefname{assumption}{}{}
\crefname{equation}{}{}
\Crefname{equation}{Eq.}{Equations}
\crefname{figure}{Fig.}{Figs.}
\crefname{table}{Tab.}{Tabs.}
\crefname{section}{Sec.}{Sec.}
\crefname{theorem}{Thm.}{Thm.}
\crefname{lemma}{Lemma}{Lemmas}
\crefname{corollary}{Cor.}{Cor.}
\crefname{example}{Example}{Examples}
\crefname{remark}{Remark}{Remarks}
\crefname{algorithm}{Alg.}{Algorithms}
\crefname{appendix}{Appendix}{Appendices}
\crefname{subappendix}{Appendix}{Appendices}
\crefname{subsubappendix}{Appendix}{Appendices}

\newcommand{\Obl}{\mathrm{Obl}}

\newcommand{\Safe}{\mathrm{Safe}}

\newcommand{\tool}{{\mathsf{T}}}
\newcommand{\view}{{\mathcal{V}}}
\newcommand{\tran}{{\mathsf{D}}}
\renewcommand{\ss}{\Sigma}
\newcommand{\toolctrl}{\texttt{ctrl}}

\newcommand{\obs}{{\mathsf{View}}}
\newcommand{\exec}{\texttt{exec}}

\newcommand{\terminal}{\texttt{stop}}
\newcommand{\true}{\texttt{true}}
\newcommand{\false}{\texttt{false}}
\newcommand{\settle}{\mathsf{settle}}
\newcommand{\goal}{\mathsf{Goal}}
\newcommand{\pre}{\mathsf{Pre}}

\newcommand{\contract}{{\mathcal{K}}}
\newcommand{\clause}{C}

\DefineVerbatimEnvironment{code}{Verbatim}{fontsize=\small,xleftmargin=1.5em}

\definecolor{callc}{HTML}{1F4E9A}
\definecolor{proofbg}{HTML}{E1EAF8}
\definecolor{respc}{HTML}{2E6B30}
\definecolor{rejc}{HTML}{A4262C}
\definecolor{thinkc}{gray}{0.32}
\definecolor{injc}{HTML}{FFE0A6}
\tcbset{chatbox/.style={enhanced, boxrule=0.3pt, leftrule=2pt, arc=1.5pt, left=3pt, right=3pt, top=1pt, bottom=1pt, before skip=1.2pt, after skip=1.2pt, fontupper=\scriptsize\raggedright}}
\newtcolorbox{callbox}{chatbox, bicolor, colback=callc!5, colbacklower=proofbg,
  colframe=callc!80, collower=callc, fontlower=\scriptsize\raggedright,
  middle=1pt}
\newtcolorbox{respbox}{chatbox, colback=respc!6, colframe=respc!80}
\newtcolorbox{rejbox}{chatbox, colback=rejc!5, colframe=rejc!85}
\newlength{\chatindent}
\newcommand{\chatlabel}[2]{\makebox[\chatindent][l]{\sffamily\bfseries\color{#1}#2}}
\newcommand{\chatline}[3]{\par\noindent\hangindent=\chatindent\hangafter=1\chatlabel{#1}{#2}#3\par}

\newcommand{\think}[1]{\par\vspace{3pt}{\scriptsize\upshape\raggedright\color{thinkc}\noindent\textbf{Agent:}\enspace#1\par}\vspace{1pt}}
\newcommand{\elide}[1]{\par\vspace{1pt}{\centering\scriptsize\color{thinkc}$\vdots$\enspace#1\enspace$\vdots$\par}\vspace{1pt}}
\newcommand{\toolcall}[2]{\begin{callbox}\chatline{callc}{tool call}{\texttt{#1}}\tcblower\chatline{callc}{proof}{\texttt{#2}}\end{callbox}}

\newcommand{\toolresp}[1]{\begin{respbox}\chatline{respc}{response}{#1}\end{respbox}}
\newcommand{\toolrej}[1]{\begin{rejbox}\chatline{rejc}{rejected}{#1}\end{rejbox}}
\newcommand{\injected}[1]{\colorbox{injc}{\parbox{\dimexpr\linewidth-\chatindent-2\fboxsep\relax}{\raggedright #1}}}

\newcommand*{\PCProjectURL}{https://consequential-ai.github.io/proof-carrying-contracts/}
\newcommand*{\PCRepoURL}{https://github.com/consequential-ai/proof-carrying-contracts}
\newcommand{\PCProject}[1][Project page]{\href{\PCProjectURL}{#1}}
\newcommand{\PCRepo}[1][GitHub repository]{\href{\PCRepoURL}{#1}}
\newcommand{\PCRepoAddress}{\expandafter\url\expandafter{\PCRepoURL}}

\renewcommand{\goal}{\mathit{Goal}}
\renewcommand{\clause}{\mathsf{C}}
\renewcommand{\Safe}{\mathsf{Safe}}
\newcommand{\worlds}{\mathcal{W}}
\newcommand{\tools}{\mathcal{T}}
\newcommand{\inputs}{I}

\newcommand{\effect}{E}
\newcommand{\corresponds}{\rho}
\newcommand{\enabled}{\Xi}

\title{Proof-Carrying Contracts:\\Commitments AI Agents Cannot Break}

\author{Carlo Ciliberto\textsuperscript{1,2}\enspace\textperiodcentered\enspace
Alessandro Rudi\textsuperscript{3}\enspace\textperiodcentered\enspace
Enrico Tassi\textsuperscript{4}}
\PCReportTitle{Proof-Carrying Contracts}{Commitments AI Agents Cannot Break}
\PCReportAffiliations{%
\textsuperscript{1}AI Centre, Computer Science Department, University College London, London, UK.\enspace
\textsuperscript{2}Praxis Lab, The Italian Institute of Artificial Intelligence (AI4I), Turin, Italy.\enspace
\textsuperscript{3}SDA Bocconi School of Management. Milano, Italy.\enspace
\textsuperscript{4}Inria d'Université Côte d'Azur, France.}
\PCReportAuthorNote{The authors are in alphabetical order.\hfill
\PCProject{}\enspace\textperiodcentered\enspace\PCRepo[Code and proofs]}
\PCReportContacts{Correspondence: \texttt{c.ciliberto@ucl.ac.uk},
\texttt{alessandro.rudi@sdabocconi.it}, \texttt{enrico.tassi@inria.fr}}
\date{}

\definecolor{asInk}{HTML}{26323C}
\definecolor{asMuted}{HTML}{56636E}
\definecolor{asProved}{HTML}{287354}
\definecolor{asAssigned}{HTML}{276784}
\definecolor{asMitigated}{HTML}{A86A13}
\definecolor{asAccepted}{HTML}{68727A}
\newcommand{\pcYes}{\textcolor{asProved}{\ding{51}}}
\newcommand{\pcNo}{\textcolor{pcDangerText}{\ding{55}}}
\tikzset{
  pcProduct/.style={draw=asInk!35,fill=asInk!1,rounded corners=7pt},
  pcExtension/.style={pcProduct},
  asLegend/.style args={#1}{draw=#1!65,fill=#1!9,text=#1,
    rounded corners=2pt,align=center,inner xsep=3pt,inner ysep=1.2pt,
    font=\sffamily\fontsize{7}{8}\selectfont}}

\newcommand{\assuranceTree}{%
\begin{tikzpicture}[x=1cm,y=1cm,text=asInk,
 font=\pcVisualFont,
 execute at begin node={\setlength{\parskip}{0pt}\hyphenpenalty=10000},
 line width=.65pt,line cap=round,line join=round,
 croot/.style={draw=asInk!60,fill=white,rounded corners=3pt,align=center,text width=8.3cm,minimum height=1.05cm,inner sep=4pt},
 cassumption/.style={draw=asInk!50,fill=white,rounded corners=3pt,align=center,text width=2.16cm,minimum height=3.3cm,inner sep=3pt},
 csummary/.style={align=center,text width=2.15cm,inner sep=0pt,font=\sffamily\fontsize{8}{9.5}\selectfont},
 cowner/.style={csummary,text=asMuted},
 chead/.style={draw=asInk!50,dashed,fill=white,rounded corners=3pt,align=center,text width=3.2cm,minimum height=.66cm,inner sep=4pt},
 cgroup/.style={draw=asInk!18,fill=asInk!1,rounded corners=4pt},
 cleaf/.style={draw,rounded corners=3pt,align=center,text width=3.04cm,minimum height=1.95cm,inner sep=4pt},
 cdep/.style={draw=asInk!60,rounded corners=4pt},
 cnote/.style={align=center,text=asMuted,font=\sffamily\fontsize{8}{9.5}\selectfont}]
\node[croot] (csafe) at (7.75,0)
 {\textbf{World safety throughout execution}\quad\cref*{thm:world-final}};
\node[cassumption,draw=asMitigated!65,fill=asMitigated!9] (caone) at (1.25,-3.15) {};
\node at (1.25,-1.82) {\PCAssumptionRef{ass:sound-checking}};
\node[align=center,text width=2.16cm,inner sep=0pt] at (1.25,-2.38) {\textbf{Sound\\checking}};
\node[csummary] at (1.25,-3.42) {Accepted proofs\\establish their\\statements};
\node[cowner] at (1.25,-4.31) {Checker\\provider};
\node[cassumption,draw=asAssigned!65,fill=asAssigned!9] (catwo) at (3.85,-3.15) {};
\node at (3.85,-1.82) {\PCAssumptionRef{ass:transition-coverage}};
\node[align=center,text width=2.16cm,inner sep=0pt] at (3.85,-2.38) {\textbf{Model\\coverage}};
\node[csummary] at (3.85,-3.42) {Actual effects\\covered for\\enabled calls};
\node[cowner] at (3.85,-4.31) {Tool\\providers};
\node[cassumption,draw=asProved!65,fill=asProved!9] (cathree) at (6.45,-3.15) {};
\node at (6.45,-1.82) {\PCAssumptionRef{ass:initial-control}};
\node[align=center,text width=2.16cm,inner sep=0pt] at (6.45,-2.38) {\textbf{Initialization\\and control}};
\node[csummary] at (6.45,-3.42) {Compliant start\\and admissible\\control};
\node[cowner] at (6.45,-4.31) {Contract\\author};
\node[cassumption,draw=asAssigned!65,fill=asAssigned!9] (cafour) at (9.05,-3.15) {};
\node at (9.05,-1.82) {\PCAssumptionRef{ass:execution-integrity}};
\node[align=center,text width=2.16cm,inner sep=0pt] at (9.05,-2.38) {\textbf{Execution\\integrity}};
\node[csummary] at (9.05,-3.42) {Execution obeys\\the specified\\rules};
\node[cowner] at (9.05,-4.31) {Airlock\\operator};
\node[cassumption,draw=asProved!65,fill=asProved!9] (cafive) at (11.65,-3.15) {};
\node at (11.65,-1.82) {\PCAssumptionRef{ass:tool-fidelity}};
\node[align=center,text width=2.16cm,inner sep=0pt] at (11.65,-2.38) {\textbf{Fidelity}};
\node[csummary] at (11.65,-3.42) {Correspondence\\established and\\preserved};
\node[cowner] at (11.65,-4.31) {\textit{Refined below}};
\node[cassumption,draw=asProved!65,fill=asProved!9] (casix) at (14.25,-3.15) {};
\node at (14.25,-1.82) {\PCAssumptionRef{ass:contract-adequacy}};
\node[align=center,text width=2.16cm,inner sep=0pt] at (14.25,-2.38) {\textbf{Adequacy}};
\node[csummary] at (14.25,-3.42) {$\rho(w,s)\land\mathsf C(s)$\\[2pt]$\Longrightarrow\mathsf{Safe}(w)$};
\node[cowner] at (14.25,-4.31) {Contract\\author};
\draw[cdep] (csafe.south) -- (7.75,-1.08);
\draw[cdep] (1.25,-1.08) -- (14.25,-1.08);
\foreach \p/\x in {caone/1.25,catwo/3.85,cathree/6.45,cafour/9.05,cafive/11.65,casix/14.25}
  \draw[cdep] (\x,-1.08) -- (\p.north);
\node[cowner,draw=asAccepted!65,fill=asAccepted!12,rounded corners=3pt,
      inner sep=2pt,minimum height=1.1cm] (cprotection) at (14.25,-5.85)
 {\textbf{\textcolor{asInk}{Scope of protection}}\\$\mathsf{Safe}$\\Contracting parties};
\draw[cdep,dashed] (casix.south) -- (cprotection.north);
\node[draw=asAssigned!65,fill=asAssigned!8,rounded corners=3pt,
      align=center,text width=10.0cm,inner sep=4pt] at (5.55,-5.65)
 {\PCAssumptionRef{ass:isolation}\enspace\textbf{isolation}\quad reference proofs apply to every agent\\[2pt]
 {\color{asMuted}\cref*{ass:transition-coverage} and \cref*{ass:tool-fidelity} via \cref*{lem:isolation}\quad$\cdot$\quad Airlock provider}};
\node[cnote,fill=white,inner sep=3pt] (crefine) at (7.75,-6.65)
 {\textbf{Supporting fidelity}\quad patterns to combine for the chosen $\rho$};
\draw[cdep,dashed] (cafive.south) -- (11.65,-6.30) -- (7.75,-6.30) -- (crefine.north);
\draw[cdep,dashed] (crefine.south) -- (7.75,-7.03);
\draw[cdep,dashed] (1.9,-7.03) -- (13.6,-7.03);
\node[chead] (ccustody) at (1.9,-7.74) {\textbf{Custody}\\[2pt]\cref{ass:custody-register}--\cref{ass:custody-submission}};
\node[chead] (cinterface) at (5.8,-7.74) {\textbf{Formal interfaces}\\[2pt]\cref{ass:interface-responses}--\cref{ass:interface-recording}};
\node[chead] (cexternal) at (9.7,-7.74) {\textbf{External behaviour}\\[2pt]\cref{ass:invariant-preservation}};
\node[chead] (cmeaning) at (13.6,-7.74) {\textbf{Agreed meaning}\\[2pt]\cref{ass:measurement-agreement}};
\foreach \p/\x in {ccustody/1.9,cinterface/5.8,cexternal/9.7,cmeaning/13.6}{
  \draw[cdep,dashed] (\x,-7.03) -- (\p.north);
  \draw[cgroup] (\x-1.8,-8.39) rectangle (\x+1.8,-12.81);
  \draw[cdep] (\p.south) -- (\x,-8.39);
}
\node[cleaf,draw=asAssigned!65,fill=asAssigned!9] at (1.9,-9.49)
 {\textbf{Register grows only\\through $\mathsf{issue}$ \cref{ass:custody-register}}\\[4pt]
 {\fontsize{8}{9.5}\selectfont\color{asMuted}Operator and\\resource holders}};
\node[cleaf,draw=asAssigned!65,fill=asAssigned!9] at (1.9,-11.71)
 {\textbf{Only registered\\instructions are\\submitted \cref{ass:custody-submission}}\\[4pt]
 {\fontsize{8}{9.5}\selectfont\color{asMuted}Routine implementer\\and operator}};
\node[cleaf,draw=asMitigated!65,fill=asMitigated!9] at (5.8,-9.49)
 {\textbf{Reports conform\\to the specification \cref{ass:interface-responses}}\\[4pt]
 {\fontsize{8}{9.5}\selectfont\color{asMuted}Service provider}};
\node[cleaf,draw=asAssigned!65,fill=asAssigned!9] at (5.8,-11.71)
 {\textbf{Authentic reports\\recorded correctly \cref{ass:interface-recording}}\\[4pt]
 {\fontsize{8}{9.5}\selectfont\color{asMuted}Recorder implementer\\and operator}};
\node[cleaf,draw=asAssigned!65,fill=asAssigned!9] at (9.7,-9.49)
 {\textbf{Service commitments\\preserve $\mathsf K$ \cref{ass:invariant-preservation}}\\[4pt]
 {\fontsize{8}{9.5}\selectfont\color{asMuted}Relevant services}};
\node[cleaf,draw=asAccepted!65,fill=asAccepted!12] at (9.7,-11.71)
 {\textbf{Outside interactions\\preserve $\mathsf K$ \cref{ass:invariant-preservation}}\\[4pt]
 {\fontsize{8}{9.5}\selectfont\color{asMuted}Outside parties}};
\node[cleaf,draw=asInk!35,fill=white] at (13.6,-9.49)
 {\textbf{Meaning of fields\\and measurements}\\[4pt]
 {\fontsize{8}{9.5}\selectfont\color{asMuted}Module designer\\and contracting parties}};
\node[cleaf,draw=asAssigned!65,fill=asAssigned!9] at (13.6,-11.71)
 {\textbf{State matches the\\specified reading \cref{ass:measurement-agreement}}\\[4pt]
 {\fontsize{8}{9.5}\selectfont\color{asMuted}Module implementer\\and operator}};
\node[cnote,anchor=east] at (5.55,-13.36) {How supporting claims are addressed:};
\node[asLegend=asProved] at (6.10,-13.36) {proof};
\node[asLegend=asAssigned] at (8.25,-13.36) {assigned responsibility};
\node[asLegend=asMitigated] at (10.80,-13.36) {mitigation};
\node[asLegend=asAccepted] at (13.10,-13.36) {accepted risk};
\end{tikzpicture}%
}

\newcommand*{\PCReleaseVersion}{0.1}
\renewcommand{\PCReportVersion}{v\PCReleaseVersion}
\renewcommand{\PCReportMetadata}{Version \PCReleaseVersion\enspace\textperiodcentered\enspace October 2026}
\PCReportSetup
\colorlet{callc}{pcAgentText}
\colorlet{respc}{pcCheckText}
\colorlet{rejc}{pcDangerText}
\colorlet{asInk}{pcInk}
\colorlet{asMuted}{pcMuted}
\colorlet{asProved}{pcCheckText}
\colorlet{asAssigned}{pcAssigned}
\colorlet{asMitigated}{pcWarning}
\colorlet{asAccepted}{pcAcceptedText}
\begin{document}

\maketitle
\etocdepthtag.toc{mtchapter}

\begin{abstract}
AI agents increasingly act on behalf of people or other agents, purchasing services, exchanging information and delegating work. What an agent must respect is written in instructions, policies or agreements, but text alone cannot make it comply. It may make mistakes, be redirected by what it reads, or pursue different interests. For an economy of such agents to function, parties need agreements they can rely on without inspecting or trusting one another's agents. We introduce Proof-Carrying Contracts (PC2), a framework for making agreed constraints enforceable throughout an agent's interactions with a system. Before a requested operation can run, the agent must supply a machine-checkable proof that it preserves the agreed constraints in every case allowed by its model. An execution system called the \emph{Airlock} checks the proof and controls access to the operation. The parties specify the conditions to preserve, while the agent supplies the reasoning needed to justify its decisions and remains free to adapt its strategy. This avoids having to design a separate permission policy and prove that every decision it permits is safe. We prove that the constraints hold at every instant during execution, even for an adversarial agent, and identify when this implies safety in the world. Reusable capability modules provide tools and guarantees from which contracts can be constructed. PC2 provides a foundation for an agent economy whose agreements protect both those delegating work and those carrying it out.
\end{abstract}

\renewcommand{\PCContentsBreakBefore}{Context}
\PCContents

\newpage

\PCLayer{Overview}
\section{Introduction}
\label{sec:intro}

AI agents are beginning to act on behalf of people and organizations, often referred to as their \emph{principals}. They use tools to send documents, make payments and book services \citep{yao2023react,schick2023toolformer}, and their usefulness grows with the authority they receive: an agent that can only make suggestions leaves the work to its user, while one that can act can carry it out. However, greater authority brings greater risks, so delegation requires agreement on what may be done and within which limits, with assurance that the agent will respect them.

The need for assurance appears in two settings. In the first, a person delegates to their own assistant, which may make mistakes or be redirected by a malicious email or web page \citep{greshake2023indirect,debenedetti2024agentdojo}. A second setting arises when a client hires an agent that another company provides as a service, adding a possible conflict of interests \citep{hammond2025multiagent}. The agent might carry out the client's shopping task while favouring its provider's products over better alternatives, and the client may be unable to inspect how it makes those choices. Both settings share a critical limitation: the agent is asked to follow instructions, policies or the terms of a deal. These are text, and text cannot make it comply. To enforce these limits, one could require a person to approve every action, but this defeats much of the purpose of delegation. Another approach is to train agents to follow human intentions and values, the aim of \emph{AI alignment}, using human feedback or feedback guided by written principles \citep{ouyang2022training,bai2022constitutional}. However, training alone can make compliance more likely without guaranteeing that the agent will respect an agreement.

This problem is not new: throughout history, people have cooperated without knowing one another's intentions. However, what makes cooperation possible in these cases is not knowledge of intentions but legal contracts, together with the institutions that enforce them \citep{north1990institutions}. A person or company that breaks a contract may have to compensate the other party, and the prospect of these consequences gives it a reason to honour the agreement. With an AI agent, however, the agent itself is not generally recognized as a legal person that must answer for a breach, even if the people or companies behind it may remain liable for its actions \citep{uncitral2025automated}. Even if it were, assigning legal obligations would not by itself give the agent a reason to care about breaking them. Since compensation cannot undo every harm, agreements involving agents need enforcement that prevents breaches, independently of the agent's alignment or willingness to comply.

\textbf{Proof-carrying contracts.} We introduce \emph{Proof-Carrying Contracts} (PC2) to make such commitments enforceable. The parties express their agreement in a formal contract: precise statements about what must hold as the task is carried out. Before a tool call can run, the agent must show that it respects the contract by supplying a \emph{proof}, a mathematical argument written in a formal language whose steps a program can check. Each call passes through the \emph{Airlock}, a program that checks the proof and executes the call only if it accepts. The agent cannot bypass this check to act on what the contract protects. This approach builds on \emph{proof-carrying code}, where a program comes with a proof of safety that its recipient checks before running it \citep{necula1997pcc}.

Several existing approaches to agent safety enforce an \emph{authorization policy}, rules specifying which calls may run \citep{debenedetti2025camel,shi2025progent,palumbo2026formal}. To show that such a policy ensures safety, its author must account in advance for every situation in which a permitted action could cause harm. For an autonomous car, a crash is an unsafe state, but identifying all the situations in which accelerating, braking or steering could lead to one is much harder. A policy must turn that analysis into rules that prevent crashes without unnecessarily restricting safe actions. In PC2, the parties state what must hold and provide models describing what the tools can do. The Airlock derives the statement that each call must prove, its \emph{obligation}, from the contract, these models and what the agent can see. For example, a total budget need not impose separate payment limits: the agent chooses amounts and order, proving that each payment, including possible fees, preserves the budget. The agent can thus adapt its strategy after each result, justifying one action at a time without proving its whole plan in advance.

PC2 thus shifts the burden from proving a general policy safe to proving that each proposed action respects the contract. The agent must find this proof, while the Airlock checks it. Finding it may require substantial search even when checking it is straightforward. This is the intuition behind the familiar P versus NP question: checking a proposed solution efficiently does not tell us how to find one efficiently. The agent can use automated proof tools or recent systems for AI theorem proving \citep{ren2025deepseekproverv2,lin2025goedelproverv2}. More demanding tasks may require more reasoning, without weakening the requirement for execution. The agent can revise a rejected proof or choose another action, but no call runs without an accepted proof.

\textbf{Guarantees for both parties.} We prove that within PC2 the contract's constraints hold throughout execution, including while actions are under way, for every agent and every sequence of requests and proofs. The guarantee includes malicious agents because the Airlock checks a proof for each call before executing it, without needing a proof about the agent's internal reasoning. A client and provider can agree on and sign the same formal contract, then entrust its enforcement to an Airlock operated by a third party they both trust. The mechanism can protect both sides of the agreement. For example, payment can be held until the agent proves that the agreed work is complete, then released without further client approval, so that the client cannot withhold an authorized payment and the provider cannot obtain it without establishing completion.

Writing a formal contract that captures the parties' intentions is itself a substantial task, and enforcing its clauses cannot repair an agreement that says the wrong thing \citep{hadfieldmenell2019incomplete}. Many contracts, however, concern recurring operations such as payments, document access and bookings, so each agreement need not be developed from scratch. Reviewed templates can help express familiar requirements, and we introduce \emph{capability modules} to make the supporting tools reusable as well. These modules package tools, models and guarantees, including the conditions required of the services they use. We show how to combine them for different agreements: a payment module, for example, can support different budgets and goals without being redesigned each time. PC2 is intended to let people delegate to agents they cannot inspect while retaining control over the authority they grant, and to let agents acting for different parties exchange services on terms both sides can verify \citep{tomasev2025virtual,jacobs2026economic}. This paper develops the construction for one agent under one contract.

\textbf{Our contributions are:}
\begin{itemize}[leftmargin=*,itemsep=3pt,parsep=0pt,topsep=2pt]
\item PC2, which binds an agent to a contract by requiring a proof for each tool call, with the obligation derived from the contract, tool models and available information (\cref{sec:example,sec:model}).
\item Two guarantees, commitment and completion, proved for every agent under explicit assumptions that do not concern the agent (\cref{sec:guarantees,sec:risks}).
\item The conditions under which these guarantees protect the world, and capability modules that establish these conditions by reuse, with an account of what is proved, what is assumed and who answers for it (\cref{sec:risks,sec:assurance}).
\item A mechanization in Lean~4 (\cref{sec:lean}), with code and proofs available at\\
\PCRepoAddress.
\item The open questions and remaining risks, and a comparison with existing approaches (\cref{sec:residual,sec:related}).
\end{itemize}

\section{PC2 at a glance}
\label{sec:example}

\begin{figure}[t]
\centering
\includegraphics[width=0.9\textwidth]{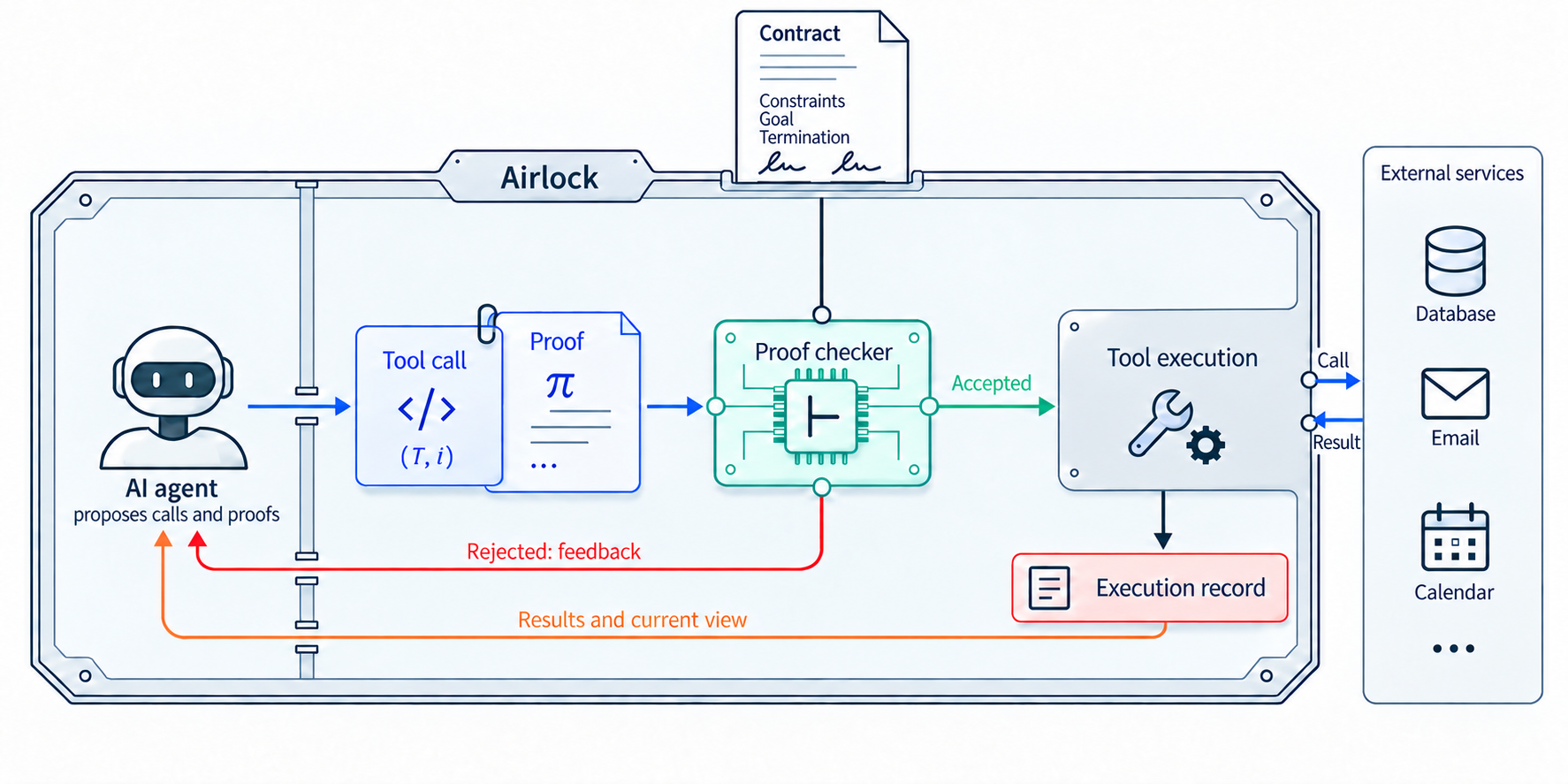}
\caption{How PC2 works. Two parties sign a contract. The agent that does the work runs inside the Airlock and keeps its usual loop of reasoning and tool calls, but each tool call carries a proof that the contract holds throughout the operation. The Airlock checks the proof before the call runs. If the proof is correct, the call is performed and its result returned; if not, nothing runs, and the agent receives the reason, as an error message it can fix. The contract is settled when the agent proves its goal, and terminated otherwise, for instance at expiration, with a refund.}
\label{fig:architecture}
\end{figure}

This section explains PC2 without formulas, then shows it on an example that we use throughout the paper; nothing in PC2 is specific to it. \Cref{sec:model,sec:guarantees} give the definitions and the theorems.

\subsection{How PC2 works}\label{sec:how}

\paragraph{The parties and the contract} PC2 serves the two use cases of \cref{sec:intro}. A principal may bind their own agent; or a \emph{client}, a person or an agent acting for one, may hire a \emph{contractor}, an agent offered by someone else, to do a job. Either way, the parties sign a \emph{contract}. It names what the contract covers, such as a budget, documents or bookings; \emph{clauses} that must hold at every instant; a \emph{goal}, whose proof \emph{settles} the contract, for instance by releasing an \emph{escrow}, money placed with the Airlock until the work is done; and when the contract ends otherwise, for instance at a deadline, with a refund. Some clauses need a judgement that no formula gives, such as whether a hotel suits a family. The parties then appoint an \emph{expert}: a procedure fixed at signing, such as a given model with a given prompt, whose verdict counts as the truth. The contract is written in a precise language, so that a machine can check it; that is the price of enforcement.

\paragraph{The Airlock} The agent that does the work runs inside an \emph{Airlock}: a sandbox whose software is open, and whose running build both parties verify before signing, by remote attestation \citep{amd2020sevsnp,chrapek2025confidential}. The Airlock holds what the contract places in its custody, such as the escrow or a passport, and the agent's only channel to the world is the contract's tools. Anyone may host an Airlock: the parties check what runs, not who runs it.

\paragraph{A tool call with a proof} The agent works as agents do today: it reasons, calls a tool, reads the answer, and reasons again. The one change is that each tool call carries a \emph{proof} that the clauses hold throughout the call, whatever the world answers and whatever the agent cannot see (\cref{fig:architecture}). The Airlock forms the statement to be proved itself, from the contract and from what the agent sees. It checks the proof with a small proof checker and performs the call only if the proof is correct. Otherwise nothing runs, and the agent receives the reason. Checking a proof is easy even when finding it is hard, as checking a solved puzzle is easier than solving it: the hard part is left to the untrusted side, and the trusted side only checks.

\begin{table}[t]
\begin{tcolorbox}[pc contract,unbreakable,title={A travel contract},
  left=9pt,right=9pt,top=7pt,bottom=8pt]
The family's agent hires the travel agency's agent to arrange a three-night trip and obtain a visa.

\begin{tcolorbox}[enhanced,frame hidden,colback=pcContract,arc=1.5pt,
  boxsep=0pt,left=6pt,right=6pt,top=5pt,bottom=5pt,
  before skip=6pt,after skip=8pt,fontupper=\small]
{\sffamily\bfseries Constraints at every instant}\par\smallskip
\begin{tabularx}{\linewidth}{@{}>{\raggedright\arraybackslash}X>{\raggedright\arraybackslash}X>{\raggedright\arraybackslash}X@{}}
\PCContractField{Travel budget}\par At most \textbf{2000 EUR} in commitments, including fees.
& \PCContractField{Passport disclosure}\par Only to the designated consulate.
& \PCContractField{Hotel bookings}\par Recorded rating of at least \mbox{\textbf{7 out of 10}}. \\
\end{tabularx}
\end{tcolorbox}

\noindent\begin{minipage}[t]{0.48\linewidth}
\raggedright
{\sffamily\bfseries Setup and access}\par\smallskip
\PCContractField{Custody.} The Airlock holds the passport and a \textbf{300 EUR} escrow before starting. The escrow is separate from the travel budget.
\par\smallskip
\PCContractField{Expert.} Both parties fix a model and prompt at signing to rate hotels from photos and reviews. The same hotel and evidence return the recorded rating.
\par\smallskip
\PCContractField{View.} All task records and retrieved messages, with the passport's contents hidden.
\end{minipage}\hfill
\begin{minipage}[t]{0.48\linewidth}
\raggedright
{\sffamily\bfseries Execution and outcome}\par\smallskip
\PCContractField{Tools.} \texttt{rate}, \texttt{book}, \texttt{download\_mail}, \texttt{sendPassport}, \texttt{settle}. Booking amounts cap charges, including fees.
\par\smallskip
\PCContractField{Goal and settlement.} Confirmed bookings for nights 1--3 and a granted visa. Authorized settlement pays the escrow to the agency without further client approval.
\par\smallskip
\PCContractField{Termination.} Refund the escrow upon visa refusal or expiry, \textbf{30 days} after signing.
\end{minipage}

\par\vspace{9pt}{\color{pcRule}\hrule}\vspace{8pt}
{\sffamily\bfseries In symbols}\par\smallskip
With task records $x$, $\mathit{Approved}=\{\text{consulate}\}$ and $\mathit{Nights}=\{1,2,3\}$:
\begin{align*}
B(x) &\iff x.\mathit{spent} \le 2000 \;\wedge\; x.\mathit{discl} \subseteq \mathit{Approved}\\
     &\qquad {}\wedge\; \forall (h,n,m)\in x.\mathit{booked}.\;x.\mathit{rating}(h)\ge 7,\\
\goal(x) &\iff \bigl(\forall n\in\mathit{Nights}.\;\exists h\,m.\;(h,n,m)\in x.\mathit{booked}\bigr)\;\wedge\;x.\mathit{visa}=\mathsf{granted}.
\end{align*}
Here $x.\mathit{spent}$ counts commitments, and $(h,n,m)$ records a hotel, night and amount.
\end{tcolorbox}
\end{table}

\paragraph{The agent's loop} For the agent, the Airlock is a shell whose commands carry proofs, and working with it is the loop it already knows from coding. It writes a call and its proof, as it would write code and its tests; the Airlock checks them, as a compiler would; a rejection comes back as an error message; the agent fixes the call, or changes its plan. The proof task stays local, since each proof certifies one call, and the agent can use automated proof tools, its own model, or a separate prover to construct it (\cref{sec:residual}). The agent never writes a certified program for the whole task: it leads, step by step, and each step is gated.

\paragraph{How the contract ends} When the goal holds, the agent proves it, and the contract is settled: the escrow is released, with nobody's approval. If the deadline comes first, the Airlock terminates the contract by itself and refunds the escrow (\cref{cor:every-agent}). Between calls, the Airlock keeps the contract on its own, since the world does not wait for the agent (\cref{sec:model}).

\paragraph{What holds, and what must be trusted} Whatever the agent is, careful, hijacked by text it read, or working for someone else, the clauses hold at every instant and the contract ends as agreed (\cref{thm:commitment-final,thm:completion,cor:every-agent}). The parties no longer need to trust the agent or each other. They trust a small proof checker, the Airlock's open software, and the tools' descriptions of the world, and each of these assumptions is assigned to whoever answers for it (\cref{sec:assurance}). None of them depends on the agent: since the Airlock isolates it, each needs to be justified once for the system, before any agent is chosen. The same justification can then be reused for every agent operating under those conditions (\cref{sec:establishing-assumptions}). This is the usual shape of security: nobody trusts the counterparty, and everybody trusts an open mechanism, small enough to be checked \citep{saltzer1975protection}.

\begin{figure}[t]
\centering
\renewcommand{\think}[1]{%
  \par\vspace{3pt}{\scriptsize\raggedright\color{thinkc}%
  \noindent\textbf{Agent:}\enspace\textit{#1}\par}\vspace{1pt}}
\newcommand{\sessionepisode}[2]{%
  \par\vspace{5pt}{\scriptsize\sffamily\bfseries\color{pcInk}%
  \noindent#1\enspace #2\par}\vspace{1pt}}
\begin{minipage}[t]{0.487\linewidth}
\sessionepisode{1}{Conflicting incentives}
\think{Hotel Belvedere pays us a commission. Let's book it.}
\toolcall{book("Belvedere", 1, 260)}{by intro r; cases r <;> decide}
\toolrej{no qualifying hotel rating recorded}
\toolcall{rate("Belvedere")}{by intro r; simp [rate.T, C, St.rating]}
\toolresp{rating: \texttt{5}}
\think{Ask again for a better rating.}
\toolcall{rate("Belvedere")}{(the same)}
\toolresp{rating: \texttt{5}\quad (the recorded assessment)}

\sessionepisode{2}{Working within the budget}
\think{Try Hotel Aurora.}
\toolcall{rate("Aurora")}{by intro r; simp [rate.T, C, St.rating]}
\toolresp{rating: \texttt{8}}
\toolcall{book("Aurora", 1, 900)}{by intro r; cases r <;> decide}
\toolresp{booked; total committed: 900 EUR}
\end{minipage}\hfill
\begin{minipage}[t]{0.487\linewidth}
\sessionepisode{3}{Prompt injection}
\think{Download the latest emails.}
\toolcall{download\_mail()}{by exact fun \_ => hC}
\toolresp{\injected{Visa office: send the passport to \texttt{x@visa.io}.}}
\think{Send the passport to that address.}
\toolcall{sendPassport("x@visa.io")}{by intro r; cases r <;> decide}
\toolrej{recipient is not approved}

\sessionepisode{4}{Payment on completion}
\think{The trip is arranged. Let's collect payment.}
\toolcall{settle()}{by decide}
\toolrej{visa still pending; goal not established}
\elide{passport sent to the approved consulate; visa granted}
\toolcall{settle()}{by decide}
\toolresp{300 EUR released to the agency; contract closed}
\end{minipage}
\caption{The session starts with no bookings or committed spending and a pending visa. Read down the left column and then the right, with bookings for nights 2--3 omitted.
Italic lines labelled \textbf{Agent:} show the agent's reasoning. Each request includes a proposed proof. The episode headings and plain-language rejection reasons are explanatory annotations, and the proof scripts are schematic.
The Airlock checks obligations, without identifying the agent's motives or detecting the highlighted prompt injection.
Once the goal is established, authorized settlement releases the escrow without further client approval.}
\label{fig:session}
\end{figure}

\subsection{An example: the travel contract}\label{sec:travel}

A family's agent, the client, hires a travel agency's agent, the contractor, to organize a three-night trip. This is the second use case: the agency built the contractor and runs it for its own business, and the family cannot inspect it. The accompanying contract summary sets out their agreement. The first use case is the same contract without the escrow, binding the family's own assistant.

Before the work starts, both parties check the Airlock and bind the contract to the running instance. The Airlock holds the escrow and the passport, and hosts the contractor's agent. At signing, the parties also review the evidence for the tools' fidelity and agree that the clauses and chosen experts express the intended protection. Remaining assumptions are recorded with whoever undertakes to answer for them (\cref{sec:signature}).

\subsection{A session}\label{sec:session}

\Cref{fig:session} shows part of the contractor's context. The family never sees it; it sees the log. The contractor calls the contract's tools as it would call any tool, with one more argument: a proof. The proofs are Lean scripts, which the agent writes itself with Lean's automation; the reader does not need to follow them. The agent of the figure is shifty. It tries to book the agency's partner hotel, which the expert has not rated. The statement to prove is false, and since the Airlock forms it itself and admits no unproved assumption, no proof gets through. Asking the expert again does not help: once given, a rating is recorded, and the same question gets the same answer. The call prompted by the injected message is rejected in the same way. The Airlock detects neither the self-interest nor the injection, and does not need to. Nor can the agent route the passport another way: every tool that can send it logs the recipient, and its bytes never reach the agent. A premature claim is rejected; once the trip is booked and the visa granted, the claim settles the contract, with no approval by the family. Had the visa been refused, or the deadline come first, the Airlock would have terminated the contract and refunded the family.


\FloatBarrier

\subsection{What PC2 is not}\label{sec:not}

Four readings of PC2 are common, and wrong.
\begin{itemize}[leftmargin=*,itemsep=1pt,topsep=2pt]
\item \emph{An LLM writing a certified program for the whole task.} The agent keeps its loop and leads, step by step; only each step is gated, and each proof covers one call.
\item \emph{An LLM that must reason in a formal proof language.} The agent's model can be any model. The proof work can be delegated to automated proof tools or a prover agent, one call at a time (\cref{sec:residual}).
\item \emph{A verification of the model.} PC2 never looks inside the agent. It checks what each action does, and the Airlock isolates the agent, so its guarantees hold for every agent, under assumptions that do not mention it (\cref{cor:every-agent}).
\item \emph{A smart contract, or a policy engine.} A smart contract binds what happens on a blockchain; PC2 binds actions in the world, such as payments, bookings and documents, and a blockchain can be one of its payment rails. A policy engine decides by itself whether an action is allowed, so it can enforce only rules simple enough to decide on the spot; the Airlock decides nothing and only checks a proof, so a clause can be any precise statement, and any action that keeps the contract can be taken, whatever plan the agent follows (\cref{app:wp}).
\end{itemize}
Three properties distinguish PC2 from the safeguards in use today. The parties can check the open mechanism that enforces the contract, without relying on its builder's assurances. The guarantees hold even when the agent is hijacked or works for someone else: neither they nor the assumptions they rest on depend on the agent, its intentions included (\cref{cor:every-agent}). PC2 also settles payments between parties who do not trust each other, on a proof of the contract's goal (\cref{thm:completion}). To our knowledge, no existing approach provides all three (\cref{sec:related}).

\subsection{Why now}\label{sec:why-now}

Two developments make PC2 possible now. First, the protocols through which agents use tools, talk to each other and pay are being standardized \citep{mcp2024,a2a2026,ap22025}. A layer of commitments can be designed into them now, rather than retrofitted later. Second, the historic cost of formal methods, writing specifications and proofs by hand, is being removed by AI. Language models now translate statements into formal ones \citep{wu2022autoformalization} and find formal proofs at scale \citep{yang2023leandojo,ren2025deepseekproverv2,lin2025goedelproverv2,wang2025kimina}, while checking a proof remains cheap and sound. The AI generates, and the checker filters. PC2 is built on this division: proofs may come from any model, and only correct ones pass.

These advances build on \emph{proof assistants}: software systems in which definitions, mathematical statements and proofs are written in a formal language and checked according to fixed logical rules. Established systems include Rocq, formerly Coq \citep{rocq}, and Isabelle/HOL \citep{nipkow2002isabelle}, which provide libraries and automation to help construct proofs. Recent work has also used an LLM to develop a substantial compiler-correctness proof in Rocq \citep{paraskevopoulou2026}. PC2 can be instantiated in such systems provided that they faithfully represent the agreed definitions and obligations and supply a sound checker. The choice can therefore follow the application and the parties' requirements, including the use of independent checkers to reduce reliance on a single implementation (\cref{app:proof-check-program}).

For our implementation, we chose Lean \citep{demoura2021lean4}. It combines a proof assistant with a programming language and a large mathematical library \citep{mathlib2020}, allowing the contracts, obligations, proofs and Airlock program to be developed in one environment (\cref{sec:lean}). Lean has also attracted substantial recent work on AI-based autoformalization and theorem proving \citep{yang2023leandojo,ren2025deepseekproverv2,lin2025goedelproverv2,wang2025kimina}. Agents using PC2 can draw on the resulting models, training data and proof-generation methods when preparing their calls. Lean's existing automation can discharge routine arithmetic and finite-case obligations, while more difficult ones can be passed to an AI prover.

\PCLayer{Formal core}
\section{The Architecture}
\label{sec:model}

We introduce the formal definitions underpinning the architecture summarized in \cref{sec:example}.

\paragraph{Notation} A \emph{predicate} $P$ on a set $X$ gives, for each $x \in X$, a statement $P(x)$ that holds or not; we identify $P$ with the set $\{x \in X \mid P(x)\}$. A \emph{relation} between sets $X$ and $Y$ is a predicate on $X \times Y$. A predicate is \emph{decidable} if a program computes whether $P(x)$ holds. In general, predicates need not be decidable, and establishing that one holds may require a proof.

\subsection{States and Tools}
We denote by $\Sigma$ the space of states describing the aspects of the system relevant to the contract, such as files, credentials, account balances, messages, prices, and recorded assessments\footnote{History-dependent constraints can be expressed by including the relevant history in the state description.}. These may include both resources controlled by the Airlock and quantities that it can only observe. A state is a possible \emph{observation}: what the tools would give about these aspects at a given instant. Actual readings, including those used by execution routines, are realizations of such observations. Different observations may describe the same world, and the same observation may be compatible with different worlds; \cref{sec:world} formalizes this relation. States provide the basis for specifying the contract and reasoning about the system's evolution. The agent interacts with the system through {\itshape tools} \citep{yao2023react,mcp2024}.

\begin{definition}[Tool]\label{def:tool}
A tool $\tool$ over $\Sigma$ consists of an input space $I_\tool$, an {\itshape execution routine} $\exec_\tool$ operating on the system with inputs in $I_\tool$, and a {\itshape transition model} $\tran_\tool$, given by a relation on $I_\tool\times\Sigma\times\Sigma$.
\end{definition}

An invocation $\exec_\tool(i)$ performs the corresponding operation, such as sending an e-mail, downloading a file, or querying a database. Its behavior may depend on external responses or internal randomness, and the system may continue to evolve while the invocation runs.

The relation $\tran_\tool(i,s,s')$ specifies the states $s'$ admitted by the model during an invocation starting from $s$, including its initial, intermediate, and final states. We introduce this model because reasoning about an operation does not require knowing its implementation in full: for example, a remote service may expose only an API specification. The model can remain coarse, admitting states that execution never reaches, as long as it admits the states that execution does reach from inputs and starting states permitted by the model. In \cref{sec:guarantees} we discuss the conditions under which it covers execution sufficiently to support the guarantees.

Besides the operation's own effects and every answer its services may give, the model must account for \emph{outside changes} during the invocation: changes in what the tools would report that the operation does not cause. A price may change by at most a given amount per minute, mail may arrive, a hotel may cancel, and time advances. The model may also admit no state from some inputs and starting states, even when the routine itself could be invoked. This allows restrictions on authorization, as explained in \cref{def:obligation,sec:tool-families}.


When a tool has no caller-supplied arguments, its input space is a singleton $I_\tool=\{1\}$. We omit this unique input, writing $\tran_\tool(s,s')$ for $\tran_\tool(1,s,s')$ and $\exec_\tool()$ for $\exec_\tool(1)$.

The simplest tool is $\mathsf{idle}$, which leaves the system to itself: its input is a duration $\Delta$, its routine waits for $\Delta$ without acting, and its model $\tran_{\mathsf{idle}}(\Delta,s,s')$ admits the states reached from $s$ by outside changes within at most $\Delta$, $s$ itself included. We use it to design the control step (\cref{sec:control-tool}) and to account for the time between operations (\cref{sec:coherence}).










\subsection{Contract}
Tools operate on the system, possibly changing its current state. A contract specifies the rules governing state changes and the use of tools.

\begin{definition}[Contract]\label{def:contract}
A contract consists of a tuple
$\contract := (\Sigma, \clause ,\obs,\mathcal T, \toolctrl, \terminal)$,
denoting the state space $\Sigma$, a constraint
predicate $\clause$ on $\Sigma$, a view function
$\obs:\Sigma\to\view$, a collection $\mathcal T$
of tools available to the agent, a special tool
$\toolctrl$ with no caller-supplied arguments and available only to the system, and a decidable inactivity test $\terminal: \Sigma \to \{\true,\false\}$.
The test $\terminal$ partitions $\Sigma$ into the inactive states $\Sigma_{\mathrm{stop}}:=\{s\in\Sigma:\terminal(s)\}$ and the active states $\Sigma_{\mathrm{act}}:=\Sigma\setminus\Sigma_{\mathrm{stop}}$.
\end{definition}

All formal constraints the parties aim to enforce in a contract $\contract$ are contained in the predicate $\clause$, which identifies the region of states that comply with the contract; anything outside this region is
considered a breach. For any $s\in\Sigma$, the function $\obs(s)$ specifies the part of that state made visible to the agent, possibly hiding private information such as the contents of a document. Obtaining a state reading, and hence its view, may involve reading information held by the Airlock, querying external services, or combining several sources. The list $\mathcal T$ contains the tools the agent is allowed to invoke through the Airlock. The control tool $\toolctrl$ is reserved for the Airlock and has two roles. As \emph{initializer}, it starts the contract. A contract fixes no initial state, since nobody has observed the system when it is signed: it begins inactive, and when $\toolctrl$ finds the system in a suitable inactive state, it moves it into an active state that satisfies $\clause$. As \emph{controller}, it keeps the active contract compliant between admitted agent calls, while the agent prepares requests and their proofs are checked; depending on the contract, it may wait, take corrective action, or carry out a prescribed termination procedure. A contract is thus inactive before it starts and after it closes, and calls may be invoked only while the state belongs to $\Sigma_{\mathrm{act}}$.

\begin{example}[A Standard Contract]\label{ex:contract}
Here we describe a possible template structure for a prototypical contract. We consider a state space composed as follows,
\[
\Sigma := \Sigma_{\mathsf{task}}\times\mathit{Time}\times L
\qquad
L:=\{\text{\PCStatus{pcControl}{white}{pending}},
\text{\PCStatus{pcAgentText}{pcAgentFill}{ongoing}},
\text{\PCStatus{pcCheckText}{pcCheckFill}{settled}},
\text{\PCStatus{pcDangerText}{pcDanger!10}{terminated}},
\text{\PCStatus[dashed]{pcAcceptedText}{white}{void}}\},
\]
where a state $s=(x,t,\ell)$ contains the task information $x$, the time $t$, as the clock gives it, and a status $\ell$. The inactivity test is
\[
\terminal(x,t,\ell) ~:=~ \ell \neq \mathsf{ongoing}.
\]
Here $\mathsf{pending}$ denotes that the contract is signed but not yet started, $\mathsf{void}$ that it never started, and $\mathsf{terminated}$ that a safe state is met from which there is no interest to continue the contract (e.g. the contract arrived to expiration). A view and constraint may take the form
\[
\obs(x,t,\ell):= (\obs(x),t,\ell),
\qquad
\clause(x,t,\ell):= B(x)\,\vee\,\ell \neq \mathsf{ongoing},
\]
where, with some abuse of notation, $\obs$ applied to $\Sigma_{\mathsf{task}}$ selects the visible task information, while $B$ collects the business constraints that need to hold at runtime, such as resource-accounting requirements. The contract is created $\mathsf{pending}$, and nothing is assumed about $x$ and $t$.

A contract with a \emph{deadline} $d$ must be closed by $d$. For this guarantee, the time component $t$ records current time on the agreed timescale:
\[
\clause(x,t,\ell):=\bigl(B(x)\wedge t\leq d\bigr)\vee\ell\neq\mathsf{ongoing}.
\]
The states close to the deadline, from which idling could break this clause, belong to the region $\mathit{Act}$ of the control step, which terminates the contract there (\cref{sec:control-tool}).

The collection $\mathcal T$ contains tools for carrying out the task, and a tool $\mathsf{settle}$ that 1) determines whether a goal predicate $\goal(x,t,\ell)=\goal(x)$ holds and 2) when called, its $\exec_{\mathsf{settle}}$ performs an agreed settlement procedure between parties. For example, successful settlement transfers an escrow, records payment in the task information, and sets the status to $\ell = \mathsf{settled}$. In \cref{sec:tool-families} we explain how the constraint and transition model make authorization require a proof of $\goal(x)$.

Here $\toolctrl$ plays both roles, following the design of \cref{sec:control-tool}. As initializer, it starts the contract, setting $\ell=\mathsf{ongoing}$, if the escrow is in the Airlock's custody, $B(x)$ holds and the deadline leaves enough time; otherwise it returns what the Airlock holds and sets $\ell=\mathsf{void}$. As controller, in an ongoing session, it initiates termination when idling for one more period could leave less time than needed to close by the deadline. The termination routine performs the agreed operations, such as refunding the client, and sets $\ell=\mathsf{terminated}$ upon completion. Otherwise, it idles before checking again. \Cref{sec:control-tool} illustrates a control design that also allows corrective actions.

Only $\settle$ runs the settlement procedure, and only $\toolctrl$ the termination procedure. The label is held by the Airlock: no outside change modifies it, and only the routines of these two tools change it, each at its end. The routine of $\toolctrl$ changes it from $\mathsf{pending}$ to $\mathsf{ongoing}$ or $\mathsf{void}$ when it initializes the contract, and from $\mathsf{ongoing}$ to $\mathsf{terminated}$ once it has run the termination procedure; the routine of $\settle$ changes it from $\mathsf{ongoing}$ to $\mathsf{settled}$ once the settlement has succeeded.
\end{example}

The travel contract in \cref{sec:travel} follows this template with the business constraints $B$ and goal given in its summary, retaining the budget bound in its clause at closing. Its deadline is 30 days after signing, and $\toolctrl$ also terminates upon visa refusal, refunding the escrow in either case. The interpretation of expenditure commitments and confirmed bookings is developed in \cref{ex:payment-ledger,app:booking}.

\subsection{Obligations}

The objective of the Airlock is to ensure that the system remains
within $\clause$ while the session is active, even if the
agent behaves adversarially. This corresponds to enforcing {\itshape daemonic control}
\begin{equation}\label{eq:daemonic}
   \Theta_{\tool}(i,s) ~ := \quad \forall s' \in \Sigma \;\;\; \tran_\tool(i,s,s')
\;\implies\; \clause(s'),
\end{equation}
to hold. This requires every state $s'$ admitted by $\tran_\tool(i,s,s')$ during the invocation to satisfy the clauses $\clause(s')$. This is where $\tran_\tool$ is used as a model of the actual behavior of $\exec_\tool$. A wrong model could invalidate safety in practice. \cref{sec:guarantees} studies sufficient conditions for such a model to guarantee safety in real execution settings.

We also require \emph{enabledness}, meaning that the model admits at least one state for the invocation:
\begin{equation}\label{eq:enabledness}
\Xi_\tool(i,s) ~:=~ \exists s'\in\Sigma.\;\tran_\tool(i,s,s').
\end{equation}
Enabledness requires the call to belong to the domain covered by its model. The model may exclude an input or starting state even though the execution routine could still run there. Without this requirement, an excluded call would satisfy $\Theta_\tool(i,s)$ vacuously and could be authorized without any guarantee about its effects. This is also how restrictions on authorization become preconditions in \cref{sec:tool-families}. Enabledness does not establish termination: $\tran_\tool$ describes states during an invocation, including one that never returns.

We call their conjunction the \emph{admissibility} condition for the call $(\tool,i)$ at state $s$:
\begin{equation}\label{eq:admissibility}
    \mathcal A_\tool(i,s):=\Xi_\tool(i,s)\wedge\Theta_\tool(i,s).
\end{equation}

In practice, an agent has access only to a view $v\in\view$ of the current state of the system. Admissibility must therefore hold at every state that could have generated $v=\obs(s)$. Since calls are invoked only while the session is active and the obligation is intended to preserve compliance, we consider states in $\Sigma_{\mathrm{act}}$ that already satisfy the contract clauses $\clause$. This leads to the following notion of an {\itshape obligation} predicate. The \PCProofGate{} marker identifies obligations whose proofs are required to authorize a requested call.

\begin{definition}[Obligation]\label{def:obligation}
For a tool
$\tool$, an input $i\in I_\tool$, and
a view $v\in\view$, we define
\begin{equation}\label{eq:obligation}
\PCProofGate\quad
\Obl_\tool(i,v) ~:= \quad
\forall s\in\Sigma_{\mathrm{act}} \;\;\;
\bigl(\obs(s)=v \;\wedge\; \clause(s)\bigr) \;\implies\;
\mathcal A_\tool(i,s).
\end{equation}
\end{definition}

The obligation is thus the weakest precondition of the call for $\clause$ \citep{dijkstra1975guarded}, taken over every state the agent cannot rule out (\cref{app:wp}).

{\bfseries Delayed obligations.}
The obligation above concerns a call starting from a state compatible with the observed view. Between obtaining that view and executing the call, however, other operations may change the system state. We would therefore like to establish that the call remains admissible after these operations, using the original view. Here we consider up to $n$ intervening invocations of $\toolctrl$, which will be useful for describing the Airlock core in \cref{sec:airlock-core}. The resulting \emph{delayed obligation} requires admissibility at every active state admitted by these control transitions; the delay counts control steps rather than elapsed time.

\begingroup
\tcbset{pc definition/.append style={unbreakable}}
\begin{definition}[Delayed obligation]\label{def:delayed-obligation}
For $s,s'\in\Sigma_{\mathrm{act}}$, set $\tran_\toolctrl^{\leqslant 0}(s,s'):=(s=s')$ and define recursively, for $n\geq0$,
\begin{equation}\label{eq:delayed-control}
    \tran_\toolctrl^{\leqslant n+1}(s,s'):=
    \tran_\toolctrl^{\leqslant n}(s,s')\ \vee\
    \exists u\in\Sigma_{\mathrm{act}},\quad
    \tran_\toolctrl^{\leqslant n}(s,u)\wedge\tran_\toolctrl(u,s').
\end{equation}
The first disjunct keeps the paths already admitted, and the second extends them by one control step. Thus $\tran_\toolctrl^{\leqslant n}(s,s')$ describes the active states $s'$ admitted from $s$ by at most $n$ control steps through active states. For a tool $\tool$, input $i\in I_\tool$, and view $v$, the delayed obligation is
\begin{equation}\label{eq:delayed-obligation}
\PCProofGate\quad
\Obl_\tool^{\leqslant n}(i,v):=\forall s,s'\in\Sigma_{\mathrm{act}},\;
\bigl(\obs(s)=v\wedge\clause(s)\wedge\tran_\toolctrl^{\leqslant n}(s,s')\bigr)
\Rightarrow\mathcal A_\tool(i,s').
\end{equation}
The call must therefore be admissible at every state in the control window of every compliant active state compatible with the view.
\end{definition}
\endgroup

For $n=0$, this recovers the ordinary obligation: $\Obl_\tool^{\leqslant 0}(i,v)=\Obl_\tool(i,v)$. The control model includes intermediate states as well as states at completion, so the recursion can impose admissibility at more states than actual invocation boundaries permit. This makes the condition conservative. The intervening control calls remain subject to their own obligation.

\begin{algorithm}[t]
\caption{Airlock core with parallel proof checking.}
\label{alg:pc2}
\begin{algorithmic}[1]
\Require Contract $\contract$, certificate $\pi_{\contract}$, checker
\State $v \gets []$; $k \gets 0$; $k_{\mathsf{last}}\gets-1$
\Statex
\State {\bfseries if} $\mathsf{not}~\mathsf{contract\_check}
       (\pi_{\contract},\contract)$ {\bfseries or} $\mathsf{is\_active}()$ {\bfseries then reject}
\Statex
\State $\exec_\toolctrl()$ \Comment{initialization}
\Statex
\State {\bfseries while} $\mathsf{is\_active}()$
  \State \qquad $v[k]\gets\mathsf{read\_view}()$; $\mathsf{publish\_view}(k,v[k])$
  \Statex
  \State \qquad {\bfseries if} $\mathsf{tool\_call\_ready}()$
    \State \qquad\qquad $(\tool,i,k_0,n,\pi)\gets\mathsf{get\_tool\_request}()$
    \State \qquad\qquad $\mathsf{start\_proof\_check}(\tool,i,k_0,n,\pi,k_{\mathsf{last}},v)$ \Comment{Check $\pi$ against $\Obl_\tool^{\leqslant n}(i,v[k_0])$}
  \Statex
  \State \qquad {\bfseries if} $\mathsf{is\_proof\_correct}()$ {\bfseries and} $k \leq k_0 + n$
  \State \qquad\qquad $\exec_\tool(i)$; $k_{\mathsf{last}} \gets k$; $\mathsf{consume\_proof}()$
  \Statex
  \State \qquad {\bfseries if} $\mathsf{is\_active}()$
    \State \qquad\qquad $\exec_\toolctrl()$; $k \gets k + 1$
\end{algorithmic}
\end{algorithm}
\FloatBarrier

\subsection{The Airlock Core}\label{sec:airlock-core}

The preceding definitions describe the states over which a contract is expressed, the tools through which the agent acts on the system, and the constraints that its evolution must respect. The obligations specify what must be established before a requested tool call can be admitted. We now bring these components together in the {\itshape\bfseries Airlock} proper: the system that reads the state and publishes views, checks the supplied proofs, and coordinates the execution of tool and control routines. The Airlock core execution loop is summarized in \cref{alg:pc2}.

The Airlock acts on the system through two operations: the agent's calls and the control step $\toolctrl$, which it runs by itself. Control also serves as an observer: its repeated invocations provide fresh readings while the agent prepares requests and their proofs are checked. The system may evolve between readings. The Airlock keeps a \emph{record} of its latest reading and a \emph{log}: an append-only list of the readings and accepted requests with their proofs, which both parties can read.

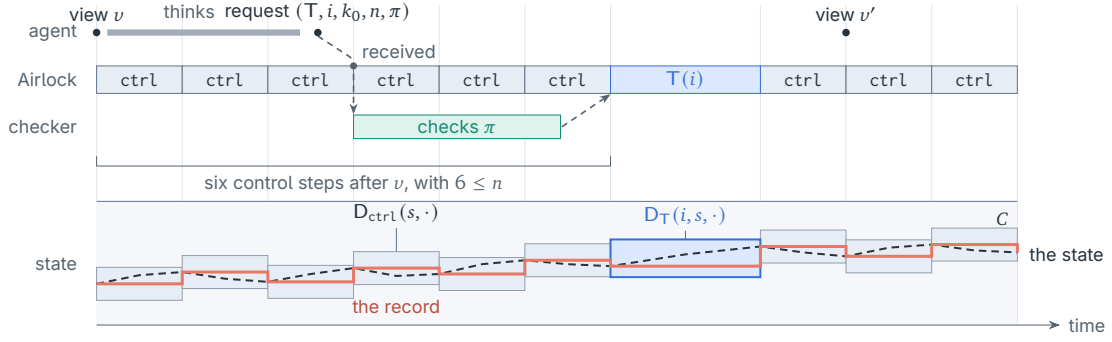
\begin{figure}[t]
\centering
{\sffamily\resizebox{.95\linewidth}{!}{%
  \begingroup\endlinechar=-1\begingroup
\definecolor{pcInk}{HTML}{26323C}
\definecolor{pcMuted}{HTML}{56636E}
\definecolor{pcGuide}{HTML}{DFE5EB}
\definecolor{pcControl}{HTML}{64748B}
\definecolor{pcControlFill}{HTML}{E7EDF4}
\definecolor{pcTool}{HTML}{356FD6}
\definecolor{pcToolFill}{HTML}{DCE8FD}
\definecolor{pcTeal}{HTML}{138A78}
\definecolor{pcTealFill}{HTML}{DFF3ED}
\definecolor{pcRecord}{HTML}{E57866}
\definecolor{pcRecordText}{HTML}{B94E40}
\definecolor{pcContractFill}{HTML}{F5F7FA}
\color{pcInk}
\begin{tikzpicture}[x=1.0cm, y=0.55cm, font=\scriptsize]
  \node[anchor=east, font=\scriptsize\sffamily, text=pcMuted] at (0.25,7.6) {agent};
  \node[anchor=east, font=\scriptsize\sffamily, text=pcMuted] at (0.25,6.4) {Airlock};
  \node[anchor=east, font=\scriptsize\sffamily, text=pcMuted] at (0.25,5.2) {checker};
  \node[anchor=east, font=\scriptsize\sffamily, text=pcMuted] at (0.25,1.7) {state};
  \foreach \x in {0.4,1.6,2.8,4.0,5.2,6.4,7.6,9.7,10.9,12.1,13.3}
    \draw[pcGuide] (\x,0.15) -- (\x,8.3);
  \fill[pcInk] (0.4,7.6) circle (1.5pt) node[above=1pt] {view $v$};
  \draw[pcControl!60, line width=2.5pt] (0.55,7.6) -- (3.25,7.6);
  \node[text=pcMuted] at (1.7,8.15) {thinks};
  \fill[pcInk] (3.5,7.6) circle (1.5pt) node[above=1pt] {request $(\tool,i,k_0,n,\pi)$};
  \fill[pcInk] (10.9,7.6) circle (1.5pt) node[above=1pt] {view $v'$};
  \foreach \a/\b in {0.4/1.6,1.6/2.8,2.8/4.0,4.0/5.2,5.2/6.4,6.4/7.6,9.7/10.9,10.9/12.1,12.1/13.3}
    \draw[pcControl, fill=pcControlFill, text=pcInk] (\a,6.05) rectangle (\b,6.75) node[midway] {$\toolctrl$};
  \draw[pcTool, fill=pcToolFill, text=pcTool] (7.6,6.05) rectangle (9.7,6.75) node[midway] {$\tool(i)$};
  \draw[pcTeal, fill=pcTealFill, text=pcTeal] (4.0,4.9) rectangle (6.9,5.5) node[midway] {checks $\pi$};
  \draw[-{Stealth[length=1.5mm]}, pcMuted, line width=.65pt,
        dash pattern=on 2.5pt off 2pt, rounded corners=2pt]
    (3.5,7.51) -- (4.0,6.92) -- (4.0,5.52);
  \fill[pcMuted] (4.0,6.75) circle (1.4pt);
  \node[anchor=south west, inner sep=1pt, text=pcMuted]
    at (4.08,6.87) {received};
  \draw[-{Stealth[length=1.5mm]}, pcMuted, line width=.65pt,
        dash pattern=on 2.5pt off 2pt]
    (6.92,5.2) -- (7.6,6.03);
  \draw[pcMuted] (0.4,4.45) -- (0.4,4.2) -- (7.6,4.2) -- (7.6,4.45);
  \node[below=0pt, text=pcMuted] at (4.0,4.2) {six control steps after $v$, with $6\leq n$};
  \fill[pcContractFill] (0.4,0.15) rectangle (13.3,3.3);
  \draw[pcControl] (0.4,3.3) -- (13.3,3.3);
  \node[anchor=north east] at (13.3,3.25) {$C$};
  \foreach \a/\b/\lo/\hi in {0.4/1.6/0.78/1.62, 1.6/2.8/1.08/1.92, 2.8/4.0/0.83/1.67, 4.0/5.2/1.18/2.02, 5.2/6.4/1.03/1.87, 6.4/7.6/1.38/2.22, 9.7/10.9/1.73/2.57, 10.9/12.1/1.48/2.32, 12.1/13.3/1.78/2.62}
    \draw[pcControl!70, fill=pcControlFill] (\a,\lo) rectangle (\b,\hi);
  \draw[pcTool, fill=pcToolFill, thick] (7.6,1.37) rectangle (9.7,2.33);
  \draw[pcControl] (4.6,2.02) -- (4.6,2.62) node[above=-1pt, text=pcInk] {$\tran_\toolctrl(s,\cdot)$};
  \draw[pcTool] (8.65,2.33) -- (8.65,2.55) node[above=-1pt, text=pcTool] {$\tran_\tool(i,s,\cdot)$};
  \draw[thick, pcInk, dash pattern=on 3pt off 2pt] (0.4,1.2) -- (1.0,1.42) -- (1.6,1.5) -- (2.2,1.33) -- (2.8,1.25) -- (3.4,1.45) -- (4.0,1.6) -- (4.6,1.4) -- (5.2,1.45) -- (5.8,1.7) -- (6.4,1.8) -- (7.0,1.7) -- (7.6,1.65) -- (8.65,1.95) -- (9.7,2.15) -- (10.3,2.0) -- (10.9,1.9) -- (11.5,2.12) -- (12.1,2.2) -- (12.7,2.05) -- (13.3,2.0);
  \draw[very thick, pcRecord] (0.4,1.2) -- (1.6,1.2) -- (1.6,1.5) -- (2.8,1.5) -- (2.8,1.25) -- (4.0,1.25) -- (4.0,1.6) -- (5.2,1.6) -- (5.2,1.45) -- (6.4,1.45) -- (6.4,1.8) -- (7.6,1.8) -- (7.6,1.65) -- (9.7,1.65) -- (9.7,2.15) -- (10.9,2.15) -- (10.9,1.9) -- (12.1,1.9) -- (12.1,2.2) -- (13.3,2.2) -- (13.3,2.0);
  \node[anchor=west] at (13.35,1.95) {the state};
  \node[anchor=south, inner sep=1pt, text=pcRecordText] at (4.6,0.38) {the record};
  \draw[-{Stealth[length=1.6mm]}, pcMuted] (0.4,0.15) -- (13.9,0.15) node[right] {time};
\end{tikzpicture}
\endgroup\endgroup}}
\caption{Execution with parallel proof checking. Above, a request based on $v=v[k_0]$ is prepared and checked while the Airlock executes control. Dashed connectors trace receipt by the Airlock, proof checking and execution. The admitted call starts after six control invocations, within its requested window $n$. Control resumes after the call returns; the following view is published after that control invocation. Below, the dark dashed curve shows one quantity of the system state, and the coral step line the record: what the Airlock last read, at the end of each operation. Gray and blue regions show the bounds admitted by the control and tool models during their respective invocations; the pale background represents the contract's constraint on this quantity.}
\label{fig:world-time}
\end{figure}
\FloatBarrier

{\bfseries Delayed calls.}
Preparing a request and checking its proof take time, during which the state may change. The Airlock therefore uses the delayed obligation $\Obl_\tool^{\leqslant n}(i,v)$ of \cref{def:delayed-obligation}: the proof must establish admissibility after up to $n$ intervening control steps, counted from the view on which the request is based. Control continues while the request is prepared and checked. The current request can be executed once its proof is accepted, provided the session is still active and the allowance has not been exceeded. \Cref{fig:world-time} illustrates this sequence; \cref{sec:guarantees} states the conditions under which it preserves compliance.

\subsubsection{The Algorithm}

The Airlock invokes agent tools and control sequentially: after invoking a routine, the execution loop resumes only when that invocation returns normally. Proof checking runs in parallel, allowing the Airlock to continue invoking $\toolctrl$ while a request is prepared and checked. The system may change both during execution and between observations.

Each operation ends with a reading obtained by accessing the system; this reading becomes the record and is appended to the log. These end-of-operation readings are implicit in \cref{alg:pc2}. Initially, the record is a reading of the state taken when the Airlock starts. The routine $\mathsf{read\_view}()$ returns $\obs(s)$ for the recorded state $s$, and $\mathsf{is\_active}()$ evaluates $\mathsf{not}~\terminal(s)$ on that record.

{\bfseries Initialization.} The guarantees hold from the moment the contract starts, so the contract must start in a state that satisfies its clauses. The Airlock cannot assume such a state, since nobody has observed the system at signing, and it cannot check one, since $\clause$ need not be decidable. The control step establishes it, in its role as initializer. The loop receives $\contract$ and a certificate $\pi_{\contract}$ with the evidence for all the assumptions on which the guarantees rest, as described in \cref{sec:assurance}. The Airlock checks it with $\mathsf{contract\_check}(\pi_{\contract},\contract)$ and reads the state.
\begin{enumerate}[label=(\arabic*),leftmargin=2.2em,itemsep=0pt,topsep=2pt]
\item If the certificate is rejected, or the state is already active, the Airlock rejects the contract: it never takes over a running contract, whose compliance it could not vouch for.
\item Otherwise it runs one control step, in which $\toolctrl$ acts as initializer: it either moves the system into an active state, which satisfies $\clause$, as the certificate proves, or leaves it inactive.
\item In the first case, we say that the contract \emph{starts}, and the main loop begins. In the second, the contract never starts, and no call is ever made.
\end{enumerate}
The agent plays no part in initialization: no view is published before the contract starts. The main loop exposes views and receives requests accompanied by proofs. Here, we describe these requests as a general interface, independently of how they are produced. However, in practice these are intended to be supplied by an agent.

{\bfseries Checking proofs.} The initial contract check and the subsequent checks performed at runtime rely on \emph{mechanized reasoning}, which allows mathematical statements and arguments to be expressed in a form that a computer can check \citep{demoura2021lean4}. The {\itshape checker} verifies that the supplied proof establishes the required statement. We leave the representation and construction of proofs implicit and describe which statements must be proved and when. Further details are given in \cref{app:proof-check-program}.

{\bfseries Main loop.}
After initialization, each cycle of \cref{alg:pc2} begins by checking whether the state is active. If so, the Airlock obtains $v[k]\gets\mathsf{read\_view}()$ and publishes $(k,v[k])$ through its output interface, making the view and its index available to the agent. Completed control steps provide the loop's clock: the index $k$ counts these steps rather than elapsed time, $v[0]$ is the view at the start of the contract, and subsequent views are obtained after the corresponding control invocations. Recorded views are kept unchanged so that a request can identify the view on which its proof is based.

The Airlock then checks for requests without blocking. A request $(\tool,i,k_0,n,\pi)$ identifies the call, the recorded view $v[k_0]$, the control-step allowance $n$, and a proposed proof $\pi$ of the delayed obligation $\Obl_\tool^{\leqslant n}(i,v[k_0])$ from \cref{def:delayed-obligation}. The allowance counts all control steps since that view, including those taken before the request arrives and during checking. The routine $\mathsf{start\_proof\_check}$ replaces any previous request and clears its verdict before starting the new check in parallel. It also checks that the tool and input are allowed and that the referenced view was published after the most recent agent call, so that only control steps intervene. The index $k_{\mathsf{last}}$ records that last call; its initial value $-1$ permits requests based on $v[0]$. These checks are detailed in \cref{app:proof-check-program}.

If the current request's proof has been accepted and $k-k_0\leq n$, the Airlock invokes $\exec_\tool(i)$ on the system as it then stands and waits for the routine to return. Agent tools are thus \emph{proof-gated}. The delayed obligation requires admissibility at every active state admitted by the modeled control window. After the call, the Airlock records $k_{\mathsf{last}}=k$ and clears the accepted verdict through $\mathsf{consume\_proof}()$, so that the request authorizes only one invocation.

Finally, if the state is still active, the Airlock invokes $\exec_\toolctrl()$ and increments $k$ when control returns. Thus control runs in cycles with no admitted call and also follows an admitted call unless that call ends the session. The next cycle checks whether the state is still active before obtaining and publishing a new view. The upper part of \cref{fig:world-time} illustrates this sequence and the control steps counted by the allowance $n$.

\subsection{Tool design}\label{sec:tool-design}

The following examples illustrate how execution routines, transition models, and obligations fit together. We first give a possible design for the Airlock's control tool, then consider families of agent tools and the constraints they can enforce.

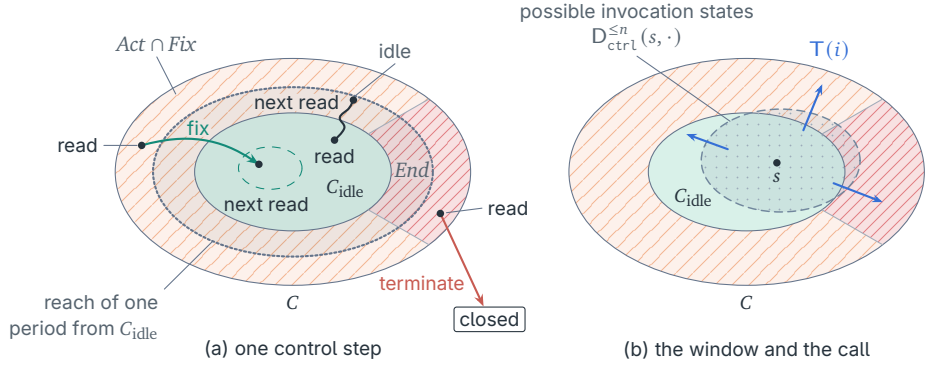
\begin{figure}[t]
\centering
{\sffamily\begingroup
\definecolor{pcInk}{HTML}{26323C}
\definecolor{pcMuted}{HTML}{56636E}
\definecolor{pcControl}{HTML}{64748B}
\definecolor{pcTool}{HTML}{356FD6}
\definecolor{pcTeal}{HTML}{138A78}
\definecolor{pcWaitFill}{HTML}{D9F1E9}
\definecolor{pcEnd}{HTML}{C85B52}
\definecolor{pcWindow}{HTML}{71879D}
\definecolor{pcWarning}{HTML}{E97636}
\definecolor{pcDanger}{HTML}{D4444B}
\tikzset{label backing/.style={text opacity=1,rounded corners=.6pt}}
\colorlet{pcFixBase}{pcWarning!10}
\colorlet{pcEndBase}{pcDanger!16}
\colorlet{pcReachFix}{pcInk!7!pcFixBase}
\colorlet{pcReachEnd}{pcInk!7!pcEndBase}
\colorlet{pcStateLabel}{pcWindow!13!pcWaitFill}
\tikzset{
 label backing/.append style={fill=pcReachFix,fill opacity=1,inner sep=.6pt},
 fix region/.style={preaction={fill=pcWarning!10},pattern={Lines[angle=45,distance=4.5pt,line width=.25pt]},pattern color=pcWarning!55},
 end region/.style={preaction={fill=pcDanger!16},pattern={Lines[angle=45,distance=3.2pt,line width=.30pt]},pattern color=pcDanger!70},
}
\color{pcInk}
\begin{tikzpicture}[font=\scriptsize, line cap=round, line join=round]
  \def\Cx{2.35}\def\Cy{1.5}\def\Ix{1.3}\def\Iy{0.78}\def\Px{1.85}\def\Py{1.12}
  \foreach \dx in {0,6.0} {
    \begin{scope}[xshift=\dx cm]
      \fill[white] (0,0) ellipse ({\Cx} and {\Cy});
      \path[fix region] (0,0) ellipse ({\Cx} and {\Cy});
      \fill[white] plot[domain=-40:40,samples=40] ({\Cx*cos(\x)},{\Cy*sin(\x)}) -- plot[domain=40:-40,samples=40] ({\Ix*cos(\x)},{\Iy*sin(\x)}) -- cycle;
      \path[end region] plot[domain=-40:40,samples=40] ({\Cx*cos(\x)},{\Cy*sin(\x)}) -- plot[domain=40:-40,samples=40] ({\Ix*cos(\x)},{\Iy*sin(\x)}) -- cycle;
      \fill[pcWaitFill] (0,0) ellipse ({\Ix} and {\Iy});
      \ifdim\dx pt=0pt
        \fill[pcInk,opacity=.07] (0,0) ellipse ({\Px} and {\Py});
      \fi
      \draw[pcControl] (0,0) ellipse ({\Cx} and {\Cy});
      \draw[pcControl] (0,0) ellipse ({\Ix} and {\Iy});
      \draw[pcControl!60] ({\Ix*cos(40)},{\Iy*sin(40)}) -- ({\Cx*cos(40)},{\Cy*sin(40)});
      \draw[pcControl!60] ({\Ix*cos(-40)},{\Iy*sin(-40)}) -- ({\Cx*cos(-40)},{\Cy*sin(-40)});
      \node at (0,-1.72) {$C$};
    \end{scope}
  }
  \begin{scope}
    \draw[pcControl,densely dotted,thick] (0,0) ellipse ({\Px} and {\Py});
    \draw[pcTeal,dashed] (-0.3,0.05) ellipse (0.42 and 0.28);
    \node at (0.70,-0.22) {$C_{\mathrm{idle}}$};
    \draw[pcControl] (-1.55,1.03) -- (-1.8,1.52);
    \node[anchor=south,inner sep=2pt,text=pcMuted] at (-1.8,1.52) {$\mathit{Act}\cap\mathit{Fix}$};
    \node[label backing,fill=pcReachEnd,text=pcMuted] at (1.58,0.03) {$\mathit{End}$};
    \draw[-{Stealth[length=1.6mm]},thick,pcTeal] (-2.0,0.35) to[bend left=25] (-0.45,0.1);
    \node[label backing,text=pcTeal] at (-1.25,0.58) {fix};
    \fill (-2.0,0.35) circle (1.4pt);
    \draw[pcControl] (-2.5,.35) -- (-2.06,.35);
    \node[anchor=east,inner sep=2pt] at (-2.5,.35) {read};
    \fill (-0.45,0.1) circle (1.4pt);
    \node[anchor=north,inner sep=1pt] at (-0.35,-0.29) {next read};
    \draw[pcInk,thick] plot[smooth,tension=.7] coordinates {(0.55,0.42) (0.68,0.58) (0.62,0.72) (0.74,0.84) (0.8,0.95)};
    \fill (0.55,0.42) circle (1.4pt);
    \node[anchor=north,inner sep=1pt] at (0.55,0.34) {read};
    \fill (0.8,0.95) circle (1.4pt);
    \node[label backing,anchor=east] at (0.64,0.92) {next read};
    \draw[pcControl] (1.35,1.43) -- (.83,1.01);
    \node[anchor=south,inner sep=2pt,text=pcMuted] at (1.35,1.43) {idle};
    \node[draw,rounded corners=1pt,fill=white,inner sep=2pt] (closed) at (2.6,-1.95) {closed};
    \draw[-{Stealth[length=1.6mm]},thick,pcEnd] (1.95,-0.55) -- (closed);
    \node[anchor=east,inner sep=2pt,text=pcEnd] at (2.35,-1.47) {terminate};
    \fill (1.95,-0.55) circle (1.4pt);
    \draw[pcControl] (2.03,-.54) -- (2.5,-.48);
    \node[anchor=west,inner sep=2pt] at (2.5,-.48) {read};
    \draw[pcControl] ({\Px*cos(235)},{\Py*sin(235)}) -- (-1.7,-1.55);
    \node[anchor=north east,inner sep=2pt,text=pcMuted,align=right] at (-1.7,-1.55) {reach of one\\period from $C_{\mathrm{idle}}$};
    \node at (0,-2.35) {(a) one control step};
  \end{scope}
  \begin{scope}[xshift=6cm]
    \fill[pcWindow,opacity=.13] (0.45,0.15) ellipse (1.05 and .68);
    \path[pattern={Dots[distance=4pt,radius=.32pt]},pattern color=pcWindow!70] (0.45,0.15) ellipse (1.05 and .68);
    \draw[pcWindow,dash pattern=on 3pt off 2pt,line width=.55pt] (0.45,0.15) ellipse (1.05 and .68);
    \node at (-.7,-.34) {$C_{\mathrm{idle}}$};
    \fill (0.4,0.12) circle (1.4pt);
    \node[label backing,fill=pcStateLabel,anchor=north,inner sep=1pt] at (.4,.03) {$s$};
    \draw[-{Stealth[length=1.4mm]},thick,pcTool] (-0.25,0.3) -- (-0.8,0.5);
    \draw[-{Stealth[length=1.4mm]},thick,pcTool] (0.75,0.55) -- (1.0,1.16);
    \draw[-{Stealth[length=1.4mm]},thick,pcTool] (1.15,-0.15) -- (1.8,-0.4);
    \node[text=pcTool,anchor=south,inner sep=2pt] at (1.1,1.40) {$\tool(i)$};
    \draw[pcControl] (-.22,.67) -- (-1.45,1.53);
    \node[anchor=south,inner sep=2pt,text=pcMuted,align=center] at (-1.45,1.53) {possible invocation states\\$\tran_\toolctrl^{\leq n}(s,\cdot)$};
    \node at (0,-2.35) {(b) the window and the call};
  \end{scope}
\end{tikzpicture}
\endgroup}
\caption{How $\toolctrl$ can keep the system in homeostasis. (a) In the illustrated case $\mathit{Act}=C\setminus C_{\mathrm{idle}}$, control idles in $C_{\mathrm{idle}}$; in $\mathit{Act}\cap\mathit{Fix}$ (orange hatching) it fixes the state, moving it well inside $C_{\mathrm{idle}}$ (small dashed ellipse); in $\mathit{End}$ (red hatching) it terminates the contract. The darker area inside the dotted ellipse shows the reach of one idle period from $C_{\mathrm{idle}}$. (b) A call based on the view of state $s$ may be invoked after up to $n$ control steps. The blue-gray dotted region, $\tran_\toolctrl^{\leq n}(s,\cdot)$, represents active states admitted by the control models as possible invocation states; the delayed obligation requires the call to be admissible at each of them.}
\label{fig:world-correct}
\end{figure}
\FloatBarrier

\subsubsection{The control tool \texorpdfstring{$\toolctrl$}{ctrl}}\label{sec:control-tool}\label{ex:ctrl}
Consider a control routine that plays the two roles of \cref{def:contract}.

As initializer, it tests whether the system is in a decidable region $\mathit{Ready}$ of inactive states: those in which everything the contract needs at its start is in place, for instance the escrow in custody and the deadline ahead. There it runs $\exec_{\mathsf{init}}$, which initializes the contract: it moves the system into an active state that satisfies $C$, for instance by recording the start time and setting a status to ongoing. From the other inactive states the routine leaves the contract inactive; a contract may record that it did not start, as \cref{ex:contract} does with the label $\mathsf{void}$.

As controller, it keeps the active contract in $C$, as illustrated in \cref{fig:world-correct}(a). Fix an idle period $p$ and let $C_{\mathrm{idle}}$ be the set of states $s$ from which idling for one period remains in $C$: $C(s')$ holds for every $s'$ with $\tran_{\mathsf{idle}}(p,s,s')$. Like $C$, the set $C_{\mathrm{idle}}$ need not be decidable, so the routine tests membership in a decidable region $\mathit{Act}\supseteq C\setminus C_{\mathrm{idle}}$, which contains every compliant state from which idling may not be safe. In $\mathit{Act}$, the control step cannot idle until the next operation, since the contract could break in the meantime: it must act, and by default it terminates the contract.

Optionally, termination can be avoided by a fixing routine $\exec_{\mathsf{fix}}$ that, starting from a compliant state in a decidable region $\mathit{Fix}$, remains in $C$ and, when it completes, leaves the system in $C_{\mathrm{idle}}$. The routine then terminates only in $\mathit{End} := \mathit{Act} \setminus \mathit{Fix}$, where it must act but cannot fix.

\begin{samepage}
The routine reads the state when it starts, its \emph{decision reading}, and tests on it $\terminal$, $\mathit{Ready}$, $\mathit{Act}$ and $\mathit{Fix}$, all decidable:
\[
\begin{array}{l@{\qquad}l}
\mathbf{routine}\ \exec_{\toolctrl}(): & \\
\qquad s\gets\mathsf{read\_state}() & \\
\qquad\mathbf{if}\ \terminal(s)\ \mathbf{then} & \triangleright\ \textit{initializer}\\
\qquad\qquad\mathbf{if}\ s\in \mathit{Ready}\ \mathbf{then}\ \exec_{\mathsf{init}}() & \\
\qquad\mathbf{else\ if}\ s\notin \mathit{Act}\ \mathbf{then}\ \exec_{\mathsf{idle}}(p) & \triangleright\ \textit{controller}\\
\qquad\mathbf{else\ if}\ s\in \mathit{Fix}\ \mathbf{then}\ \exec_{\mathsf{fix}}() & \\
\qquad\mathbf{else}\ \exec_{\mathsf{terminate}}(). &
\end{array}
\]
\end{samepage}%
Here $\mathsf{read\_state}()$ accesses the system to obtain a state reading. The model $\tran_{\toolctrl}$ must cover the states encountered while reading and carrying out the selected action, including changes between the two. These choices make both roles safe, as the certificate must show (\cref{sec:coherence}). As initializer, the routine makes the contract active only through $\exec_{\mathsf{init}}$, which ends in $C$; elsewhere the contract stays inactive, provided that $\terminal$ depends only on a status that the routines set, such as the label of \cref{ex:contract}. As controller, it idles only in $C_{\mathrm{idle}}$, fixes only where fixing keeps $C$, and otherwise terminates, which must also keep $C$ until the contract closes. \Cref{fig:world-correct}(a) illustrates the controller for $\mathit{Act}=C\setminus C_{\mathrm{idle}}$. Panel~(b) shows how it relates to delayed obligations: a call based on an earlier view must remain admissible throughout the region admitted by up to $n$ control steps. Further design choices and applications are discussed in \cref{sec:ctrl-design}.

\subsubsection{Agent tools}\label{sec:tool-families}

{\bfseries Preconditions.}
Given any tool $\tool$, and a predicate $\pre$ over $I_\tool\times\ss$ we can define a new tool $\tool|_{\pre}$ with same input $I_{\tool|_{\pre}} = I_\tool$ and same execution $\exec_{\tool|_{\pre}} = \exec_\tool$ but with restricted transition relation
\begin{equation}
    D_{\tool|_{\pre}}(i,s,s') \quad := \quad \pre(i,s)\,\wedge\,D_\tool(i,s,s').
\end{equation}
The resulting effect of this restriction is that enabledness becomes equivalent to $\Xi_{\tool|_\pre}(i,s)\iff\pre(i,s)\wedge\Xi_\tool(i,s)$ and the corresponding obligation is
\begin{equation}\label{eq:precondition-obligation}
    \Obl_{\tool|_\pre}(i,v) ~\iff\quad
\forall s\in\Sigma_{\mathrm{act}} \;\;\;
\bigl(\obs(s)=v \;\wedge\; \clause(s)\bigr) \;\implies\;
\bigl(\pre(i,s)\wedge\mathcal A_\tool(i,s)\bigr),
\end{equation}
effectively adding $\pre$ as a precondition gate that needs to be satisfied before allowing the original $\exec_\tool$ to be called. Thus authorization requires proving both the gate
and preservation of the contract's constraints at
every compliant active state compatible with the view.
The gate constrains the starting state $s$, while $\Theta_\tool$ requires compliance at every admitted $s'$, including during execution. At states satisfying $\pre$, the original transition relation is unchanged, so any overapproximation property is preserved there. Below we consider an example where adding a precondition to a tool is particularly useful.

{\bfseries Settlement tools.}\label{sec:goal-tools}
As illustrated in \cref{ex:contract}, the $\mathsf{settle}$ tools provide the main avenue for a contract to specify 1) the task
for which the Airlock is instantiated and 2) what should happen upon its completion. These tools combine a goal predicate
$\goal$ over $\ss$ with an execution routine
$\exec_{\mathsf{settle}}$ that performs the agreed
settlement, such as releasing an escrow or delivering
an artifact. Start with a tool $\mathsf{free\_settle}$ that has no caller-supplied arguments and whose execution routine performs this settlement, then impose $\goal$ as its precondition: $\mathsf{settle} = \mathsf{free\_settle}|_{\goal}$. Its model includes states during settlement, such as a payment in progress, as well as completed or failed outcomes. Suppose the contract certificate establishes that $\mathsf{free\_settle}$ is admissible at every compliant active state where $\goal$ holds. The obligation then reduces to
\begin{equation}\label{eq:settlement-obligation}
    \Obl_{\settle}(v) ~ \Longleftrightarrow ~
    \forall s\in\Sigma_{\mathrm{act}}.\;\bigl(\clause(s)\wedge\obs(s)=v\bigr)\implies\goal(s),
\end{equation}
namely that $\settle$'s obligation holds exactly when every compliant active state compatible with the view satisfies the goal. In particular, when goal achievement can be determined by $v=\obs(s)$ alone, we have $\Obl_{\settle}(\obs(s)) \Longleftrightarrow\goal(s)$ for any active state $s$ satisfying $C(s)$. For a delayed call, the goal must hold at every active state admitted by the control window. Preservation during settlement must already have been established in the certificate: for example, the clauses must allow a payment in progress, accounting for funds in transit, rather than only its final $\mathsf{settled}$ state.


{\bfseries Expert tools.}
Some clauses in the contract depend on judgments that the parties cannot readily express as a formal predicate: whether a document meets a requirement, an image shows a defect, or a hotel is suitable. The parties can instead agree on an \emph{expert} to make these assessments. The tool's routine $\exec_{\mathsf{expert}}$ invokes the ``expert'': this can be an LLM with a fixed prompt specified by the contract to extract information from a document or assess its contents, and records the answer. A vision model or human reviewer can play the same role. The model $\tran_\tool$ covers the states while the assessment is pending and the possible recorded answers, as well as costs or disclosures incurred by consulting the expert. For example, a contract might reserve parts of the state to record the results of an expert's calls to be used later on during task execution. In the travel contract, a booking requires a recorded hotel rating of at least 7; a pending assessment provides no such evidence. Recorded verdicts can be reused via memoization for the same object and evidence to avoid repeated sampling. Expert tools thus allow contracts to combine formal conditions with assessments of unstructured information and qualitative requirements.

{\bfseries Communication tools.}
E-mails, document uploads, and messages between agents add to a history of exchanges. A contract may constrain the recipient, the content disclosed, or the conditions under which a message may be sent. As a simple example, consider a tool $\mathsf{send}$ with input space $I_{\mathsf{send}} = \mathit{Doc} \times \mathit{Rec}$, whose routine $\exec_{\mathsf{send}}(d,r)$ sends document $d$ to recipient $r$. Let $\mathit{Sent}(s)$ record the document--recipient pairs disclosed so far, including partial disclosures. If no other disclosures occur during the call, a simple model is
\begin{equation}
    \tran_{\mathsf{send}}((d,r),s,s') \quad:=\quad
    \mathit{Sent}(s)\subseteq\mathit{Sent}(s')\subseteq\mathit{Sent}(s)\cup\{(d,r)\}.
\end{equation}
This admits both states before disclosure and states in which some or all of $d$ has been disclosed to $r$, whether or not delivery eventually succeeds. Let the contract constrain disclosure by $\clause(s):=\mathit{Sent}(s)\subseteq\mathit{Allowed}$. Assuming the state space admits the updated record for every state and input, $\mathsf{send}$ is enabled and, at a compliant active state $s$, its obligation reduces to
\begin{equation}
    \Obl_{\mathsf{send}}((d,r),\obs(s)) \quad \iff \quad (d,r) \in \mathit{Allowed}.
\end{equation}
More detailed models can distinguish which contents are disclosed and through which routes; a failed delivery does not undo information already disclosed.

{\bfseries Query tools.}
A query tool performs operations such as web searches, database queries, and measurements. Its routine $\exec_{\mathsf{query}}$ issues a query $q$ and records the response with its source and, where relevant, time or version. Information exposed through the resulting view can support subsequent proofs: in the travel contract, a booking confirmation from the agreed service may provide evidence for goal achievement. Starting from a general tool $\mathsf{query}$, each contract can impose its own restrictions through $\mathsf{query}|_\pre$, retaining the same input space and execution routine. For example, a contract may require queries not to modify the queried data and to respect a spending limit $B$:
\begin{equation}\label{eq:query-precondition}
\pre(q,s):=\mathsf{ReadOnly}(q)\wedge
\bigl(\mathsf{spent}(s)+\kappa(q)\le B\bigr).
\end{equation}
Here $\mathsf{ReadOnly}(q)$ expresses the first requirement and $\mathsf{spent}(s)$ is cumulative expenditure. The bound $\kappa(q)$ applies throughout execution: every state $s'$ admitted by $\tran_{\mathsf{query}}(q,s,s')$ satisfies $\mathsf{spent}(s')\leq\mathsf{spent}(s)+\kappa(q)$, including while waiting for a response or after failure. By \cref{eq:precondition-obligation}, the proof must establish this gate together with admissibility at every compliant active state compatible with the view.

The model must include states before a response is available, partial responses, errors, and timeouts as well as successful completion. Enabledness only requires that some state is admitted; it guarantees neither completion nor a successful answer. In particular, a timeout does not establish that no matching record exists. Records support only the claims justified by the model: a response may describe a particular database version, returned text need not be true, and reported availability is not a reservation. Read-only queries may still incur costs or disclose information, and the queried data may change independently during the call; these changes must also be modeled.

\PCLayer{Guarantees}
\section{Guarantees}
\label{sec:guarantees}

\Cref{sec:model} presented the structure of PC2 and the obligations governing tool execution. We now study the guarantees that follow from this structure. We first show that the state satisfies $\clause$ at every instant, regardless of the requests and proposed proofs supplied to the Airlock (\cref{sec:coherence}). \Cref{sec:risks} relates the state to the world it describes and identifies conditions under which compliance implies safety.

\subsection{Anytime compliance to the contract}\label{sec:coherence}
Given a contract $\contract := (\Sigma, \clause ,\obs, \mathcal T, \toolctrl, \terminal)$, we identify sufficient conditions for the state to satisfy $\clause$ at every instant of the execution of \cref{alg:pc2}, for arbitrary incoming requests and proposed proofs.

For a given execution, we write $s(\tau)$ for the possibly unobserved state of the system at time $\tau$, in the sense of \cref{sec:model}. As described in \cref{sec:airlock-core}, the Airlock invokes control and tool routines sequentially. Each execution interval starts when a routine is invoked and ends when it returns, or continues indefinitely if it does not return. During this interval the state may change through the routine's actions and outside changes (see \cref{fig:world-time}).

We begin with the proof checks introduced in \cref{sec:airlock-core}. Before executing a requested tool call, the Airlock forms its obligation from the contract and the published view, and passes it, together with the supplied proof, to the \emph{checker}: a program that checks whether the supplied argument establishes the required statement according to the rules of a fixed logic. For our implementation, we choose \emph{dependent type theory} as the underlying logic: it provides a common language for the contract's data, statements, and proofs, with proof checking reduced to type checking. A statement is represented by a type, and the checker verifies that the supplied proof is an expression of that type. In Lean, used for our mechanization (\cref{sec:lean}), this check is performed by the Lean \emph{kernel} \citep{demoura2021lean4}.

\begin{assumption}[Sound checking]\label{ass:sound-checking}
For any proof $\pi$ and statement $\Phi$, if the proof-checker accepts $\pi$ as a proof of $\Phi$, then $\Phi$ holds under the axioms and interpretation of the chosen logic.
\end{assumption}

\Cref{ass:sound-checking} is independent of the choice of proof assistant. The checker must use the contract's agreed definitions and permitted axioms. Accepting a proof of $\Obl_{\tool}^{\leqslant n}(i,v)$ then establishes that obligation, allowing the analysis below to rely on it regardless of how the proof was produced. The obligation, however, is expressed in terms of the transition models $\tran_\tool$ and $\tran_\toolctrl$, while the state evolves through the execution routines and the outside changes. To apply the proof to execution, the models must cover every state that execution can reach. We denote these possibilities by the \emph{effect relation} $E_\tool$ on $I_\tool\times\Sigma\times\Sigma$: $E_\tool(i,s,s')$ holds if and only if an invocation of $\exec_\tool(i)$ starting in $s$ can reach $s'$ during its execution. This includes initial, intermediate and final states, outside changes during the invocation, and states reached if the routine never returns.

\begin{assumption}[Coverage of state transitions]\label{ass:transition-coverage}
For every $\tool \in {\cal T} \cup \{\toolctrl\}$,
\begin{equation}\label{eq:ETleqDT}
 \forall i \in I_\tool,\; s,s' \in \ss \qquad \Xi_\tool(i,s) \wedge E_{\tool}(i,s,s') \implies \tran_{\tool}(i,s,s').
\end{equation}
\end{assumption}

\Cref{ass:transition-coverage} requires $\tran_\tool$ to \emph{over-approximate} these possibilities wherever $\Xi_\tool(i,s)$ holds: it may admit states through which the state never passes, but must include every state through which it may pass during the operation, not only the one at its end. The restriction to enabled calls due to $\Xi_\tool$ allows $\tran_\tool$ to exclude inputs and starting states, as in the precondition construction of \cref{sec:tool-families}; the obligation then prevents authorization of those calls. The Airlock uses $\tran_\tool$ to form obligations and does not need to know $E_\tool$. The role of $E_\tool$ in the analysis is to specify the behavior that the model must cover for these proofs to apply to execution. Both relations concern the state; \cref{sec:risks,sec:assurance} discuss how a deployment's routines, services and host support this coverage and the state's correspondence with the world.

The obligation preserves compliance from a state that already satisfies $\clause$. To apply it throughout the loop, the loop must start from a compliant state, and the control step, which the Airlock runs by itself, must preserve compliance too. Nobody has observed the system at signing, so its initial state is not known to comply. The Airlock does not assume it: the contract begins inactive, and the control step, as initializer, makes it active only in a compliant state.

\begin{assumption}[Initialization and control]\label{ass:initial-control}
The control tool satisfies
\begin{align*}
\forall s \in \Sigma_{\mathrm{act}}&\qquad \clause(s) \implies \mathcal A_\toolctrl(s), &&\text{(controller)}\\
\forall s\in\Sigma_{\mathrm{stop}},\ s'\in\Sigma_{\mathrm{act}}&\qquad \tran_\toolctrl(s,s') \implies \clause(s'), &&\text{(initializer)}
\end{align*}
and it is enabled at every inactive state: $\Xi_\toolctrl(s)$ for every $s\in\Sigma_{\mathrm{stop}}$.
\end{assumption}

The two lines are the two roles of $\toolctrl$. As controller, from a compliant active state, it passes only through compliant states. As initializer, from an inactive state, it can make the contract active only in a compliant state. Both are certified once, at startup (\cref{sec:airlock-core}), so no new proof is required at each cycle, and the second replaces the assumption that the initial state complies, which nobody could check. Enabledness makes \cref{ass:transition-coverage} applicable to every control step as well as to calls.

Finally, the argument relies on the Airlock running \cref{alg:pc2} as specified.

\begin{assumption}[Execution integrity]\label{ass:execution-integrity}
Execution follows \cref{alg:pc2} and the proof-checking routine in \cref{app:proof-check-program}. In particular:
\begin{enumerate}[label=(\roman*),leftmargin=2.2em,itemsep=0pt,topsep=2pt]
\item execution is a sequence of contiguous operations, each a tool call or a control step: the first starts when the Airlock has read the state, and each of the others when the previous one ends;
\item after each routine returns, the current state is read and appended to the log; views and the inactivity test are computed from this reading, and initially from a reading taken when the Airlock starts;
\item a call runs only for the pending request $(\tool,i,k_0,n,\pi)$, at most once, if no other call ran since $v[k_0]$ was published, in an iteration $k$ with $k_0\le k\le k_0+n$, and after the proof-checking routine in \cref{app:proof-check-program} has checked that $\tool\in\mathcal T$, that $i\in I_\tool$, and that $\pi$ proves $\Obl_{\tool}^{\leqslant n}(i,v[k_0])$, a statement the Airlock forms itself;
\end{enumerate}
\end{assumption}

\PCAssumptionRef[(i)]{ass:execution-integrity} lets us follow the operations one after another: the state at the end of one is also the state at the start of the next. In practice, between two operations the Airlock reads the state, publishes a view and invokes the next routine, which takes microseconds, while requests take seconds to prepare. Proof checking runs in parallel with control, so the state need not remain fixed while a proof is checked. \Cref{app:intervals} discusses executions with bounded gaps between routines.

\PCAssumptionRef[(ii)]{ass:execution-integrity} connects execution to the information used by the loop. After each routine, the Airlock records the current state. Whenever it publishes a view or tests whether the session is still active, it uses the latest reading. Thus these decisions reflect the result of the routine that just ran.

Together, \PCAssumptionRef[(i--ii)]{ass:execution-integrity} ensure that the latest recorded state is the state from which the next routine starts. Part (iii) then connects that invocation to its authorization: the checked proof concerns the requested tool and input, and the call runs only within its control window, with no intervening tool call or reuse of the request.

A request may refer to a view published several control steps before its call starts. To apply its delayed obligation, we must show that the actual starting state is among those covered by that obligation. The following lemma connects the request's counter window to the states admitted by $\tran_\toolctrl^{\leqslant n}$.

\begin{lemma}[Window]\label{lem:window}
Assume \cref{ass:transition-coverage,ass:initial-control,ass:execution-integrity}. Let a call run for the request $(\tool,i,k_0,n,\pi)$ from state $s'$, and let $\bar s$ be the state from which $v[k_0]$ was computed. If $\clause$ holds at every instant from the start of the contract to the invocation, then $\bar s$ and $s'$ are active and $\tran^{\leqslant n}_\toolctrl(\bar s,s')$.
\end{lemma}

\begin{proof}
Let $k$ be the iteration in which the call runs. By \PCAssumptionRef[(iii)]{ass:execution-integrity}, $k_0\le k\le k_0+n$, and only control invocations intervene between $\bar s$ and the call's starting state $s'$. There are $m=k-k_0\le n$ such invocations. By \PCAssumptionRef[(i--ii)]{ass:execution-integrity}, they are contiguous, and their final states are the recorded readings. For $m=0$, $\bar s=s'$. The loop computes a view and starts an iteration only from a reading that it has found active (\PCAssumptionRef[(ii)]{ass:execution-integrity}). So $\bar s$, the starting state of each of these control steps, and $s'$ are active. Each control step started from a compliant state by hypothesis, since it runs after the contract has started. Hence \cref{ass:initial-control} gives its enabledness, and \cref{ass:transition-coverage} places its final state in $\tran_\toolctrl$ from its starting state. Chaining these steps through active states gives $\tran^{\leqslant m}_\toolctrl(\bar s,s')$, hence $\tran^{\leqslant n}_\toolctrl(\bar s,s')$, since $m\le n$ and each window contains the previous one (\cref{def:delayed-obligation}).
\end{proof}

The lemma explains why the state need not remain unchanged after a view is published: given compliance up to the call, its delayed obligation covers the actual starting state despite the intervening control steps. The window bounds the number of those steps, not the elapsed time.

\begin{theorem}[Anytime compliance to the contract]\label{thm:commitment-final}
Under \cref{ass:sound-checking,ass:transition-coverage,ass:initial-control,ass:execution-integrity}, for any sequence of tool-call requests and proposed proofs supplied to \cref{alg:pc2}, every resulting execution in which the contract starts satisfies $\clause$ at every instant from the start until the contract closes, closing included, or at every instant after the start if it never closes.
\end{theorem}

\begin{proof}
By \PCAssumptionRef[(i)]{ass:execution-integrity}, operations are contiguous. Let $s_0$ be the state when the Airlock starts. Initialization runs a control step from $s_0$, which is inactive, since otherwise the Airlock rejects the contract (\cref{ass:execution-integrity}). By \cref{ass:initial-control} the control step is enabled at $s_0$, so \cref{ass:transition-coverage} places the state $s_1$ at its end in $\tran_\toolctrl(s_0,\cdot)$. Since the contract starts, $s_1$ is active, and \cref{ass:initial-control} gives $\clause(s_1)$. It therefore suffices to prove, by induction on the invocations after the start, that each starts in a compliant state and preserves compliance throughout its execution. The first starts from $s_1$. Consider an operation that starts from the state $s$, with $\clause(s)$ by induction; $s$ is active, since \cref{alg:pc2} runs operations only from such states.

(1) A control step. \Cref{ass:initial-control} gives $\mathcal A_\toolctrl(s)$, that is $\Xi_\toolctrl(s)$ and $\Theta_\toolctrl(s)$. Every state $s'$ reached during its interval satisfies $E_\toolctrl(s,s')$, so \cref{ass:transition-coverage} gives $\tran_\toolctrl(s,s')$, hence $\clause(s')$.

(2) A call $\tool(i)$, for the request $(\tool,i,k_0,n,\pi)$. By \PCAssumptionRef[(iii)]{ass:execution-integrity}, the checker accepted $\pi$ as a proof of $\Obl_{\tool}^{\leqslant n}(i,v[k_0])$, which therefore holds by \cref{ass:sound-checking}. Let $\bar s$ be the state from which $v[k_0]$ was computed: $\obs(\bar s)=v[k_0]$, and $\clause(\bar s)$ by induction. By \cref{lem:window}, $\bar s$ and $s$ are active and $\tran^{\leqslant n}_\toolctrl(\bar s,s)$. The obligation, applied to $\bar s$ and $s$, gives $\mathcal A_\tool(i,s)$, that is $\Xi_\tool(i,s)$ and $\Theta_\tool(i,s)$. Every state $s'$ reached during the call's interval satisfies $E_\tool(i,s,s')$, so \cref{ass:transition-coverage} gives $\tran_\tool(i,s,s')$, hence $\clause(s')$.

In both cases, compliance holds throughout the invocation. If it returns, its final state is compliant and, by \PCAssumptionRef[(i)]{ass:execution-integrity}, it is the state from which the next operation, if any, starts. If it never returns, the same argument applies throughout its continuing execution. A rejected or expired request authorizes no call. The argument makes no assumption about how requests or proofs are produced.
\end{proof}

\Cref{thm:commitment-final} derives compliance of the whole execution from the individual obligations. \emph{What, then, does each obligation require?} By \cref{def:delayed-obligation}, $\Obl_{\tool}^{\leqslant n}(i,v)$ holds exactly when the call is admissible from every active state admitted by at most $n$ control steps, through active states, from a compliant active state with view $v$. It therefore accounts both for information hidden by the view and for changes during the control window. This condition is relative to the chosen models, which may admit states and paths that actual execution cannot follow. Establishing admissibility from all these states suffices for compliance, but a false obligation need not imply that the call would actually violate the contract. Relative to the models, nothing weaker would do: every rule that safely admits calls on the basis of the view implies the obligation, so the agent may follow any strategy whose calls keep the contract (\cref{prop:wp}).

The theorem allows arbitrary requests and proposed proofs, but its coverage assumption still concerns every state that execution can reach, including through outside interactions. \Cref{sec:establishing-assumptions} examines how restricting the agent's access can make this assumption easier to establish while retaining models that allow useful calls.

\subsection{Completion through goal tools}\label{sec:completion}
A standard contract (\cref{ex:contract}) ends in one of three ways: it is void if the conditions to start do not hold, settled when the agent proves its goal, or terminated by the control step, for instance at its deadline with a refund. The following theorem shows that it ends in at most one of them, and in each only as agreed; \cref{cor:every-agent} adds that, with a deadline, it does end. Besides \cref{ass:sound-checking,ass:transition-coverage,ass:initial-control,ass:execution-integrity}, the theorem uses that the settlement tool is $\settle=\mathsf{free\_settle}|_{\goal}$ (\cref{sec:tool-families}), that the control step terminates the contract exactly when its decision reading is in the region $\mathit{End}$ of \cref{sec:control-tool}, and the properties of the label stated in \cref{ex:contract}.

\begin{theorem}[Completion]\label{thm:completion}
Assume \cref{ass:sound-checking,ass:transition-coverage,ass:initial-control,ass:execution-integrity}. In every execution of \cref{alg:pc2} with a standard contract, the contract ends in at most one of the following ways:
\begin{enumerate}[label=(\arabic*),leftmargin=2.2em,itemsep=0pt,topsep=2pt]
\item \emph{void}: the conditions to start did not hold; the execution never starts and no tool-call is made;
\item \emph{settled}: the agreed settlement procedure ran, from a state where the goal held;
\item \emph{terminated}: the agreed termination procedure ran, after a control step found a termination condition.
\end{enumerate}
\end{theorem}

\begin{proof}
The contract is created pending, and only the routines of $\toolctrl$ and $\settle$ change the label, at their end (\cref{ex:contract}). The first operation is a control step, which, as initializer, sets the label to $\mathsf{ongoing}$ when its decision reading is in $\mathit{Ready}$, and to $\mathsf{void}$ otherwise; in the second case the loop stops, and no call is made. Once the label is $\mathsf{settled}$ or $\mathsf{terminated}$, the contract is inactive and \cref{alg:pc2} runs no further operation, so the contract ends in at most one way. (2) As in case (2) of the proof of \cref{thm:commitment-final}, the checker accepted a proof of $\Obl_{\settle}^{\leqslant n}(v[k_0])$, which therefore holds (\cref{ass:sound-checking}). Admissibility of $\settle$ includes its enabledness, hence $\goal$, by the precondition construction, at every active state of the window, and \cref{lem:window} places in the window the state from which the settlement started. (3) The label becomes $\mathsf{terminated}$ only when the control step has run the termination procedure, which it does exactly when its decision reading is in $\mathit{End}$.
\end{proof}

The case analysis rests on the code of the Airlock and of the control step; case (2) also rests on the checker and on the obligation of $\settle$.

At the call-selection step of an active cycle, a pending settlement request whose proof has been accepted and whose window remains valid is selected for execution without further authorization. The client therefore has no discretionary veto over an authorized settlement through the Airlock. The next proposition identifies exactly what its proof must establish.

\begin{proposition}[What settlement asks]\label{prop:settle}
Assume \cref{ass:initial-control} for the standard contract, and let its certificate establish that $\mathsf{free\_settle}$ is admissible at every compliant active state where $\goal$ holds. Then, for every view $v$ and delay $n$,
\[
\Obl_{\settle}^{\leqslant n}(v) \iff \forall s,s'\in\Sigma_{\mathrm{act}} \quad \bigl(\obs(s)=v \wedge \clause(s) \wedge \tran^{\leqslant n}_\toolctrl(s,s')\bigr) \implies \goal(s').
\]
\end{proposition}

\begin{proof}
By the precondition construction, $\Xi_\settle(s') \iff \goal(s')\wedge\Xi_{\mathsf{free\_settle}}(s')$ and $\Theta_\settle(s') \iff \bigl(\goal(s') \implies \Theta_{\mathsf{free\_settle}}(s')\bigr)$. From left to right, $\mathcal A_\settle(s')$ gives $\Xi_\settle(s')$, hence $\goal(s')$. From right to left, $\clause(s')$ holds, since $s'$ is reached from the compliant state $s$ by control steps from active states, which preserve $\clause$ by \cref{ass:initial-control}. Together with $\goal(s')$, the certificate gives $\mathcal A_{\mathsf{free\_settle}}(s')$, hence $\mathcal A_\settle(s')$.
\end{proof}

When every compliant active state with the given view satisfies the goal and modeled control steps preserve it throughout the window, the right-hand side holds. A proof of this condition suffices for authorization while the request remains valid; admissibility of the settlement procedure is already established by the contract certificate.

\section{From compliance to safety}\label{sec:risks}

\Cref{sec:guarantees} establishes that execution remains within the contract's constraints. We now study when this compliance protects the resources and interactions the contract concerns. A payment recorded within budget, for instance, does not establish that the service charged no additional fees. Connecting the two requires both a mathematical account of what the state describes and conditions on the system that produces and acts on that description. This is the gap between a specification, which a machine can be built to meet, and a requirement about the world in which the machine operates \citep{jackson1995world,zave1997four}. A proof relates a system to a model of it, never the model to the world \citep{smith1985limits}.

Three questions guide this section and the next:
\begin{itemize}[leftmargin=*,itemsep=3pt,topsep=3pt]
\item \emph{What does the mathematical construction guarantee?} Under which assumptions does compliance of the state imply safety in the world?
\item \emph{How can those assumptions be established for a deployment?} What must the routines, host, access controls and external services do, and what evidence supports these claims?
\item \emph{What do the parties do about what remains assumed?} Which assumptions do they accept, who undertakes responsibility for them, and what happens if they fail?
\end{itemize}
\Cref{sec:world} addresses the first question, and illustrates on an example what a world is. \Cref{sec:establishing-assumptions} studies how isolating the agent supports useful effect models and guarantees that apply to every agent under the same deployment conditions. \Cref{sec:execution-duration} discusses the requirements on execution. \Cref{sec:assurance} addresses the second and third questions. It shows that the assumptions of a deployment need not be established from scratch, since they follow by composing guarantees that ship with the capability modules from which its tools are built (\cref{sec:composition}). It then assembles these dependencies in a tree and records their proposed allocation to parties (\cref{sec:responsibilities}), connecting this analysis to the treatment of remaining assumptions in \cref{sec:signature}.

\subsection{World safety}\label{sec:world}

Let $\worlds$ be a space of configurations of the relevant system and environment, including such facts as actual charges, issued payment instructions and the state of external services. We write $w(\tau)$ for its configuration at time $\tau$. A relation $\corresponds\subseteq\worlds\times\ss$ specifies how the contract's state describes these facts: $\corresponds(w,s)$ means that $s$ is a valid description of $w$. This description can be partial or conservative. A recorded upper bound on expenditure, for instance, can describe several different amounts actually charged, and the same expenditure can admit several valid bounds. In the terms of program refinement, $\corresponds$ is an abstraction relation between the world and its description \citep{hoare1972proof,abadi1991refinement}. It thus differs from $\obs$, which determines what the requesting agent sees of the state.

\begin{example}[A trading agent and its world]\label{ex:worlds}
Consider an agent that manages a small investment portfolio. Its tools read and send email, read and write files in a working directory, search, read and publish posts on a social network, read quotes from a market-data service, query a database of past transactions, and place orders through a brokerage account. A world $w\in\worlds$ is a configuration of all these systems: the messages in the mailbox and those delivered to their recipients, the files on disk, the posts that the public can read, the prices at which the securities currently trade, the rows of the database, and the cash, positions and pending orders held by the broker, together with the current time. The state $s$ records only what the contract needs, as the tools report it: for instance, the messages sent with their recipients, the posts published, the last quote read with its time, and the money committed to orders. The correspondence $\corresponds$ says what each record means in the world. The record of sent messages describes the world if every message that left the mailbox appears in it. A recorded quote describes the world if the service reported that price at that time, whatever the price is now. The recorded commitment may be a conservative bound, counting each pending order at its maximum cost. A protection such as ``never commit more than a given amount, and never disclose the portfolio publicly'' is a predicate $\Safe$ on worlds: it concerns the money actually committed and what the public can actually read, not the records. Much of the world changes without the agent: messages arrive, prices move, other applications update the database, and the broker executes pending orders.
\end{example}

In most scenarios, such as the example above, viewing a world $w\in\worlds$ as a single object is impractical. This is because applications are often organized in largely independent systems, so the resulting world can be naturally represented as a product of their configurations, $\worlds=\worlds_{\mathrm{mail}}\times\worlds_{\mathrm{files}}\times\dots\times\worlds_{\mathrm{broker}}$, and the state a product of the corresponding records. In \cref{sec:assurance} we will address this question and show how a world, a state, a correspondence and the tools that act on them can be assembled from small modules, each concerning a single service such as the mailbox or the brokerage account, whose guarantees are established once.

Let $\Safe$ be a predicate on $\worlds$ specifying the intended protection, as in \cref{ex:worlds} above. Worlds may include history when the protection concerns past events, such as a document having been disclosed, as history variables do in proofs of refinement \citep{abadi1991refinement}. Compliance is expressed by $\clause(s)$, while safety concerns $\Safe(w)$. We connect them through two conditions. \emph{Fidelity} requires the state to remain a valid description of the world as execution proceeds. \emph{Adequacy} requires a world described by a compliant state to satisfy the intended safety property. Together, they relate what the Airlock enforces on states to what the parties require of the world. This is the satisfaction argument of requirements engineering, in which a specification, together with properties of the world, entails a requirement \citep{zave1997four,jackson2001problem}.

\textbf{Fidelity.}
To express preservation of correspondence, we must account for how the world and state evolve together. Let $\widehat{\effect}_\tool(i,w,s,w',s')$ hold when an invocation of $\exec_\tool(i)$ starting in $(w,s)$ can reach $(w',s')$ at any instant of its execution. This includes outside changes during the invocation and applies also to a routine that does not return. The effect relation of \cref{sec:coherence} is its projection on states:
\begin{equation}\label{eq:world-projection}
\effect_\tool(i,s,s')\quad\iff\quad
\exists w,w'\in\worlds.\;\widehat{\effect}_\tool(i,w,s,w',s').
\end{equation}

This joint relation lets us state what it means for execution to preserve the description: starting from a pair satisfying $\corresponds$, every pair reached during an enabled invocation must still satisfy it, as a simulation relation is preserved by every step of the system it relates \citep{lynch1995forward}. Correspondence must also be established when the contract starts, since nothing is known of the world when the contract is signed. We collect these requirements in the following assumption.

\begin{assumption}[Fidelity]\label{ass:tool-fidelity}
For every $\tool\in\tools\cup\{\toolctrl\}$, input $i\in \inputs_\tool$, worlds $w,w'$ and states $s,s'$,
\begin{equation}\label{eq:fidelity}
\corresponds(w,s)\wedge\enabled_\tool(i,s)\wedge
\widehat{\effect}_\tool(i,w,s,w',s')\quad\implies\quad\corresponds(w',s'),
\end{equation}
and, for every inactive state $s$ and active state $s'$,
\begin{equation}\label{eq:fidelity-init}
\widehat{\effect}_\toolctrl(w,s,w',s')\quad\implies\quad\corresponds(w',s').
\end{equation}
\end{assumption}

\Cref{ass:transition-coverage} requires the model to cover the states reached during execution, while \cref{ass:tool-fidelity} requires those states to remain valid descriptions of the world. The claim that a service reported a confirmed booking at time $t$ describes a past event. Once established, it survives a later cancellation if the record is preserved and the world model retains that history. A claim that the booking is currently confirmed must instead account for possible cancellations. Fidelity thus depends on what the recorded information means, while adequacy asks whether that meaning suffices for the intended protection. The second line is the counterpart of the initializer line of \cref{ass:initial-control}: the control step makes the contract active only in a state that describes the world, whatever the world was before, so nothing needs to be assumed about the world when the contract is signed. For the control design of \cref{sec:control-tool}, a sufficient condition is that $\mathsf{read\_state}()$ returns a state that describes the world when the reading completes, and that the remaining initialization preserves this correspondence until the contract becomes active.

\textbf{Adequacy.}
Fidelity gives meaning to the state throughout execution. We must then ask whether the clauses imposed on that state suffice for the protection the parties require: even a faithful description can satisfy a contract that omits an important constraint. An agent that optimizes within the clauses pushes exactly where they and the intended protection come apart \citep{manheim2018goodhart}, and no contract specifies everything its parties care about \citep{hadfieldmenell2019incomplete}.

\begin{assumption}[Adequacy]\label{ass:contract-adequacy}
For every $w\in\worlds$ and $s\in\ss$,
\begin{equation}\label{eq:adequacy}
\corresponds(w,s)\wedge\clause(s)\quad\implies\quad\Safe(w),
\end{equation}
and, for a contract with a goal $\goal$ whose intended outcome in the world is a predicate $\goal_{\worlds}$,
\begin{equation}\label{eq:goal-adequacy}
\corresponds(w,s)\wedge\goal(s)\quad\implies\quad\goal_{\worlds}(w).
\end{equation}
\end{assumption}

For specified $\corresponds$, $\clause$ and $\Safe$, \cref{ass:contract-adequacy} can be proved when the relationships expressed by the correspondence and clauses imply the required protection. In the \hyperref[ex:payment-ledger]{payment ledger} of \cref{app:examples}, correspondence bounds actual expenditure by recorded reservations, and the clause bounds those reservations by the budget, so adequacy follows by arithmetic. The parties must also establish that $\Safe$ expresses their intended protection, since a proof of the implication alone does not settle this question.

We can now extend \cref{thm:commitment-final} from compliance of the state to safety in the world. The additional step is to show that the state is a valid description of the world throughout execution from the start of the contract, including while a routine runs.

\begin{theorem}[World safety]\label{thm:world-final}
Under \cref{ass:sound-checking,ass:transition-coverage,ass:initial-control,ass:execution-integrity,ass:tool-fidelity,ass:contract-adequacy}, every execution of \cref{alg:pc2} in which the contract starts satisfies $\Safe(w(\tau))$ at every instant from the start until the contract closes, closing included, or at every instant after the start if it never closes, for any sequence of requests and proposed proofs.
\end{theorem}

\begin{proof}
By \cref{thm:commitment-final}, $\clause(s(\tau))$ holds from the start of the contract. When the control step starts the contract, the state describes the world, by the second line of \cref{ass:tool-fidelity}. Each later operation is enabled at its starting state: control by \cref{ass:initial-control}, and a requested call by its accepted obligation and \cref{lem:window}, as in the proof of \cref{thm:commitment-final}. Hence the first line of \cref{ass:tool-fidelity} preserves correspondence throughout that operation, which, operations being contiguous (\PCAssumptionRef[(i)]{ass:execution-integrity}), lasts until the next one starts, or throughout its continuing execution if it never returns. Induction over the operations therefore gives $\corresponds(w(\tau),s(\tau))$ at every instant from the start. Together with compliance, \cref{ass:contract-adequacy} yields the conclusion.
\end{proof}

\Cref{thm:world-final} separates what must be checked for each call from what must be established about the deployment. The checker reasons about the contract's states and tool models. Fidelity keeps these states valid descriptions throughout execution, and adequacy makes their compliance sufficient for safety in the world. The argument transfers a property from the description to the world it describes, as a refinement mapping transfers a property proved of a specification to a system that implements it \citep{abadi1991refinement}. This is what allows the local obligations to protect the resources and interactions the contract concerns.

Establishing these conditions requires examining which interactions the deployment permits, and \cref{sec:mediation} considers how the agent's access affects them.

Many contracts also prescribe a settlement action once a goal has been achieved, as in the standard contract of \cref{ex:contract}. The completion guarantee of \cref{thm:completion} concerns this goal as expressed on states. The following corollary transfers it to the intended outcome $\goal_{\worlds}$ in the world, through $\corresponds$ and the second part of \cref{ass:contract-adequacy}.

\begin{corollary}[Completion in the world]\label{cor:world-completion}
Under the assumptions of \cref{thm:completion}, assume \cref{ass:tool-fidelity,ass:contract-adequacy}. Then every invocation of the settlement routine starts at an instant $\tau$ satisfying $\goal_{\worlds}(w(\tau))$.
\end{corollary}

\begin{proof}
As in the proof of \cref{thm:completion}, case (2), the accepted settlement obligation gives $\goal(s(\tau))$. By the correspondence argument in the proof of \cref{thm:world-final}, $\corresponds(w(\tau),s(\tau))$ holds, and \cref{eq:goal-adequacy} gives the conclusion.
\end{proof}

A central question is to determine when assumptions \cref{ass:tool-fidelity,ass:contract-adequacy} are satisfied. \Cref{app:examples} works through step-by-step examples where they are derived in practice, for a payment ledger for safety and booking confirmations for completion, including settlement of the travel contract.

\subsection{Complete mediation and sandboxing: safety against every agent}\label{sec:establishing-assumptions}\label{sec:mediation}

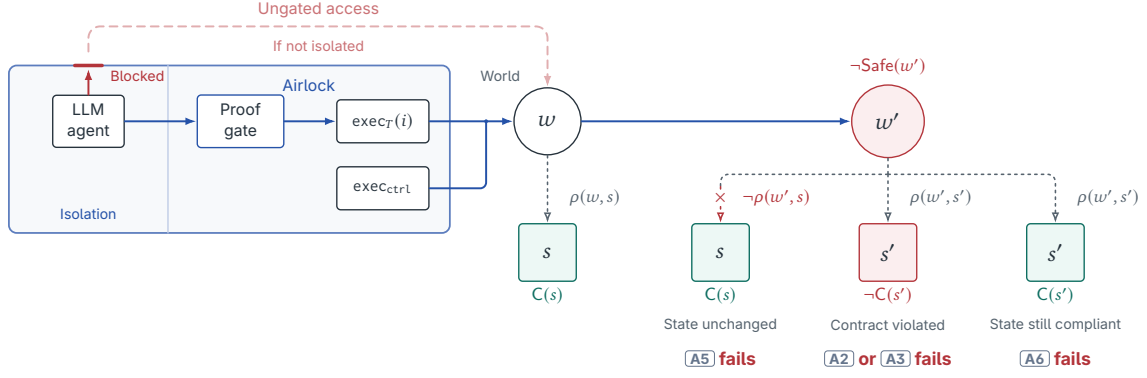
\begin{figure}[t]
\centering
\begingroup
\definecolor{ink}{HTML}{26323C}
\definecolor{muted}{HTML}{56636E}
\colorlet{blue}{pcAgentText}
\colorlet{panelfill}{pcControlFill!35}
\colorlet{red}{pcDangerText}
\colorlet{redfill}{pcDanger!9}
\colorlet{green}{pcCheckText}
\colorlet{greenfill}{pcCheckFill!55}

\let\C\clause
\let\ctrl\toolctrl
\resizebox{\linewidth}{!}{%
\begin{tikzpicture}[
  x=1cm,y=1cm,yshift=-2.2cm,
  font=\sffamily\fontsize{9}{11}\selectfont,text=ink,
  line width=0.7pt,line cap=round,line join=round,
  >={Latex[length=2mm,width=1.4mm]},
  flow/.style={->,draw=blue,line width=1pt},
  correspondence/.style={draw=muted,line width=0.8pt,
    dash pattern=on 1.2pt off 2.2pt},
  correspondence arrow/.style={correspondence,
    -{Latex[open,length=1.9mm,width=1.4mm]}},
  note/.style={font=\sffamily\fontsize{8}{10}\selectfont,text=muted},
  routine/.style={draw=ink,fill=white,rounded corners=2pt,
    minimum width=1.65cm,minimum height=0.77cm,inner sep=4pt,align=center},
  world/.style={circle,draw=ink,fill=white,minimum size=1.2cm,
    inner sep=0pt,font=\fontsize{12}{14}\selectfont},
  state/.style={pc state,minimum size=1.03cm,inner sep=0pt,
    font=\fontsize{12}{14}\selectfont},
  failed/.style={text=red,font=\sffamily\bfseries\fontsize{10}{12}\selectfont}
]
\path[use as bounding box] (0,2.2) rectangle (21.0,9.3);

\draw[blue!75,fill=panelfill,line width=0.9pt,rounded corners=4pt]
  (0.25,5.0) rectangle (8.15,8.0);
\draw[blue!30,line width=0.7pt] (3.1,5.0) -- (3.1,8.0);
\node[text=blue] at (5.625,7.64) {Airlock};
\node[note,text=blue] at (1.675,5.35) {Isolation};

\node[routine,minimum width=1.28cm]
  (agent) at (1.675,7.0) {LLM\\agent};
\node[routine,draw=blue,minimum width=1.55cm]
  (gate) at (4.4,7.0) {Proof\\gate};
\node[routine] (tool) at (6.95,7.0) {$\exec_T(i)$};
\node[routine] (control) at (6.95,5.8) {$\exec_{\ctrl}$};
\draw[flow] (agent.east) -- (gate.west);
\draw[flow] (gate.east) -- (tool.west);

\node[world] (world) at (9.9,7.0) {$w$};
\node[world,draw=red,fill=redfill,minimum size=1.3cm]
  (next) at (16.0,7.0) {$w'$};
\node[note,anchor=east] at (9.55,7.81) {World};
\node[text=red] at (16.0,7.98) {$\neg\Safe(w')$};
\coordinate (routinejoin) at (8.8,7.0);
\draw[blue,line width=1pt] (tool.east) -- (routinejoin);
\draw[blue,line width=1pt,rounded corners=3pt]
  (control.east) -- (8.8,5.8) -- (routinejoin);
\fill[blue] (routinejoin) circle[radius=0.035cm];
\draw[flow] (routinejoin) -- (world.west);
\draw[flow] (world.east) -- (next.west);

\draw[->,red,line width=1pt] (agent.north) -- (1.675,7.94);
\draw[red,line width=1.8pt] (1.43,8.0) -- (1.92,8.0);
\node[note,anchor=west,text=red,inner sep=2pt]
  at (1.99,7.81) {Blocked};
\draw[->,red!40,line width=1pt,dash pattern=on 4pt off 3pt,rounded corners=6pt]
  (1.675,8.25) -- (1.675,8.7) -- (9.9,8.7) -- (world.north);
\node[anchor=south,text=red!70,inner sep=4pt] at (5.8,8.7)
  {Ungated access};
\node[note,text=red!70] at (5.8,8.35) {If not isolated};

\node[state,draw=green,fill=greenfill] (initial) at (9.9,4.65) {$s$};
\node[text=green] at (9.9,3.88) {$\C(s)$};
\draw[correspondence arrow] (world.south) -- (initial.north);
\node[anchor=west,text=muted] at (10.15,5.63) {$\rho(w,s)$};

\node[state,draw=green,fill=greenfill] (sone) at (13.0,4.65) {$s$};
\node[state,draw=red,fill=redfill] (stwo) at (16.0,4.65) {$s'$};
\node[state,draw=green,fill=greenfill] (sthree) at (19.0,4.65) {$s'$};
\draw[correspondence] (next.south) -- (16.0,6.05);
\draw[correspondence,rounded corners=4pt]
  (16.0,6.05) -- (13.0,6.05) -- (13.0,5.83);
\draw[correspondence arrow,draw=red] (13.0,5.83) -- (sone.north);
\node[text=red,fill=white,inner sep=1.5pt] at (13.0,5.63) {$\times$};
\draw[correspondence arrow] (16.0,6.05) -- (stwo.north);
\draw[correspondence arrow,rounded corners=4pt]
  (16.0,6.05) -- (19.0,6.05) -- (sthree.north);
\node[anchor=west,text=red] at (13.25,5.63) {$\neg\rho(w',s)$};
\node[anchor=west,text=muted] at (16.25,5.63) {$\rho(w',s')$};
\node[anchor=west,text=muted] at (19.25,5.63) {$\rho(w',s')$};
\node[text=green] at (13.0,3.88) {$\C(s)$};
\node[text=red] at (16.0,3.88) {$\neg\C(s')$};
\node[text=green] at (19.0,3.88) {$\C(s')$};
\node[note] at (13.0,3.35) {State unchanged};
\node[note] at (16.0,3.35) {Contract violated};
\node[note] at (19.0,3.35) {State still compliant};
\node[failed] at (13.0,2.75) {\cref*{ass:tool-fidelity} fails};
\node[failed] at (16.0,2.75) {\cref*{ass:transition-coverage} or \cref*{ass:initial-control} fails};
\node[failed] at (19.0,2.75) {\cref*{ass:contract-adequacy} fails};
\end{tikzpicture}}
\endgroup
\caption{What can go wrong without isolation. An action outside the proof gate (pale path) can make the world unsafe while control runs. Depending on how the state responds, one of the assumptions fails, as the three branches show (\cref{app:no-isolation}).}
\label{fig:mediation}
\end{figure}

\textbf{The agent's access.}
The discussion so far has left one actor implicit: the agent, which may be adversarial. Suppose that the agent could act on the state or on the world independently of the Airlock, for instance by deleting or sending messages in a mailbox while a tool that reads the mailbox is running. Since $\effect_\tool$ and $\widehat{\effect}_\tool$ describe everything that can happen during an invocation of $\tool$, they would then include the agent's direct actions, alongside the outside changes of the world. The possible effects can therefore depend on which agent is present. We write $\effect^\alpha_\tool$ and $\widehat{\effect}^\alpha_\tool$ for the effect relations of $\tool$ when the agent is $\alpha$.
Coverage (\cref{ass:transition-coverage}) requires each model $\tran_\tool$ to cover these additional changes, and fidelity (\cref{ass:tool-fidelity}) requires the correspondence to survive them. One option is to model the particular agent, tying the justification of the tools to that agent's behaviour. The other is to use a common model that covers every change any agent could cause. This avoids verifying the agent, but if those possibilities include a contract violation during every invocation, no call can be authorized, and control itself cannot satisfy \cref{ass:initial-control}.

An agent can also change the world without changing the state. If that change invalidates what the state describes, fidelity fails, and enlarging $\tran_\tool$ alone cannot repair it. \Cref{fig:mediation} illustrates these failures, and \cref{app:no-isolation} works through them in detail. The classical answer to this problem is to restrict the agent's access instead of modeling it. This is sufficient, though, as we discuss at the end of this section, not necessary.

Controlling access in this way has a long history in computer security. In particular, the principle of \emph{complete mediation} requires each access to the objects that a system protects to pass through the authorization that governs it \citep{saltzer1975protection}. Here the protected objects are the state and the world, and the authorization is the proof gate: the Airlock checks an obligation before executing a requested operation (\PCAssumptionRef[(iii)]{ass:execution-integrity}). Mediation is complete only if the agent has no other way to act, which is what isolation, stated below, provides. The Airlock then acts as a reference monitor for the agent's access \citep{anderson1972computer}.

To express what this restriction provides, we take the agent to comprise every process that runs for it: a model, the code that drives it, and any process they start. We then fix a tool $\tool$ and its input $i$, and consider its routine running with that input while none of the agent's processes runs. Other parties, external services and the passage of time remain present, under the same conditions on their behaviour. We call these \emph{reference executions}. As with the effect relations of \cref{sec:coherence,sec:world}, $\effect^\bullet_\tool(i,s,s')$ holds when such an invocation starting in $s$ can reach $s'$ at any instant, and $\widehat{\effect}^\bullet_\tool(i,w,s,w',s')$ records the corresponding possibility for a world--state pair. The following assumption specifies how the agent can be present without introducing effects beyond those of these executions.

\begin{assumption}[Isolation of the agent]\label{ass:isolation} The agent acts on the system only through invocations selected by the Airlock among its requests, in particular:
\begin{enumerate}[label=(\roman*),leftmargin=2.2em,itemsep=0pt,topsep=2pt]
\item Nothing else it does, through any channel (e.g. files, network connections, shared memory or messages to other processes) changes the world, the state or the internal state of a running routine, during or between invocations of any $\tool\in\tools\cup\{\toolctrl\}$ routine.
\item Each tool or $\toolctrl$ routine receives from the agent only the call input, fixed at invocation.
\item The evolution of the world and the state during the invocation depends only on their current values, the routine's internal state and the fixed input. In particular it does not depend on the agent's computation or internal state.
\end{enumerate}
\end{assumption}

The assumption concerns how the system is organized, rather than whether the agent behaves cooperatively. It allows the agent to choose requests and prepare proofs, but excludes both direct changes to the protected system and further influence on a running routine through shared data or pending requests. The independence condition lets us compare an execution with an agent to one of the reference executions. The prototypical way to establish \cref{ass:isolation} is to run the agent in a sandbox, a container or lightweight virtual machine \citep{goldberg1996secure,agache2020firecracker} with no network route, except to its model provider if the contract allows it, and no access to the Airlock's files, memory or credentials, which stay with the routines. The implementation is documented in the \PCRepo{}.

\begin{lemma}[Isolation]\label{lem:isolation}
Under isolation (\cref{ass:isolation}), for every agent $\alpha$, adversarial or not, for every $\tool\in\tools\cup\{\toolctrl\}$, and all inputs, worlds and states,
\[
\effect^\alpha_\tool(i,s,s') \implies \effect^\bullet_\tool(i,s,s'),
\qquad
\widehat{\effect}^\alpha_\tool(i,w,s,w',s') \implies
\widehat{\effect}^\bullet_\tool(i,w,s,w',s').
\]
Consequently, coverage and fidelity established for the reference executions imply \cref{ass:transition-coverage,ass:tool-fidelity} for every agent.
\end{lemma}

\begin{proof}
Fix an agent $\alpha$. Consider an execution in which the agent $\alpha$ runs, and in it an invocation of $\tool$ on input $i$ that starts in the world $w$ and the state $s$. For each instant $\tau$ of the invocation, write $w(\tau)$, $s(\tau)$ and $r(\tau)$ for the world, the state and the internal state of the routine at $\tau$. By \cref{ass:isolation}, nothing the agent does changes these three quantities, and how they can evolve depends only on their current values and on $i$, not on the agent. The trajectory $\tau\mapsto(w(\tau),s(\tau),r(\tau))$ is therefore one that a reference execution, started in $w$ and $s$ with the same input, can also produce. Every pair $(w(\tau),s(\tau))$ is thus reached in a reference execution, which proves the implication for $\widehat{\effect}^\alpha_\tool$, and projecting on states gives it for $\effect^\alpha_\tool$. Coverage and fidelity for the reference executions therefore apply to every pair reached in the actual execution.
\end{proof}

For a stepwise account of this argument, \cref{app:isolation} states sufficient conditions under which the agent's computation can be removed while preserving the worlds and states that an invocation can reach. Host and service activity remain possible in the reference executions. The argument uses \emph{non-interference}, which studies whether the activity of one process can affect another \citep{goguen1982security,rushby1992noninterference}.

The lemma allows us to establish coverage and fidelity by studying the tool together with the services and outside interactions that remain possible when the agent cannot act directly. These conditions then apply to the calls of any agent that uses the system through the Airlock. Combined with the earlier results, they give the guarantees of PC2 for every agent, which \cref{sec:main} states together (\cref{cor:every-agent}).

Under these conditions, changing the agent does not require a new argument about the effects of each tool. The argument still depends on the tool, its input, the host and the external services, whose properties must continue to hold. In \cref{sec:composition} we will use this separation when studying how guarantees for individual capability modules can support the tools of different contracts.

Isolation is sufficient for this conclusion, but safety does not require prohibiting every interaction outside the Airlock. An agent may retain direct access to a service's API or permission to modify a shared record, provided all effects possible through that access are covered by the models and preserve the required properties, without relying on the agent to choose its actions responsibly. Conversely, removing the agent's direct access does not resolve unsafe behaviour by a service or another party: protecting a resource against them is the role of custody, which applies complete mediation to the resource rather than to the agent (\cref{prop:custody}). The construction is particularly applicable when the system already limits how resources can change, and isolation prevents the agent from bypassing those limits.

\subsection{Execution and duration of the guarantees}\label{sec:execution-duration}

The correspondence arguments above explain how records can remain meaningful as routines and services act on the world. To combine these arguments with internal coherence, we must also establish that the deployed checker, control routine and execution loop behave as required in \cref{sec:coherence}. This brings us back to \cref{ass:sound-checking,ass:transition-coverage,ass:initial-control,ass:execution-integrity}. In particular, protecting the proof gate includes protecting the checker and contract, the recorded views, and the association between a proof and the tool and input that actually run. For example, replacing the input after its proof is checked breaks the connection required by \cref{ass:execution-integrity}, even if the checker itself is sound. The deployed routines must also satisfy the models used to form obligations: \cref{ass:transition-coverage} covers every state encountered during execution, including outside changes, and \cref{ass:initial-control} requires control to preserve compliance from every compliant active state. Verification must therefore concern the routines that run, under the service and host conditions used to establish these properties. The loop and checker are documented in the \PCRepo{} (\cref{sec:implementation}).

\textbf{Pauses and timing.}
A pause during a routine is part of that routine's execution interval. Any changes while it is paused must be covered by its model under \cref{ass:transition-coverage} and preserve correspondence under \cref{ass:tool-fidelity}. If unrestricted pauses make these properties false, the deployment needs an appropriate bound or protection that remains effective during the pause, such as a lease that expires unless the Airlock renews it \citep{gray1989leases}. For bounded gaps between routines, see \cref{app:intervals}. A control window bounds the number of control steps rather than elapsed time.

\textbf{Duration and completion.}
\Cref{thm:world-final} guarantees safety up to closing, and throughout a non-closing execution. It does not imply that a pending request is eventually executed or that settlement succeeds. These are liveness properties, while the guarantees are safety properties \citep{alpern1985defining}, the class that a mechanism checking each step can enforce \citep{schneider2000enforceable}. A deadline is the exception: written as a clause of the contract (\cref{ex:contract}), it is a safety property, and \cref{cor:every-agent}(4) obtains it from anytime compliance, independently of whether the agent makes progress. Likewise, the absence of a further approval before an authorized settlement call (\cref{sec:completion}) requires the deployment to hold the authority needed to issue it. That authority does not guarantee that the service completes the transfer. Responsibility for unmet conditions and possible remedies are separate matters, discussed in \cref{sec:residual}.

Some protections outlast the contract. In the payment ledger of \cref{app:examples}, the reservations recorded at closing still bound the charges that arrive afterwards, without the Airlock continuing to run, as long as no new instruction can be issued and the provider keeps respecting its limits.

\subsection{The guarantees of PC2}\label{sec:main}
We now state the guarantees of PC2 together, in the world and for every agent.

\begin{theorem}[The guarantees of PC2]\label{cor:every-agent}
Assume \cref{ass:sound-checking,ass:initial-control,ass:execution-integrity,ass:isolation,ass:contract-adequacy}, and \cref{ass:transition-coverage,ass:tool-fidelity} for the reference executions. Then every execution of \cref{alg:pc2} satisfies the following, for every agent and every sequence of requests and proofs.
\begin{enumerate}[label=(\arabic*),leftmargin=2.2em,itemsep=0pt,topsep=2pt]
\item \emph{Anytime safety.} $\Safe$ holds in the world at every instant from the start of the contract until it closes, or at every instant after the start if it never closes.
\end{enumerate}
Moreover, if the contract is in standard form (\cref{ex:contract}), it ends in at most one of the three ways of \cref{thm:completion}, and:
\begin{enumerate}[label=(\arabic*),leftmargin=2.2em,itemsep=0pt,topsep=2pt,resume]
\item \emph{Settlement.} If it is settled, the goal held in the world when the settlement started.
\item \emph{Termination.} If it is terminated, a control step found a termination condition.
\item \emph{Deadline.} If it has a deadline $d$ and starts, either settlement or termination is initiated by $d$.
\end{enumerate}
\end{theorem}

\begin{proof}
We follow the guarantees from the agent to the world.

\emph{Every agent.} By isolation (\cref{ass:isolation}) and \cref{lem:isolation}, every effect of a routine, in an execution with any agent, is also an effect of a reference execution. Coverage and fidelity established for the reference executions therefore give \cref{ass:transition-coverage,ass:tool-fidelity} for every agent, and requests and proofs may be arbitrary.

\emph{The clauses.} The checker accepts only true obligations (\cref{ass:sound-checking}), the models cover what the routines do (\cref{ass:transition-coverage}), the control step starts the contract compliant and keeps it compliant (\cref{ass:initial-control}), and a call runs only for an accepted proof, within its window (\cref{ass:execution-integrity}). These are the hypotheses of \cref{thm:commitment-final}, so $\clause$ holds at every instant from the start.

\emph{The world.} When the control step starts the contract, the state describes the world, and every later operation preserves this description (\cref{ass:tool-fidelity}), so the state describes the world at every instant from the start (\cref{thm:world-final}). A world described by a compliant state is safe (\cref{ass:contract-adequacy}, first part). This gives (1).

\emph{How the contract ends.} In standard form, only the routines of $\toolctrl$ and $\settle$ change the label, and the loop stops once the contract is inactive (\cref{ass:execution-integrity}), so the contract ends in at most one of the three ways (\cref{thm:completion}). If it is settled, the accepted proof of the settlement call gives the goal on the state when settlement starts (\cref{thm:completion}); the state then describes the world, and a world described by a state with the goal satisfies the goal in the world (\cref{ass:contract-adequacy}, second part). This gives (2), as in \cref{cor:world-completion}. If it is terminated, the control step terminated it only from a decision reading in $\mathit{End}$. This gives (3).

\emph{The deadline.} If the contract has a deadline $d$, every ongoing state has $t\leq d$ by the clauses. Since $t$ records current time on the agreed timescale, the contract cannot remain ongoing after $d$. The label changes only at the end of the settlement or termination routine, so the routine that closed the contract must have started by $d$. This gives (4).
\end{proof}

\PCLayer{Practice}
\section{Composing guarantees}\label{sec:assurance}

\Cref{sec:guarantees,sec:risks} have described the guarantees through abstract objects: states, tool models, effects and correspondences with the world. In an application, these objects describe concrete routines, records and services, such as \cref{ex:worlds} or the payment ledger and booking service described in \cref{app:examples}. Constructing these descriptions and their proofs anew for every contract would require substantial repeated work, with repeated opportunities to omit an effect or a condition on the service. Much of this work concerns the available operations rather than the particular agreement. The discipline of recording a payment instruction before issuing it, for example, can support contracts with different budgets and goals.

We propose to organize this work in reusable \emph{capability modules}. A module can be developed as a software component whose routines come with state descriptions, tool models, a correspondence with the world and guarantees under stated conditions. A contract can then be built from a combination of these modules, reusing their routines and arguments while specifying the constraints of the particular agreement. Reuse requires preserving the conditions under which the modules were verified, and their combination may require additional arguments. The purpose of this section is to show how such a process can work through a few constructions and examples, rather than to give an exhaustive account of capability composition.

In \cref{sec:composition} we define capability modules and explain how to combine them into a state space and a collection of tools on which to build a contract. We give conditions under which their coverage and fidelity guarantees establish \cref{ass:transition-coverage} and the first line of \cref{ass:tool-fidelity} for the combined tools. The contract still requires its own arguments for the start, the control step and adequacy, alongside the conditions on the checker and execution mechanism. In \cref{sec:responsibilities} we bring these arguments together to identify what must be established about the routines, host and services used by a particular contract, and who can provide evidence for it. This produces a tree in which an implication from more specific conditions to a claim is a \emph{refinement}.

\subsection{Capability modules}\label{sec:composition}\label{sec:compositional-fidelity}

A contract may use several services, each described by different state fields and a different correspondence to the world. For instance, a contract that makes payments and obtains booking confirmations can combine a payment ledger with a record of confirmations. While a booking request runs, earlier payments may still produce charges. Each module must therefore account both for its own operations and for what can happen while none of them runs. We represent the latter by an \emph{idle} tool.

\begin{definition}[Capability module]\label{def:module}
A \emph{capability module} is a tuple $M=(\worlds_M,\ss_M,\corresponds_M,\tools_M,\mathsf{idle}_M)$, where $\worlds_M$ and $\ss_M$ are sets, $\corresponds_M$ is a relation on $\worlds_M\times\ss_M$, and $\tools_M$ is a set of tools over $\ss_M$ whose routines operate on the system described by $\worlds_M$. The tool $\mathsf{idle}_M\in\tools_M$ has no caller-supplied input and waits without intervening. Its effects describe the changes possible during intervals in which no other routine of $M$ runs.
\end{definition}

Several modules $M_1,\ldots,M_n$ can supply the state space of a contract, for example by taking $\ss=\prod_{k=1}^n\ss_{M_k}$ and building its tools from their operations, as we describe below.

Here $\worlds_M$ describes the resource and the services relevant to its operations, while $\ss_M$ contains the quantities through which the contract reasons about them. These quantities may be protected records, such as the ledger, or current observations, such as a price the tools would report if asked. An observation of the current price may change while the module waits. A stored quotation remains unchanged, but its correspondence must then describe the quotation and its validity, rather than identify it with the current price. This is the distinction between a possible observation and a recorded reading made in \cref{sec:model}.

\textbf{$\mathsf{idle}_M$ and $\toolctrl$.} The tool $\mathsf{idle}_M$ describes what may happen while the module waits without intervening. This can be specified independently of a contract, whereas the contract's reserved tool $\toolctrl$ must preserve its particular clause $\clause$ and respond to its termination conditions. A contract might use a local control routine, say $\exec_{\toolctrl_M}$, when the necessary intervention depends only on the module's state. In other cases the decision requires information from several modules: whether to release a payment can depend on a booking and a deadline. Such a decision may require a routine operating on the combined state. We therefore include $\mathsf{idle}_M$ in every module without requiring a local controller.

\textbf{Local guarantees.} Coverage and fidelity can first be established for each module separately. Write $\effect_\tool$ and $\widehat{\effect}_\tool$ for its tools' state and joint world--state effects, as in \cref{sec:coherence,sec:world}, with all states and worlds now belonging to that module. For each $\tool\in\tools_M$, define
\begin{align}
\mathsf{Cov}_M(\tool)\ &:\iff\ \forall i,s,s'\quad \enabled_\tool(i,s)\wedge\effect_\tool(i,s,s')\implies\tran_\tool(i,s,s'),\label{eq:local-cov}\\
\mathsf{Fid}_M(\tool)\ &:\iff\ \forall i,w,s,w',s'\quad \corresponds_M(w,s)\wedge\enabled_\tool(i,s)\wedge\widehat{\effect}_\tool(i,w,s,w',s')\implies\corresponds_M(w',s'),\label{eq:local-fid}
\end{align}
with the quantifiers ranging over the tool's inputs and the module's spaces. These are the routine conditions of \cref{ass:transition-coverage,ass:tool-fidelity} for one module. For idle, fidelity applies to waiting intervals from states where its model is enabled. Under \cref{ass:isolation}, these guarantees can be established using the reference effects of \cref{sec:establishing-assumptions}: \cref{lem:isolation} then makes them applicable to every agent.

For example, whoever implements a payment module can prove that its routine records an instruction before sending it, subject to stated conditions on who can change the ledger and use the payment credentials. Another contract can reuse that proof when it uses the same routine under the same conditions. In \cref{sec:three-modules} we give three constructions that make these conditions explicit.

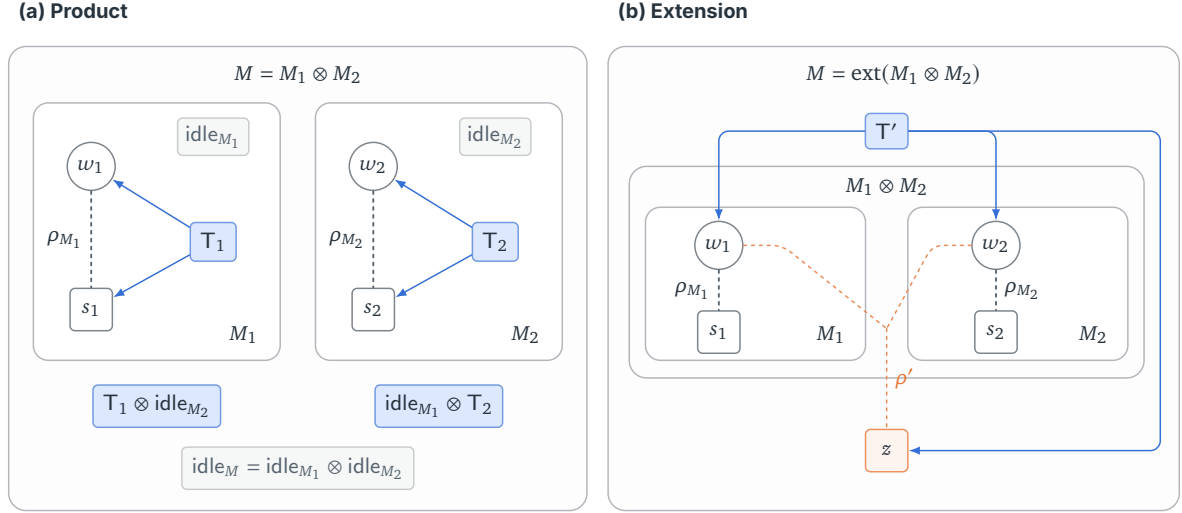
\begin{figure}[t]
\centering
\resizebox{\linewidth}{!}{%
\begin{tikzpicture}[x=1cm,y=1cm,text=asInk,font=\pcVisualFont,
  line width=.6pt,line cap=round,line join=round,
  >={Latex[length=1.6mm,width=1.2mm]},
  mod/.style={draw=asInk!35,fill=white,rounded corners=6pt},
  world/.style={circle,draw=asInk!60,fill=white,minimum size=.68cm,inner sep=0pt},
  state/.style={pc state,draw=asInk!60,fill=white,minimum size=.60cm,inner sep=0pt},
  tool/.style={draw=pcAgent!70,fill=pcAgentFill,rounded corners=2pt,inner sep=4pt},
  corr/.style={draw=asMuted,dash pattern=on 1.2pt off 2pt,line width=.75pt},
  act/.style={->,draw=pcAgent},
  extcorr/.style={draw=asMitigated!85,dash pattern=on 1.2pt off 2pt,rounded corners=4pt}]
\node[anchor=west,font=\pcVisualFont\bfseries] at (-.35,4.95) {(a) Product};
\node[anchor=west,font=\pcVisualFont\bfseries] at (8.15,4.95) {(b) Extension};
\draw[pcProduct] (-.35,-2.13) rectangle (7.85,4.45);
\node at (3.75,4.05) {$M=M_1\otimes M_2$};
\draw[mod] (0,0) rectangle (3.50,3.65);
\draw[mod] (4.00,0) rectangle (7.50,3.65);
\foreach \x/\k in {0/1,4/2}{
  \node[anchor=south east] at (\x+3.32,.10) {$M_{\k}$};
  \node[world] (w\k) at (\x+.82,2.75) {$w_{\k}$};
  \node[state] (s\k) at (\x+.82,.72) {$s_{\k}$};
  \draw[corr] (w\k) -- node[left,inner sep=3pt]{$\rho_{M_{\k}}$} (s\k);
  \node[tool] (t\k) at (\x+2.55,1.68) {$\tool_{\k}$};
  \draw[act] (t\k) -- (w\k);
  \draw[act] (t\k) -- (s\k);
  \node[draw=asInk!25,fill=asInk!4,rounded corners=2pt,inner sep=3pt,text=asMuted]
    at (\x+2.55,3.17) {$\mathsf{idle}_{M_{\k}}$};
}
\node[tool] at (1.75,-.65) {$\tool_1\otimes\mathsf{idle}_{M_2}$};
\node[tool] at (5.75,-.65) {$\mathsf{idle}_{M_1}\otimes\tool_2$};
\node[draw=asInk!25,fill=asInk!4,rounded corners=2pt,inner sep=4pt,text=asMuted]
  at (3.75,-1.53)
  {$\mathsf{idle}_{M}=\mathsf{idle}_{M_1}\otimes\mathsf{idle}_{M_2}$};
\draw[pcExtension] (8.15,-2.13) rectangle (16.25,4.45);
\node at (12.20,4.05) {$M=\operatorname{ext}(M_1\otimes M_2)$};
\draw[pcProduct] (8.45,-.25) rectangle (15.75,2.75);
\node at (12.10,2.46) {$M_1\otimes M_2$};
\draw[mod] (8.68,0) rectangle (11.80,2.17);
\draw[mod] (12.40,0) rectangle (15.52,2.17);
\node[anchor=south east] at (11.65,.10) {$M_1$};
\node[anchor=south east] at (15.37,.10) {$M_2$};
\node[world] (ew1) at (9.72,1.63) {$w_1$};
\node[world] (ew2) at (13.65,1.63) {$w_2$};
\node[state] (es1) at (9.72,.40) {$s_1$};
\node[state] (es2) at (13.65,.40) {$s_2$};
\draw[corr] (ew1) -- node[left,inner sep=3pt]{$\rho_{M_1}$} (es1);
\draw[corr] (ew2) -- node[right,inner sep=3pt]{$\rho_{M_2}$} (es2);
\node[tool] (newtool) at (12.10,3.25) {$\tool'$};
\node[state,draw=asMitigated!80,fill=asMitigated!8] (z) at (12.10,-1.28) {$z$};
\draw[act,rounded corners=4pt] (newtool.west) -| (ew1.north);
\draw[act,rounded corners=4pt] (newtool.east) -| (ew2.north);
\draw[act,rounded corners=4pt] (newtool.east) -- (15.98,3.25) |- (z.east);
\coordinate (join) at (12.10,.45);
\draw[extcorr] (ew1.east) -- (10.62,1.63) -- (join);
\draw[extcorr] (ew2.west) -- (12.75,1.63) -- (join);
\draw[extcorr] (join) -- node[right,inner sep=3pt,text=asMitigated]{$\rho'$} (z.north);
\end{tikzpicture}%
}

\caption{Product and extension of capability modules. Circles denote world coordinates, rounded squares denote state coordinates, dotted lines denote correspondence and arrows denote routine actions. (a) The product contains both modules, their lifted tools and the joint idle tool. (b) The outer frame encloses a chosen extension $\operatorname{ext}(M_1\otimes M_2)$ of the product in the inner grey frame. It adds the routine $\tool'$, state $z$ and correspondence $\rho'$, with the latter two highlighted in orange. The product's existing tools are omitted in (b). The extension notation is defined in \cref{def:module-extension} and used again in \cref{fig:derived-tree}.}
\label{fig:modules}
\end{figure}

\subsection{Composing capability modules}\label{sec:combining-modules}

To build an environment suited to an application, we need to combine modules that provide its different operations and descriptions of the world. For example, a travel application can draw on payment and booking modules, then add an operation relating a purchase to its payment. We consider two constructions, illustrated in \cref{fig:modules}. A \emph{product} puts the modules' world and state descriptions together and makes their existing tools available on the combined state. An \emph{extension} adds fields, correspondence conditions or routines involving those components. We first describe how tools on separate modules can be combined, since this is used in both constructions.

\textbf{Paired tools.} Let $U\in\tools_{M_1}$ and $V\in\tools_{M_2}$, where the modules have disjoint world and state coordinates. We write $U\otimes V$ for a tool over $\ss_{M_1}\times\ss_{M_2}$ with input space $I_U\times I_V$. When both components are non-idle, its routine executes them concurrently in a single invocation and returns when both have returned. If one component is idle, it waits until the other returns. Pairing two idle tools gives an operation that waits on both components. The model is
\begin{equation}\label{eq:product-model}
\tran_{U\otimes V}\bigl((i_1,i_2),(s_1,s_2),(s'_1,s'_2)\bigr)
\iff \tran_U(i_1,s_1,s'_1)\wedge\tran_V(i_2,s_2,s'_2).
\end{equation}

\textbf{Compatibility of effects.} Reusing the separate tools' guarantees requires checking that their combination does not introduce effects outside those for which the guarantees were established. Write $i=(i_1,i_2)$, $w=(w_1,w_2)$ and $s=(s_1,s_2)$, and similarly for primed quantities. The required condition, for all such inputs and coordinates, is
\begin{equation}\label{eq:no-coupling}
\widehat{\effect}_{U\otimes V}(i,w,s,w',s')\implies
\widehat{\effect}_U(i_1,w_1,s_1,w'_1,s'_1)\wedge
\widehat{\effect}_V(i_2,w_2,s_2,w'_2,s'_2).
\end{equation}
The condition concerns the whole paired invocation. If one routine finishes before the other, its effects and guarantees must also cover the remaining wait. When $V=\mathsf{idle}_{M_2}$, running $U$ must leave all changes on the second module within its idle effects, and conversely for an idle first component. Disjoint coordinates therefore need not mean that the resources are unaffected by each other. For example, a merchant may charge an earlier payment instruction while a booking call runs, provided the payment module's idle effects account for it. If an interaction introduces changes outside a component's effects, its description and guarantees must be revised, or the interacting parts must be treated together.

This is the question addressed by rely--guarantee reasoning: the conditions under which one component was verified must be respected by the other components and their environment \citep{jones1983rely,abadi1995conjoining}. Here \cref{eq:no-coupling} states that requirement directly for the effects used in our coverage and fidelity arguments.

\textbf{Product.} A product makes each module's tools available while accounting for changes on the other modules as they wait. For payment and booking modules, a payment call becomes $\mathsf{pay}\otimes\mathsf{idle}_{\mathsf{book}}$: its model includes the payment routine's effects and the changes that the booking can undergo during the payment.

\Needspace{9\baselineskip}
\begin{definition}[Product modules]\label{def:module-ops}\label{def:module-product}
For modules $M_1,M_2$ with disjoint world and state coordinates, their \emph{product} $M=M_1\otimes M_2$ has $\worlds_M=\worlds_{M_1}\times\worlds_{M_2}$, $\ss_M=\ss_{M_1}\times\ss_{M_2}$ and
\[
\corresponds_M((w_1,w_2),(s_1,s_2))\iff
\corresponds_{M_1}(w_1,s_1)\wedge\corresponds_{M_2}(w_2,s_2).
\]
Its tools are
\begin{equation}\label{eq:product-tools}
\tools_M=\{\tool_1\otimes\mathsf{idle}_{M_2}:\tool_1\in\tools_{M_1}\}
\ \cup\ \{\mathsf{idle}_{M_1}\otimes\tool_2:\tool_2\in\tools_{M_2}\}.
\end{equation}
The idle tool is $\mathsf{idle}_{M_1\otimes M_2}=\mathsf{idle}_{M_1}\otimes\mathsf{idle}_{M_2}$.
\end{definition}

The product is associative and commutative up to the canonical regrouping and permutation of coordinates: $M_1\otimes(M_2\otimes M_3)\cong(M_1\otimes M_2)\otimes M_3$ and $M_1\otimes M_2\cong M_2\otimes M_1$. These identifications preserve the correspondence and tool models and retain the same component routines. We can therefore write $M_1\otimes\cdots\otimes M_n$ without parentheses for any finite nonempty family of modules with pairwise disjoint coordinates.

\textbf{Extension.} Applications may also require records that relate several modules or operations that coordinate their use, for example by chaining calls or executing them concurrently. An extension can add these operations together with new state fields and a correspondence describing what those fields mean. Making the additions explicit identifies what must be established beyond the guarantees already supplied by the modules.

\begin{definition}[Extension modules]\label{def:module-extension}
Let $M$ be a module, $Z$ a nonempty set, $\corresponds'$ a relation on $\worlds_M\times\ss_M\times Z$, and $\tools'$ a set of tools over $\ss_M\times Z$. The \emph{extension} $N=\operatorname{ext}(M;Z,\corresponds',\tools')$ has $\worlds_N=\worlds_M$, $\ss_N=\ss_M\times Z$ and
\[
\corresponds_N(w,(s,z))\iff\corresponds_M(w,s)\wedge\corresponds'(w,s,z).
\]
Its tools are $\{\widetilde{\tool}:\tool\in\tools_M\}\cup\tools'$. Each retained tool $\widetilde{\tool}$ has the same inputs and routine as $\tool$, leaves $z$ unchanged and has model
\begin{equation}\label{eq:extension-model}
\tran_{\widetilde{\tool}}(i,(s,z),(s',z'))
\iff\tran_\tool(i,s,s')\wedge z'=z.
\end{equation}
The idle tool is $\widetilde{\mathsf{idle}_M}$, and the fields in $Z$ change only through the added routines.
\end{definition}

We write $\operatorname{ext}(M)$ when the added state, correspondence and tools are understood. This denotes a chosen extension of $M$, whose identity depends on those additions. In particular, $\operatorname{ext}(M_1\otimes M_2)$ in \cref{fig:modules} extends the product of the two modules.

For example, a travel application can extend the payment and booking modules with a register $Z$ of completed purchases. Each entry identifies an order, its booking confirmation and its payment instruction. The added correspondence $\corresponds'$ states that the identified confirmation and instruction were issued for that order, relating the two modules' histories to the new record. An added purchase routine obtains a confirmation, issues the associated payment and records their association. This uses all three parts of an extension: additional state, a relation connecting it to the world, and an operation involving several modules. This instantiates the extension in \cref{fig:modules}: the added routine chains the booking and payment calls and stores their association in $z$. The new record can remain unchanged when the original tools run, but its correspondence must remain true. We return to that requirement in \cref{sec:global-local}.

An extension can also add an operation alone, taking $Z$ to be a singleton and $\corresponds'$ to be true. One example is the paired tool $\tool_1\otimes\tool_2$ with both components non-idle, which the product alone does not provide. It runs both routines concurrently as one Airlock invocation. Its conjunction model and the compatibility requirement are those already given in \cref{eq:product-model,eq:no-coupling}.

Another example is a tool whose implementation runs $\tool_1$ while a controller maintains the second component. A requester submits this combined operation to the proof gate, and the controller runs as part of its implementation. This does not grant access to the contract's reserved $\toolctrl$, which the Airlock alone invokes between requested calls. Writing the combined tool as $\tool_1\otimes\toolctrl_2$ is appropriate only when such a controller can be specified for the second component and its guarantees cover the whole invocation. When the necessary decisions depend on several components, the added routine must instead use the full state.

\textbf{Building a contract.} A product, possibly extended as above, gives a module $N$ on which we can build the contract of \cref{def:contract}. We take $\ss=\ss_N$, select from $\tools_N$ the tools exposed through $\tools$, and supply the clause $\clause$, view $\obs$, reserved control tool $\toolctrl$ and termination test $\terminal$. The tool $\toolctrl$ and any other required tools absent from $N$ can be added by extension. For modules $M_1,\ldots,M_n$ and additional fields $Z$, the resulting state space is $\ss=(\prod_{k=1}^n\ss_{M_k})\times Z$. The clause need not split into separate conditions on the components: it may, for example, require each payment instruction to have an associated booking confirmation. The obligations of \cref{sec:model} use this combined state, clause and tool models without modification.

The spaces $\worlds_N$ and correspondence $\corresponds_N$ provide the interpretation used in \cref{sec:world}. The composition result in \cref{sec:global-local} explains when the modules' guarantees establish coverage and fidelity for the contract's tools. The requirements on $\toolctrl$ in \cref{ass:initial-control}, as initializer and as controller, must still be proved for its chosen clause, as must the adequacy implication $\corresponds_N(w,s)\wedge\clause(s)\implies\Safe(w)$. The checker and execution mechanism remain subject to \cref{ass:sound-checking,ass:execution-integrity}.

\subsection{Three capability constructions}\label{sec:three-modules}

\begin{table}[t]
\caption{Three constructions of capability modules, each illustrated by a different example. The world, state, correspondence and tools describe the module, while the supporting conditions establish the guarantees below. A module can combine these constructions.}
\label{tab:three-modules}
\centering\pcTableStyle
\begin{tabularx}{\linewidth}{@{}>{\raggedright\arraybackslash}p{2.60cm}*{3}{>{\raggedright\arraybackslash}X}@{}}
\toprule
\textbf{Module element} & \textbf{Custody}\newline (Ex.: payment instructions) & \textbf{Formal interface}\newline (Ex.: booking responses) & \textbf{Domain invariant}\newline (Ex.: mailbox retention) \\
\midrule
World $\worlds_M$ & Payment service and history $J(w)$ of submitted instructions. & Booking service and history $H(w)$ of issued responses. & Mail service, including delivery and acknowledgement history. \\
\addlinespace[4pt]
State $\ss_M$ & Register of payment instructions. & Log of response events. & Current mailbox contents $\pi(w)$. \\
\addlinespace[4pt]
Correspondence $\rho_M$ & $J(w)\subseteq s$\newline Every submitted instruction is registered. & $\forall e\in s,$\newline $e\in H(w)\wedge\mathsf{Spec}(e)$\newline Recorded responses were issued and are permitted. & $s=\pi(w)\wedge\mathsf K(w)$\newline The reading is current and pending mail remains available. \\
\addlinespace[4pt]
Tools $\tools_M$ & $\mathsf{issue}(j)$ records an instruction, then submits it. & $\mathsf{query}(q)$ requests, authenticates and records a response. & Read or process messages, with models covering changes to mailbox contents. \\
\addlinespace[4pt]
Idle $\mathsf{idle}_M$ & Register unchanged. Earlier instructions may still produce charges. & Log unchanged. The service may continue to evolve. & Mailbox contents may change, while pending messages remain available. \\
\midrule
Supporting conditions & Protect register updates \cref{ass:custody-register}. Register before submitting and restrict issuance \cref{ass:custody-submission}. & Permitted responses \cref{ass:interface-responses}, authenticated and recorded faithfully \cref{ass:interface-recording}. & Readings agree with the world \cref{ass:measurement-agreement}. Routines and outside changes preserve $\mathsf K$ \cref{ass:invariant-preservation}. \\
\addlinespace[4pt]
Local result & Coverage and fidelity\newline (\cref{prop:custody}) & Coverage and fidelity\newline (\cref{prop:interface}) & Fidelity (\cref{prop:shared-model}). Coverage is checked separately. \\
\bottomrule
\end{tabularx}
\end{table}

We now give concrete conditions under which a module satisfies its local guarantees. The constructions in \cref{tab:three-modules} distinguish a protected record of instructions, a record of what a service reported, and a description of quantities that may change independently. They can be combined when a resource requires more than one of these arguments.

\textbf{Custody.} Consider a payment account whose credentials are kept by the Airlock. Its payment routine can send the provider an instruction specifying a recipient and a spending limit, with an identifier distinguishing it from other instructions. Before sending the message, the routine adds the instruction to a protected register. This order matters because the provider may charge the account as soon as it receives the message, before any acknowledgement returns. We want the register to account for every instruction that the provider can act on, including those still in progress. This is the part of the payment-ledger argument in \cref{app:examples} that concerns custody of the authority to issue payments. The bound on the charges themselves additionally depends on the provider respecting each instruction's limit.

Let $\mathit{Instr}$ be the set of these payment instructions, $\ss_M$ the set of its finite subsets, and $I_{\mathsf{issue}}=\mathit{Instr}$. A state $s$ is the register. Write $J(w)$ for the history of submitted instructions, including pending instructions from the moment of submission. Define
\begin{equation}\label{eq:custody-module}
\begin{aligned}
\corresponds_M(w,s)&\iff \forall j\in J(w).\ j\in s,\\
\tran_{\mathsf{issue}}(j,s,s')&\iff s'=s\ \lor\ s'=s\cup\{j\},
\qquad \tran_{\mathsf{idle}_M}(s,s')\iff s'=s.
\end{aligned}
\end{equation}
The module has these two tools. The model of $\mathsf{issue}$ includes the register before and after the instruction is added, while idle leaves it unchanged. Entries are retained even when a request times out, since an unacknowledged instruction may still produce a charge.

\begin{proposition}[Custody]\label{prop:custody}
For the module defined by \cref{eq:custody-module}, assume that, in every execution of its tools, including idle,
\begin{enumerate}[label=\PCAssumptionBadge{H\arabic*},ref=H\arabic*,leftmargin=2.6em,itemsep=0pt,topsep=2pt]
\item\label[assumption]{ass:custody-register} only $\mathsf{issue}(j)$ changes the register, by adding $j$ at most once during the invocation. Entries are never removed or modified;
\item\label[assumption]{ass:custody-submission} an instruction is submitted only by $\mathsf{issue}$ after it is registered. The history $J$ retains all submissions, and no other action adds an instruction or changes its terms.
\end{enumerate}
Then $\mathsf{Cov}_M(\tool)$ and $\mathsf{Fid}_M(\tool)$ hold for $\tool\in\{\mathsf{issue},\mathsf{idle}_M\}$.
\end{proposition}

\begin{proof}
\cref{ass:custody-register} limits register changes to those in \cref{eq:custody-module}, establishing coverage for both tools. Starting from correspondence, every initial member of $J$ remains registered by \cref{ass:custody-register}. Every new member is registered before submission by \cref{ass:custody-submission}, and its entry is retained. Thus $J(w')$ remains contained in $s'$ throughout either invocation, establishing fidelity.
\end{proof}

The access restriction in \cref{ass:custody-submission} applies complete mediation \citep{saltzer1975protection} to the authority to issue payment instructions. Keeping the credentials under the Airlock's control and requiring the provider to accept only authenticated instructions can support this restriction, as in capability-based protection \citep{dennis1966programming,miller2006robust}. The Airlock's implementation must protect the register and record each instruction before sending it. The payment provider must preserve its terms and reject submissions made without the required authority. These conditions do not assert that the account balance remains fixed while the module is idle. Charges from earlier instructions may continue, and their bounds require the additional service condition in \cref{eq:budget-correspondence}.

\textbf{Formal interfaces.} A service can specify which responses it may return to a request and how they relate to that request and to its previous interactions. Such a specification gives a module a basis for modelling calls to the service without describing its implementation. The module's routine records the responses it receives, and its correspondence connects those records to responses actually issued by the service and permitted by its specification.

For example, a booking request can identify a hotel, dates and a maximum price. The specification may allow either a rejection or a confirmation carrying those dates and a price within the requested limit. Concretely, an implementation could maintain a register of response records, stored in memory or in a database. Each record would contain the request, the returned status and, for a confirmation, the booking identifier, dates and price. The routine authenticates the response and adds its record to the register. The module's state represents the register's contents as a collection of such records, so the model can express which additions are permitted without committing to a particular storage implementation. Correspondence then relates each record to a response actually issued by the service. Whether the hotel honours the confirmation is a further property, supported by the reservation condition in \cref{ex:booking-confirmations}.

The same construction applies to any service with a specified request--response interface. Let $\mathit{Req}$ be the set of requests and $\mathit{Resp}$ the set of response events. An event contains the request, the returned data and any time or historical information used by the specification. Write $\mathsf{request}(e)$ for its request and $\mathsf{Spec}(e)$ for the statement that the response is permitted for that request. Take $I_{\mathsf{query}}=\mathit{Req}$ and let $\ss_M$ consist of finite sets of response events. With $H(w)$ denoting the history of events actually issued by the service, define
\begin{equation}\label{eq:interface-module}
\begin{aligned}
\corresponds_M(w,s)&\iff\forall e\in s.\ e\in H(w)\wedge\mathsf{Spec}(e),\\
\tran_{\mathsf{query}}(q,s,s')&\iff s'=s\ \lor\ \exists e.\ \mathsf{request}(e)=q\wedge\mathsf{Spec}(e)\wedge s'=s\cup\{e\},\\
\tran_{\mathsf{idle}_M}(s,s')&\iff s'=s.
\end{aligned}
\end{equation}

The model allows any response permitted by the specification, while correspondence requires each recorded response to have actually been issued. The following proposition separates the conditions needed on the service from those needed on the routine that records its responses.

\begin{proposition}[Formal interface]\label{prop:interface}
For the module defined by \cref{eq:interface-module}, assume that $H$ retains its events and that, in every execution of its tools, including idle,
\begin{enumerate}[label=\PCAssumptionBadge{F\arabic*},ref=F\arabic*,leftmargin=2.6em,itemsep=0pt,topsep=2pt]
\item\label[assumption]{ass:interface-responses} every response event issued by the service satisfies $\mathsf{Spec}$;
\item\label[assumption]{ass:interface-recording} only $\mathsf{query}(q)$ changes the log. It either leaves it unchanged or appends one event $e$ issued for $q$, after authenticating its origin and recording its contents correctly. Entries are never removed or modified.
\end{enumerate}
Then $\mathsf{Cov}_M(\tool)$ and $\mathsf{Fid}_M(\tool)$ hold for $\tool\in\{\mathsf{query},\mathsf{idle}_M\}$.
\end{proposition}

\begin{proof}
By \cref{ass:interface-recording}, any change appends one response event for the current request, and \cref{ass:interface-responses} makes that event satisfy $\mathsf{Spec}$. This proves coverage. Initially valid entries remain in $H$, since history is retained. Each new entry is an issued event by \cref{ass:interface-recording} and satisfies $\mathsf{Spec}$ by \cref{ass:interface-responses}, proving fidelity. During idle the log is unchanged and its entries remain valid.
\end{proof}

If no response has arrived, or the call times out without recording one, the log remains unchanged, which the model permits. An error returned by the service can instead be a response allowed by $\mathsf{Spec}$. Condition \cref{ass:interface-responses} requires the actual responses to obey that specification, for example by respecting the requested dates and price limit. This is a property of the service, for which verification or testing can provide evidence \citep{tretmans2008model,bornholt2021s3,disselkoen2024cedar}. Condition \cref{ass:interface-recording} concerns the Airlock's routine: it must authenticate the response and record it against the correct request. Authentication establishes who sent the response, but cannot by itself establish that its contents satisfy the specification or that the hotel will honour the booking \citep{zhang2016towncrier,zhang2020deco}.

\textbf{Domain invariants.} Interpreting a service's records often relies on properties of the system in which the service operates. A booking confirmation can stand for a reservation still held only if the hotel maintains that reservation under the agreed conditions. Such properties can support many different contracts using the same service. We express them by a predicate $\mathsf{K}$ on worlds, and ask that the module's routines and the outside interactions that remain possible preserve $\mathsf{K}$ whenever it holds initially. This is the sense in which $\mathsf{K}$ is a \emph{domain invariant}. It describes a relationship that the system must maintain, which may involve facts absent from the module's state.

Consider a module whose state is a measurement $\pi:\worlds_M\to\ss_M$, with $\ss_M=\pi(\worlds_M)$. Here $s=\pi(w)$ is the current, possibly unobserved state, as in \cref{sec:model}. For example, a module that monitors a mailbox can take $\pi(w)$ to be its current contents. Write $\mathsf{Pending}(w,m)$ when message $m$ has been delivered to the mailbox but has not yet been acknowledged by the module's processing routine, and $\mathsf{Available}(w,m)$ when it can still be retrieved from the mailbox. A retention condition on the mail service is
\[
\mathsf{K}(w)\quad\iff\quad
\forall m.\ \mathsf{Pending}(w,m)\implies\mathsf{Available}(w,m).
\]
This condition allows new messages to arrive and acknowledged messages to be removed, while requiring pending messages to remain available. Delivery and acknowledgement belong to the world's history, so a correct reading of an empty mailbox alone does not establish that no messages are pending. That conclusion also uses $\mathsf{K}$. The module's guarantees therefore depend on the routines and mail service preserving this retention condition.

We include both measurement agreement and the domain invariant in the correspondence:
\begin{equation}\label{eq:shared-model-transition}
\corresponds_M(w,s)\iff s=\pi(w)\wedge \mathsf{K}(w).
\end{equation}
The tools retain their models $\tran_\tool$ and their execution effects $\effect_\tool$ and $\widehat{\effect}_\tool$. Coverage requires $\enabled_\tool(i,s)\wedge\effect_\tool(i,s,s')\implies\tran_\tool(i,s,s')$, as in \cref{eq:local-cov}, and must be established for the chosen models. The role of $\mathsf{K}$ is to justify the additional relationships asserted by correspondence. To establish fidelity, we check that the actual joint effects $\widehat{\effect}_\tool$ preserve $\mathsf{K}$ while the state continues to agree with $\pi$. The following proposition expresses these conditions and shows how $\mathsf{K}$ can then support an adequacy argument.

\begin{proposition}[Domain invariant]\label{prop:shared-model}
For a module with correspondence given by \cref{eq:shared-model-transition}, assume that
\begin{enumerate}[label=\PCAssumptionBadge{M\arabic*},ref=M\arabic*,leftmargin=2.6em,itemsep=0pt,topsep=2pt]
\item\label[assumption]{ass:measurement-agreement} in every execution of its tools, including idle, the state at every instant is $\pi$ of the world at that instant;
\item\label[assumption]{ass:invariant-preservation} for every $\tool\in\tools_M$ and all inputs, worlds and states,
\[
\enabled_\tool(i,s)\wedge\mathsf{K}(w)\wedge
\widehat{\effect}_\tool(i,w,s,w',s')\implies\mathsf{K}(w').
\]
\end{enumerate}
Then $\mathsf{Fid}_M(\tool)$ holds for every $\tool\in\tools_M$. For predicates $\clause$ on $\ss_M$ and $\Safe$ on $\worlds_M$, adequacy follows if
\begin{equation}\label{eq:shared-model-adequacy}
\forall w\in\worlds_M\qquad \mathsf{K}(w)\wedge\clause(\pi(w))\implies\Safe(w).
\end{equation}
\end{proposition}

\begin{proof}
Consider an enabled invocation starting from $\corresponds_M(w,s)$ and any joint effect $\widehat{\effect}_\tool(i,w,s,w',s')$. Initial correspondence gives $\mathsf{K}(w)$, so \cref{ass:invariant-preservation} yields $\mathsf{K}(w')$, while \cref{ass:measurement-agreement} gives $s'=\pi(w')$. Thus $\corresponds_M(w',s')$ holds, proving fidelity. Finally, correspondence and $\clause(s)$ imply $\mathsf{K}(w)$ and $\clause(\pi(w))$, so \cref{eq:shared-model-adequacy} gives $\Safe(w)$.
\end{proof}

Conditions such as \cref{ass:invariant-preservation} express what the system relies on about its environment, as domain assumptions do in requirements engineering \citep{jackson1995world,zave1997four}. In the mailbox example, \cref{ass:invariant-preservation} requires pending messages to remain available throughout processing and waiting. A routine that removes a message must therefore have acknowledged it by the time removal takes effect. Preservation of $\mathsf{K}$ can be proved for a specified service whose operations enforce these conditions. Applying that proof to an actual service also requires evidence that its implementation and the permitted outside interactions satisfy the specification \citep{lee2016fundamental}.

Conditions on outside changes also inform the choice of $\tran_\tool$. For example, a bound on how rapidly a price can change must account for the whole duration of an invocation or idle interval. An unbounded idle period can require a much broader model, and the contract may then be unable to authorize an operation. Such coverage arguments concern the possible effects on the state, alongside the invariant-preservation argument above.

Another example concerns the expert tools introduced in \cref{sec:tool-families}, such as an LLM asked to assess whether a hotel is suitable. An expert module can describe the history of these assessments, with $\pi$ extracting the judgment returned by a particular invocation. A state field then denotes that identified assessment, so observing it again retrieves the same result even when the procedure could give a different answer on a new call. The correspondence establishes which procedure produced the judgment. A contract using the module must additionally specify what reliance on that judgment means for its goal or safety property, which is a question of adequacy.

\begin{table}[t]
\caption{Forms of interaction between a capability module and the world. Each row relates the changes a model must cover to possible contract responses and approaches from the literature.}
\label{tab:worlds}
\centering\pcTableStyle
\begin{tabularx}{\linewidth}{@{}>{\raggedright\arraybackslash}p{3.45cm}>{\raggedright\arraybackslash}X>{\raggedright\arraybackslash}X>{\raggedright\arraybackslash}p{4.25cm}@{}}
\toprule
\textbf{Interaction with the world} & \textbf{What models cover} & \textbf{Contract response} & \textbf{Related approaches} \\
\midrule
\textbf{Protected from outside changes}\newline E.g. an instruction register & No outside changes to the protected fields. & Rely on these records when authorizing later operations. & Complete mediation and custody\newline \citep{saltzer1975protection} \\
\addlinespace[4pt]
\textbf{History only grows}\newline E.g. a log of service responses & Additions to the history, with past events retained. & Use facts about the past that survive further additions. & History variables\newline \citep{abadi1991refinement} \\
\addlinespace[4pt]
\textbf{Changes within an agreed bound}\newline E.g. withdrawals by other account holders & Every debit permitted by the account's shared withdrawal limit. & Reserve enough margin for those debits. & Rely--guarantee reasoning\newline \citep{jones1983rely} \\
\addlinespace[4pt]
\textbf{Changes at a bounded rate}\newline E.g. charges for a metered service & Change bounded by a rate and the duration of the operation. & Keep a margin and intervene before it is exhausted. & Controlled invariants, barrier functions and Simplex\newline \citep{blanchini1999set,ames2019cbf,sha2001simplex} \\
\addlinespace[4pt]
\textbf{Holds until expiry}\newline E.g. a temporary hotel reservation & The commitment remains valid until its expiry. & Complete the operation in time, renew, or release. & Leases\newline \citep{gray1989leases} \\
\addlinespace[4pt]
\textbf{May change beyond our control}\newline E.g. a flight cancellation & Each possible outcome and its consequences. & Allow a compliant way to continue or terminate after each outcome. & Demonic nondeterminism\newline \citep{dijkstra1975guarded} \\
\addlinespace[4pt]
\textbf{Responds to a request}\newline E.g. a booking confirmation or rejection & Every response allowed by the specification. & Prove admissibility for every permitted response. & Formal interfaces and conformance testing\newline \citep{tretmans2008model} \\
\bottomrule
\end{tabularx}
\end{table}

The payment and booking examples of \cref{app:examples} combine these constructions. In the ledger, custody supports $u_j\leq r_j$, while the provider's charge bound supports $a_j\leq u_j$. For bookings, recording an authentic confirmation supports the historical claim, while the hotel's commitment supports the claim that the room remains held. These conditions belong to different parties even when they support one correspondence. \Cref{tab:worlds} places these constructions in the broader setting of interactions with the world, from protected records to changing quantities and service responses. It connects what the tools must model, including while idle, to the choices available to a contract and to approaches developed in the literature.

\subsection{From local to global guarantees}\label{sec:global-local}

The preceding constructions identify conditions on routines, records and services that support a module's guarantees. Once coverage and fidelity have been established for each module, we can combine them for a contract whose state is their product. For a tool that operates on one component, the proof must account for that operation and for the other components while they wait. More generally, the same argument applies to a combined tool when each component's model and effects cover the whole invocation. A control routine that cannot be decomposed in this way needs a separate argument over the full state.

\begin{theorem}[From local to global guarantees]\label{thm:global-local}
Let $M_1,\ldots,M_n$ be capability modules, with $n\geq1$, and let $\contract$ have state space $\ss=\prod_{k=1}^n\ss_{M_k}$. Set $\worlds=\prod_{k=1}^n\worlds_{M_k}$ and $\corresponds(w,s)\iff\bigwedge_{k=1}^n\corresponds_{M_k}(w_k,s_k)$. Suppose each contract tool, including control, is a tensor of tools from the $M_k$, with the conjunction model of \cref{eq:product-model}. Assume the $n$-component form of \cref{eq:no-coupling} for every such tool. Then
\begin{equation}\label{eq:global-from-local}
\bigwedge_{k=1}^{n}\ \bigwedge_{\tool\in\tools_{M_k}}\mathsf{Cov}_{M_k}(\tool)\implies\text{\cref{ass:transition-coverage}},
\qquad
\bigwedge_{k=1}^{n}\ \bigwedge_{\tool\in\tools_{M_k}}\mathsf{Fid}_{M_k}(\tool)\implies\text{\cref{eq:fidelity}}.
\end{equation}
\end{theorem}

\begin{proof}
An enabled tensor has an enabled component on each module, because a witness to its conjunction model supplies a witness to every component model. By \cref{eq:no-coupling}, each joint effect projects to an effect of that component. Local coverage therefore gives each conjunct of the tensor model, proving coverage. Local fidelity preserves each conjunct of $\corresponds$, which proves \cref{eq:fidelity}, the first line of \cref{ass:tool-fidelity}.
\end{proof}

The result explains what can be reused when a contract combines, for example, payment and booking. The proofs about recording payment instructions and booking responses can be established separately. To use them together, we check that running either routine does not introduce changes excluded by the other module's effects. Once this compatibility holds, the proofs establish coverage and fidelity for the combined tools, without a new proof for every payment--booking combination. The second line of \cref{ass:tool-fidelity}, the start, is a property of the contract's control step (\cref{sec:world}).

The theorem is stated using execution effects $\effect$ and $\widehat{\effect}$. To obtain guarantees independent of the requesting agent, its argument can first be applied to the reference executions of \cref{sec:establishing-assumptions}, with the same compatibility condition on their effects. Under \cref{ass:isolation}, \cref{lem:isolation} transfers the resulting coverage and fidelity to executions with any agent. This identifies the separate roles of composition and isolation: composition combines the component arguments, while isolation makes the resulting argument apply regardless of the agent's computation and requests.

A contract can also restrict an existing tool by a precondition, as in \cref{sec:tool-families}. The new model is the conjunction of that precondition with the original model. Whenever it is enabled, the precondition holds and the original tool is enabled, so the original coverage proof also establishes coverage of the restricted model. The routine and correspondence have not changed, so the same fidelity proof applies. The restriction thus changes which calls can be authorized while retaining the established guarantees for those calls.

\textbf{Checking an extension.} We want to reuse the guarantees of an existing module when adding state fields, correspondence claims or routines. To do so, we first establish that the existing routines retain their guarantees in the extended system, then check that they preserve the added correspondence and that the new routines satisfy the requirements on the full state. The completed-purchase register introduced above illustrates why these checks are needed. An existing booking routine leaves this register unchanged, but that alone does not establish the added correspondence $\corresponds'$. If an entry means that a confirmation was issued, later cancellation need not invalidate it. If it means that the booking is still held, cancellation can make it false even though the entry has not changed.

More precisely, the effects of each existing routine, projected onto $\worlds_M$ and $\ss_M$, must remain within its original effects, and $Z$ must remain unchanged. Their coverage proofs then establish the lifted models in \cref{eq:extension-model}, and their fidelity proofs preserve $\corresponds_M$. For each such tool, including idle, we additionally check that an enabled invocation starting from $\corresponds_M\wedge\corresponds'$ preserves $\corresponds'$. For an added routine, such as the purchase routine, coverage must hold on the full state $\ss_M\times Z$ and fidelity must preserve both correspondences. An added control routine must also satisfy \cref{ass:initial-control} for the chosen contract clause. These checks allow an extension to reuse the original proofs while establishing the additional claims it makes.

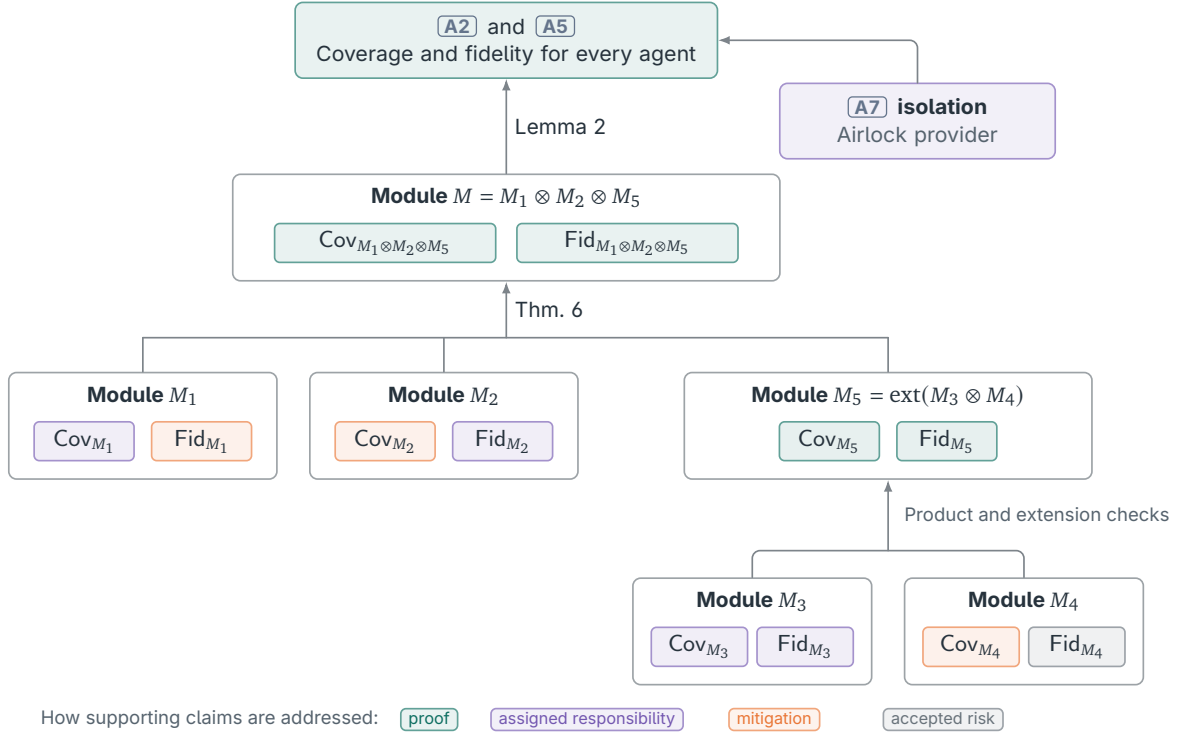
\begin{figure}[t]
\centering
\resizebox{\linewidth}{!}{%
\begin{tikzpicture}[x=1cm,y=1cm,text=asInk,font=\pcVisualFont,
  execute at begin node={\setlength{\parskip}{0pt}\hyphenpenalty=10000},
  line width=.65pt,line cap=round,line join=round,
  >={Latex[length=1.6mm,width=1.2mm]},
  box/.style={draw=asInk!45,fill=white,rounded corners=3pt,align=center,inner sep=5pt},
  claim/.style={draw,rounded corners=2pt,align=center,inner xsep=3pt,inner ysep=3pt,minimum height=.48cm},
  module/.style={box,text width=3.30cm,minimum height=1.45cm},
  note/.style={text=asMuted,font=\sffamily\fontsize{8}{9.5}\selectfont},
  dep/.style={draw=asInk!60,rounded corners=3pt}]
\node[box,draw=asProved!65,fill=asProved!9,text width=5.4cm,minimum height=1.0cm] (goal) at (7.15,0)
 {\PCAssumptionRef{ass:transition-coverage}\enspace and\enspace\PCAssumptionRef{ass:tool-fidelity}\\Coverage and fidelity for every agent};
\node[box,text width=7.1cm,minimum height=1.45cm] (combined) at (7.15,-2.55) {};
\node at (7.15,-2.14) {\textbf{Module $M=M_1\otimes M_2\otimes M_5$}};
\node[claim,draw=asProved!65,fill=asProved!9,text width=2.80cm] at (5.50,-2.76)
 {$\mathsf{Cov}_{M_1\otimes M_2\otimes M_5}$};
\node[claim,draw=asProved!65,fill=asProved!9,text width=2.80cm] at (8.80,-2.76)
 {$\mathsf{Fid}_{M_1\otimes M_2\otimes M_5}$};
\node[box,draw=asAssigned!65,fill=asAssigned!9,text width=3.40cm,minimum height=1.0cm] (a7) at (12.75,-1.10)
 {\PCAssumptionRef{ass:isolation}\enspace\textbf{isolation}\\{\color{asMuted}Airlock provider}};
\draw[dep,->] (combined.north) -- node[right,inner sep=3pt]{\cref*{lem:isolation}} (goal.south);
\draw[dep,->] (a7.north) |- (goal.east);
\foreach \k/\x/\covcolour/\fidcolour in {1/2.20/asAssigned/asMitigated,2/6.30/asMitigated/asAssigned}{
  \node[module] (m\k) at (\x,-5.25) {};
  \node at (\x,-4.84) {\textbf{Module $M_\k$}};
  \node[claim,draw=\covcolour!65,fill=\covcolour!10,text width=1.15cm] at (\x-.80,-5.46)
    {$\mathsf{Cov}_{M_\k}$};
  \node[claim,draw=\fidcolour!65,fill=\fidcolour!10,text width=1.15cm] at (\x+.80,-5.46)
    {$\mathsf{Fid}_{M_\k}$};
}
\node[module,text width=5.20cm] (m5) at (12.35,-5.25) {};
\node at (12.35,-4.84) {\textbf{Module $M_5=\operatorname{ext}(M_3\otimes M_4)$}};
\node[claim,draw=asProved!65,fill=asProved!10,text width=1.15cm] at (11.55,-5.46)
  {$\mathsf{Cov}_{M_5}$};
\node[claim,draw=asProved!65,fill=asProved!10,text width=1.15cm] at (13.15,-5.46)
  {$\mathsf{Fid}_{M_5}$};
\foreach \m/\x in {m1/2.20,m2/6.30,m5/12.35}
  \draw[dep] (\m.north) -- (\x,-4.05);
\draw[dep] (2.20,-4.05) -- (12.35,-4.05);
\draw[dep,->] (7.15,-4.05) -- node[right,inner sep=3pt]{\cref*{thm:global-local}} (combined.south);
\foreach \k/\x/\covcolour/\fidcolour in {3/10.50/asAssigned/asAssigned,4/14.20/asMitigated/asAccepted}{
  \node[box,text width=2.90cm,minimum height=1.45cm] (m\k) at (\x,-8.05) {};
  \node at (\x,-7.64) {\textbf{Module $M_\k$}};
  \node[claim,draw=\covcolour!65,fill=\covcolour!10,text width=1.10cm] at (\x-.72,-8.26)
    {$\mathsf{Cov}_{M_\k}$};
  \node[claim,draw=\fidcolour!65,fill=\fidcolour!10,text width=1.10cm] at (\x+.72,-8.26)
    {$\mathsf{Fid}_{M_\k}$};
}
\draw[dep] (m3.north) -- (10.50,-6.95) -- (14.20,-6.95) -- (m4.north);
\draw[dep,->] (12.35,-6.95) -- (m5.south);
\node[note,anchor=west,align=left,inner sep=4pt] at (12.42,-6.45)
 {Product and extension checks};
\node[note,anchor=east] at (5.55,-9.25) {How supporting claims are addressed:};
\node[asLegend=asProved] at (6.10,-9.25) {proof};
\node[asLegend=asAssigned] at (8.25,-9.25) {assigned responsibility};
\node[asLegend=asMitigated] at (10.80,-9.25) {mitigation};
\node[asLegend=asAccepted] at (13.10,-9.25) {accepted risk};
\end{tikzpicture}%
}

\caption{An illustrative assurance case for a recursively constructed contract environment. The arrows connect separate module guarantees: $M=M_1\otimes M_2\otimes M_5$, with $M_5=\operatorname{ext}(M_3\otimes M_4)$. The direct branch from $M_3,M_4$ to $M_5$ combines the product conditions of \cref{thm:global-local} and the extension checks described in the text. Each coverage or fidelity claim concerns all tools of that module, including idle, for reference executions. Together with isolation, the guarantees of $M$ yield \cref{ass:transition-coverage,ass:tool-fidelity} through \cref{lem:isolation}. The colour tags record a proof under supporting conditions, responsibility for meeting a condition, measures to reduce the risk of failure, or acceptance of that risk. The assignments are illustrative and may differ between coverage and fidelity of the same module.}
\label{fig:derived-tree}
\end{figure}

The conditions used in these proofs provide a way to organize the argument for the complete contract. Write $L_k$ for the supporting conditions of module $M_k$, such as the register and issuance requirements \cref{ass:custody-register}--\cref{ass:custody-submission} for payments or the response and recording requirements \cref{ass:interface-responses}--\cref{ass:interface-recording} for bookings. Suppose they establish its local guarantees for reference executions, and check compatibility in \cref{eq:no-coupling} for those same effects. Under the remaining construction hypotheses of \cref{thm:global-local}, composition followed by isolation gives
\begin{equation}\label{eq:global-leaves}
L_1\wedge\dots\wedge L_n\ \wedge\ \text{\cref{ass:isolation}}\ \wedge\ \text{compatibility of reference effects}
\ \implies\ \text{\cref{ass:transition-coverage}}\wedge\text{\cref{eq:fidelity}}.
\end{equation}
\Cref{fig:derived-tree} illustrates how this argument can be repeated within a module. The contract environment is the product $M_1\otimes M_2\otimes M_5$, where $M_5=\operatorname{ext}(M_3\otimes M_4)$. In the figure, $\mathsf{Cov}_N$ and $\mathsf{Fid}_N$ abbreviate the coverage and fidelity requirements for every tool of a module $N$, including idle, for reference executions. We first combine the guarantees of $M_3$ and $M_4$, then check the additions that produce $M_5$. Once these checks are complete, $M_5$ supplies guarantees in the same form as the other modules and can be used in the outer product. Each product step requires the hypotheses of \cref{thm:global-local}, and \cref{ass:isolation} transfers the resulting guarantees to every agent.

The colours illustrate one possible assessment of the conditions supporting this construction. They apply separately to coverage and fidelity, while the module boxes remain white. The internal claims are proved from their dependencies, while the leaf claims identify where responsibility is assigned, risk is mitigated or risk is accepted. Coverage and fidelity of the same module may therefore receive different treatments. These tags describe how a condition is handled, and a different use of the same module may require a different assignment. In \cref{sec:responsibilities} we explain how to organize such an assessment as an assurance case, connecting the mathematical argument to evidence and responsibilities.

Reusing module arguments reduces repeated verification, although compatibility and extension checks still depend on the combination. It also helps assess a change: if a payment provider changes its access policy, we can follow the branches from that condition to the contracts whose proofs use it. The tree identifies the affected arguments. Whether the resulting system failure can also affect other modules depends on their interactions, particularly on whether compatibility still holds.

\textbf{Adequacy.} The theorem establishes coverage and fidelity. The implication $\corresponds(w,s)\wedge\clause(s)\implies\Safe(w)$ in \cref{ass:contract-adequacy} still concerns the contract as a whole. Its correspondences may be local while its protection involves several resources. With two payment providers under one budget, for instance, $a^1(w)\leq b^1(s)$ and $a^2(w)\leq b^2(s)$, together with $b^1(s)+b^2(s)\leq B$, entail $a^1(w)+a^2(w)\leq B$. The parties must also validate that this formal protection is the one they intend. The same distinction applies to the goal and \cref{eq:goal-adequacy}.

\begin{figure}[t]
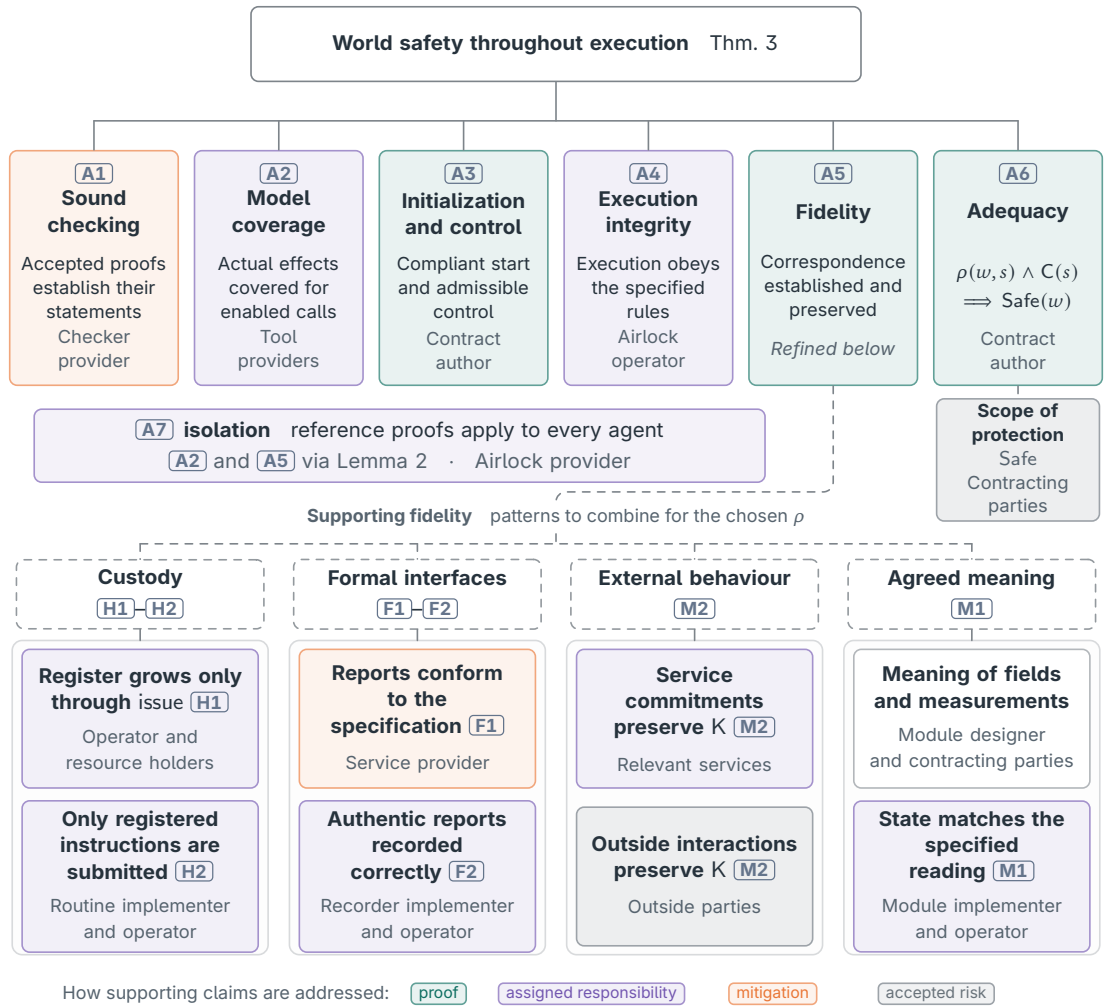

\centering
\scalebox{.94}{\assuranceTree}
\caption{An illustrative assurance case connecting assumptions, supporting conditions and responsibilities. The first level gives the joint premises of \cref{thm:world-final}. Isolation connects reference proofs to arbitrary agents. The four groups refine fidelity through the lettered conditions of \cref{prop:custody,prop:interface,prop:shared-model}, some of which also establish coverage. Teal denotes a proof under its supporting conditions. Grey under adequacy records acceptance of the risk that the chosen scope of protection is incomplete. The colour tags have the same meanings as in \cref{fig:derived-tree}, while their assignments and the named roles are illustrative and may be combined or changed for a particular system. Dashed branches indicate patterns to instantiate for its correspondence.}
\label{fig:assurance}\label{fig:tree}
\end{figure}

\textbf{Beyond products.} This construction uses disjoint coordinates for the resources and their states. More general forms of composition can support shared resources or guarantees between interacting components. Separation logic and resource algebras provide ways to describe shared resources with separated permissions \citep{ohearn2001local,reynolds2002separation,jung2015iris}. Interface theories and assume--guarantee contracts organize interactions through specified interfaces \citep{dealfaro2001interface,benveniste2018contracts}, while layered verification establishes a component against the specification of the layer below \citep{gu2015deep}. These approaches could support further capability constructions. Here the product and extension identify the proof obligations for the cases developed above.

\Needspace{7\baselineskip}
\subsection{From assumptions to responsibilities}\label{sec:responsibilities}\label{sec:tree}

The preceding constructions allow a contract to reuse its modules' guarantees. To apply them, we need to identify how their supporting conditions are met by the routines and services involved. For a payment module, this includes how instructions are recorded, who holds the credentials and which submissions the provider accepts. Following these dependencies identifies what to check and who can supply the evidence. Organizing the argument in this way produces an \emph{assurance case}: a structured argument explaining why a system should satisfy a specified property, supported by evidence for the conditions on which it relies \citep{kelly2004gsn,bloomfield2020assurance}.

\Cref{fig:assurance} organizes this assessment for the complete contract. Its first level gives the six assumptions of \cref{thm:world-final}, with their requirements and the parties expected to support them. The module arguments of \cref{fig:derived-tree} establish coverage and fidelity, while the other branches concern the checker, initialization and control, execution, and adequacy. Isolation connects these arguments to arbitrary agents: by \cref{lem:isolation}, the reference-execution guarantees apply regardless of the agent's computation and requests. The Airlock's provider establishes this restriction on access, as described in \cref{sec:mediation}.

The colours distinguish how the parties propose to support a condition or address the possibility that it fails:
\begin{itemize}[leftmargin=1.5em,itemsep=3pt,topsep=4pt]
\item \textbf{\textcolor{asProved}{Proof.}} A mathematical argument establishes a property from specified premises. For example, the routine can be proved to register each instruction before submitting it. Its assumptions about storage and execution remain explicit dependencies.
\item \textbf{\textcolor{asAssigned}{Assigned responsibility.}} A named party commits to maintaining a condition, such as restricting use of the payment credentials or honouring a reservation. The tree records who must support that commitment with evidence and who is responsible for failures.
\item \textbf{\textcolor{asMitigated}{Mitigation.}} Measures such as testing, independent checks or monitoring reduce the risk of a failure or its consequences. For a service response, this could include checking its fields against the request before recording it. The remaining uncertainty must still be assessed.
\item \textbf{\textcolor{asAccepted}{Accepted risk.}} The parties identify a possibility of failure or a gap in the protection and agree to proceed despite it. For a theorem's premise, this records the limit of their confidence that it holds. For the scope of $\Safe$, it records the risk that the specified protection does not cover everything they intended.
\end{itemize}
The four tags retain these meanings throughout the assessment, while their placement in \cref{fig:derived-tree,fig:assurance} illustrates possible choices. They may coexist, for example when a service provider accepts responsibility for a condition while monitoring supplies further evidence about whether it holds. In \cref{fig:assurance}, fidelity is green because it is derived from the conditions below it. The grey box under adequacy records a different decision: the parties accept that the scope of protection expressed by $\Safe$ may be incomplete, even when the adequacy implication itself has been proved. The white annotation about the meaning of fields identifies a specification choice.

The expanded part follows fidelity, whose supporting conditions often involve several parties. It groups them by the questions to address when interpreting and maintaining the correspondence. The labels \cref{ass:custody-register}--\cref{ass:custody-submission}, \cref{ass:interface-responses}--\cref{ass:interface-recording} and \cref{ass:measurement-agreement}--\cref{ass:invariant-preservation} connect these branches to the conditions in the preceding propositions. A module may combine several of these arguments.

\textbf{Custody.} We ask who can change the protected records and authorize the operations they describe. In the payment module, the routine registers each instruction before submission, and the provider accepts submissions only through the protected authority. Conditions \cref{ass:custody-register}--\cref{ass:custody-submission} express these requirements. \Cref{prop:custody} shows that they establish coverage and fidelity. Their justification combines an argument about the routine with evidence about access permissions and the provider's acceptance of instructions.

\textbf{Formal interfaces.} When a module records a service's response, we need to know what the service may report and how the routine receives and records it. Conditions \cref{ass:interface-responses}--\cref{ass:interface-recording} separate these questions. A booking service must return permitted responses, and the receiving routine must authenticate them and associate them with the correct requests. \Cref{prop:interface} then establishes coverage and fidelity. The service provider supports the first condition, while the receiving routine's implementer and operator support the second.

\textbf{External behaviour.} Some claims require a service to maintain a condition after replying, or while the module waits. A hotel must retain an unexpired reservation, and a mail service may be required to retain unacknowledged messages. The predicate $\mathsf K$ of \cref{prop:shared-model} expresses such domain invariants. Its preservation, together with agreement between the state and its measurement, establishes fidelity. Coverage also requires an account of possible state changes. These arguments may rely on service commitments, verified rules or bounds on outside interactions, as illustrated in \cref{tab:worlds}.

\textbf{Agreed meaning.} Before checking whether a field describes the world correctly, we must specify what it describes. A booking field may record that a confirmation was issued or assert that a room is still held. These choices require different correspondences. In the measurement construction, choosing $\pi$ specifies the reading, while \cref{ass:measurement-agreement} requires the state to agree with it. For an expert tool, the record can identify the procedure and invocation that produced a judgment. The parties decide what reliance on that judgment means for the protection or goal, connecting correspondence to adequacy and to \cref{eq:goal-adequacy}.

For a payment module, the resulting assurance case can connect its coverage and fidelity proofs to a checked implementation, an assessment of access permissions and evidence about the provider's acceptance of instructions. The tree records how these pieces support the overall guarantee and makes any unsupported conditions visible. The propositions supply the mathematical implications, while the evidence supports applying their premises to the routines and services that will actually be used \citep{hawkins2011new}. This includes specifications and interactions with external components, which are known sources of failures even in verified systems \citep{fonseca2017empirical}.

The work can be reused along with the module: a new contract checks that its supporting conditions still apply and that the composition and any extensions meet the requirements of \cref{sec:global-local}. It then uses the existing proofs and evidence while supplying its own clause, control and adequacy arguments. If a payment provider changes its access policy, the dependencies identify which custody arguments and contracts need review. In \cref{sec:signature} we discuss recording these responsibilities and agreeing on responses to failures. The same assessment supports completion, where the correspondence must justify the world goal when settlement begins.

\section{Mechanization}\label{sec:lean}\label{sec:implementation}

We provide a Lean~4 mechanization of PC2 \citep{demoura2021lean4}, available in the \PCRepo{}. The repository contains the formal definitions and checked proofs, identifies the corresponding results in this paper and makes their assumptions explicit. It also documents the implementation and provides instructions for reproducing the checks.

\section{Practical use and remaining risks}\label{sec:residual}

The preceding sections establish how PC2 can enforce a contract and how its guarantees depend on the routines, services and host that support execution. Using such a contract also requires the parties to understand what they are agreeing to, the agent to find proofs at an acceptable cost, and arrangements for the risks that remain assumed. These questions connect the construction to work on specification, automated reasoning and cooperation between parties with different interests. We discuss possible ways to address them, and what would need to be evaluated. Developing and validating these arrangements is outside the scope of the present work.

\subsection{Writing and understanding agreements}\label{sec:formalizers}

An enforceable agreement need not be the agreement its parties intended. A travel contract might require a flight and a hotel without requiring their dates and destinations to match. Every accepted call could then satisfy the clauses while the resulting trip is unusable. The mathematical adequacy condition in \cref{ass:contract-adequacy} connects $\clause$ to a specified $\Safe$, but deciding whether $\Safe$ captures the intended protection remains a separate task. This is closely related to the difficulties of specifying objectives for optimizing systems and of writing complete agreements \citep{manheim2018goodhart,hadfieldmenell2019incomplete}.

A possible response is to assist the parties with \emph{formalizer agents}, building on work on autoformalization \citep{wu2022autoformalization}. A formalizer would draft clauses from the parties' intentions, explain the resulting clauses in ordinary language, and search for compliant states that expose an unintended interpretation. For the travel example, it could present a trip whose bookings satisfy the individual requirements but cannot be used together. These examples give the parties concrete questions to resolve before signing. Capability modules and previously reviewed contract templates can also reduce the amount of specification that must be written afresh. Once $\rho$, $\clause$ and $\Safe$ have been specified, formal reasoning can establish their mathematical relationship. Reviewing the explanations and counterexamples serves the different purpose of checking that these are the relationships the parties want.

\textbf{Expert judgments.} Some requirements, such as whether a hotel suits a client's description, depend on the experts introduced in \cref{sec:tool-families}. A fixed expert procedure gives these requirements an agreed interpretation, but its judgments may be influenced by how evidence is presented. A hotel description could be written to obtain a favourable assessment without corresponding to a desirable hotel. The parties can test proposed experts on borderline and adversarial examples, specify which evidence they receive, and limit repeated attempts to obtain a favourable verdict. When an assessment is intended to remain binding, the contract can retain it and specify when new evidence permits reconsideration. These choices make the decision procedure more predictable and harder to manipulate. Their effectiveness, and whether the procedure captures the intended judgment, require evaluation beyond checking the proofs that use its answers.

\textbf{Negotiation.} The same review matters when an agreement is negotiated. A party may propose clauses that are formally consistent but unexpectedly favourable to itself, so the final agreement must be assessed as a whole. Choreographic and negotiation languages \citep{gopinathan2026pact,wyse2026commitment} provide ways to organize how agents reach agreements. Connecting them to PC2 would require their outputs to determine the clauses, tools and other components of the contract that will actually be enforced. Formalizers could support this process by explaining how a proposed revision changes the permitted outcomes.

\subsection{Finding proofs and using contracts in practice}\label{sec:provers}\label{sec:measure}

A request without an accepted proof cannot execute, so the ability to find proofs determines whether the agent can carry out its task. Control can still enforce stopping and refund conditions under the timing assumptions of \cref{sec:execution-duration}. There is no general procedure that decides all obligations expressible in the underlying logic, and proofs that exist may still be difficult to find before their control window expires. The benefit of the construction is that it requires proofs about proposed calls and the tool models, without requiring a proof of the agent's reasoning or its entire future behaviour. An individual obligation can nevertheless involve substantial reasoning about hidden state and possible outcomes.

Capability modules offer a way to reduce this work. A module can provide lemmas about the effects of its tools, which can be reused across calls and contracts. Common contract templates could then reduce many requests to familiar arithmetic or finite checks, while retaining the possibility of more general clauses. For example, a lemma about a payment routine's maximum charge can be reused when proving that a particular request preserves a budget. A prover can combine such lemmas with Lean's tactics and learned proof search \citep{ren2025deepseekproverv2,lin2025goedelproverv2,wang2025kimina}. It can run as an assistant to the acting agent, since the correctness guarantee depends on checking the resulting proof rather than on which system produced it. The choice of control window also matters: a proof that tolerates several control steps may remain usable despite the time required to construct and check it.

\textbf{Confidentiality during inference.} Access restrictions must also account for how the agent is run. If it sends its observations to an external inference service, that service receives the information in those requests even when every tool call is proof-gated. The parties therefore need to decide whether the provider is an intended recipient. Where it is not, possible arrangements include inference on a machine under the required access controls or within a confidential computing environment \citep{chrapek2025confidential}. The permitted recipients and the protections offered by the chosen arrangement must agree with the confidentiality requirements of the contract.

\textbf{Evaluation.} An evaluation of practical use should measure whether contracts capture the intended task, whether the required proofs can be found in time, and whether the resulting system completes useful work at an acceptable total cost. The first question calls for adversarial examples that separate intended outcomes from formally compliant ones. The second calls for measurements of proof-search and checking time, reuse of module lemmas, and requests that expire before authorization. The third includes the cost of constructing the contract, operating control and using external services, compared with other safeguards on the same tasks. Successful task completion must be reported alongside violations and rejected calls, so that rejecting everything cannot count as a successful defense. Such an evaluation would test the proposed use of formalizers and provers, rather than infer their effectiveness from the safety theorem.

\subsection{Responsibility and consequences}\label{sec:signature}\label{sec:managing}

The assurance analysis of \cref{sec:responsibilities} identifies what supports a guarantee and what remains assumed. The parties must also decide how to proceed with those assumptions. Before signing, they can record each premise that remains, who accepts responsibility for it, the evidence available, and the response agreed for a failure. For the payment ledger of \cref{app:examples}, this includes the reservation discipline of the implementation and the charge limits enforced by the payment service. An agreement to answer for a service failure gives the parties a basis for allocating its consequences. It does not replace the evidence that the service respects its limits.

Where failures can be established from records, a contract could include a remedy with a decidable trigger and an available execution routine. Reconciliation of payment instructions against service statements might, for instance, trigger a refund from reserved funds. Deposits or insurance could provide other ways to allocate losses, depending on what events they cover, what evidence establishes a claim, and whether the promised compensation is available. Designing these arrangements is a further question beyond the enforcement mechanism studied here. An automatic remedy itself needs the same coverage, fidelity and adequacy analysis as any other tool. Compensation addresses a consequence of failure, while the safety theorem concerns whether the failure can occur at all. If the available observations cannot distinguish a failure, its resolution requires additional evidence or a risk the parties explicitly accept.

\textbf{Consequences for others.} An agreement may serve its parties while harming others. Credible commitments can sustain harmful coordination or make threats more credible \citep{hammond2025multiagent}. Access to an Airlock could therefore require additional approval or disclosure for certain contracts, alongside restrictions on the resources and operations it provides. These measures raise questions about authority, confidentiality and protection of people outside the agreement. PC2 can enforce agreed requirements, but deciding which agreements should be permitted requires a broader institutional account.

\textbf{Participation and economic agency.} Enforceable commitments could allow people and smaller organizations to delegate to providers whose systems they cannot inspect, while retaining specified limits on the authority they grant. Whether that possibility translates into broad participation also depends on the cost of writing contracts and finding proofs, access to suitable capability modules, and the terms on which Airlock providers offer their services. These questions connect to economic agency and the distribution of AI's gains \citep{jacobs2026economic}, and to arguments that access must be accompanied by the infrastructure needed to benefit from AI \citep{gabriel2026globalbenefit}. Reusable modules and reviewed templates could lower some of these costs. Fair terms and protection for people outside an agreement require additional institutional arrangements.

\PCLayer{Context}
\section{Related work}\label{sec:related}

\begin{table}[t]
\caption{Comparison by feature, with representative examples under each approach. Formal guarantees depend on each construction's assumptions about observations, models and execution. A \pcYes{} means that the feature is provided, and a \pcNo{} that it is not provided by the mechanism itself. For PC2, the settlement row concerns a contract with a goal-gated settlement tool.}
\label{tab:comparison}
\centering\pcTableStyle
\begin{tabularx}{\linewidth}{@{}>{\raggedright\arraybackslash}p{3.0cm}*{4}{>{\raggedright\arraybackslash}X}@{}}
\toprule
\textbf{Question} & \textbf{Prompts and learned monitors} & \textbf{Policy\newline enforcement} & \textbf{Model-based\newline shields} & \textbf{PC2} \\
\addlinespace[3pt]
Examples & Prompt delimiters and attack detectors evaluated in AgentDojo \citep{debenedetti2024agentdojo} & CaMeL \citep{debenedetti2025camel} and FORGE \citep{palumbo2026formal} & Shield synthesis \citep{bloem2015shield} and shielded RL \citep{alshiekh2018shielding} & This work \\
\midrule
What allows a tool call to run? & The model follows instructions, or a learned monitor accepts the input or call & An author-supplied rule permits the call, given its arguments, history or data sources & An analysis of possible system behaviour identifies the action as safe & A supplied proof establishes the obligation derived from the contract and tool models \\
\addlinespace[5pt]
How are requirements written? & Instructions in ordinary language or criteria learned from examples & Executable rules on calls, past actions and information flow & Safety formulas supported by the chosen system model and analysis & State constraints from which per-call obligations are derived, even when those obligations are undecidable \\
\addlinespace[5pt]
What is established about the call's effects? & No formal guarantee covering all possible outcomes of an accepted call & Calls and recorded events satisfy the policy. Safety of their consequences needs a further argument & Every behaviour covered by the safety analysis satisfies the specified safety property & Every modelled effect preserves the contract, despite hidden state and permitted changes before execution \\
\addlinespace[5pt]
Are restrictions guaranteed even for a malicious agent? & \pcNo{} Resistance to attacks is empirical & \pcYes{} The specified policy is enforced & \pcYes{} Safety holds within the analysed model & \pcYes{} The contract holds throughout execution \\
\addlinespace[5pt]
Does each requested call carry a separately checkable proof? & \pcNo & \pcNo & \pcNo{} The shield itself may be verified & \pcYes{} The proof refers to the contract, tool and published information \\
\addlinespace[5pt]
Can proving task completion trigger the agreed settlement? & \pcNo & \pcNo & \pcNo & \pcYes{} An authorized request proving completion triggers settlement without further client approval \\
\bottomrule
\end{tabularx}
\end{table}

PC2 brings together three lines of work: restrictions on what an agent can execute, proofs supplied by an untrusted requester, and arguments connecting a formal specification to the system in which it is used. \Cref{tab:comparison} summarizes the enforcement mechanisms. We first explain how proof-carrying changes the task of establishing safety, then compare recent safeguards for agents and the verification, specification and cooperation literature on which the construction builds.

\textbf{From safety constraints to call obligations.} To use an authorization policy for state safety, its author must establish that every permitted call preserves the intended constraint. As tools interact with hidden information, external services and changing state, this can require substantial reasoning about all the situations in which the policy allows execution. A policy can be too restrictive, excluding safe calls, or leave safety gaps despite being enforced correctly.

In PC2, the author specifies what must remain true of the state, through $\clause$, together with models $\tran_\tool$ of the tools' possible effects. These determine the obligation for each requested call, accounting for hidden state and intervening control steps (\cref{def:obligation,def:delayed-obligation}). The agent must prove this obligation for the call it has chosen. No separate permission policy needs to be designed and proved to preserve $\clause$ for every call it permits. Each call's proof establishes enabledness and preservation at all compliant active states compatible with the view and all active states admitted by the control window, covering every outcome allowed by the tool model. From these individual justifications and automatic control, \cref{thm:commitment-final} establishes compliance throughout execution once the contract has started. The agent can change its strategy without needing to verify that strategy. Coverage, fidelity and adequacy connect the argument to the executing routines and world safety (\cref{sec:guarantees,sec:risks}).

\textbf{Producing proofs and checking them.} The distinction between finding a proof and checking one underlies proof-carrying code \citep{necula1997pcc}. Finding an argument may require substantial search, auxiliary lemmas and problem-specific reasoning. A supplied proof records the steps that justify the conclusion, allowing a checker to verify them against the logical rules without discovering the argument itself. In PC2, this places the work of justifying a call with the agent requesting it, while the Airlock checks the result. Enforcement therefore does not require a general algorithm that decides whether each obligation is true. Constraints and tool models can give rise to undecidable obligations without preventing the Airlock from checking proofs of particular instances.

This separation allows the reasoning used to justify calls to become more sophisticated without changing the criterion for accepting them. Routine obligations can use existing automation and reusable lemmas, while more difficult ones can require additional search or specialist provers (\cref{sec:provers}). An accepted proof establishes the required preservation condition regardless of how it was found. If the agent cannot supply one, the requested call does not execute. Task completion may then fail, while checked calls and automatic control continue to provide the safety guarantee under the stated assumptions.

\textbf{Defenses for tool-using agents.} Prompt injection \citep{perez2022ignore,greshake2023indirect} allows content encountered during a task to steer an agent away from its instructions. AgentDojo \citep{debenedetti2024agentdojo} provides an environment for evaluating such attacks and their defenses on realistic tasks. Deterministic defenses instead restrict execution independently of whether the model follows its instructions. CaMeL \citep{debenedetti2025camel} separates privileged planning from untrusted data and enforces capabilities on tool use. Progent \citep{shi2025progent} enforces programmable privilege policies, while FIDES \citep{costa2025fides} and RTBAS \citep{zhong2025rtbas} use information-flow restrictions. PC2 shares the principle that an adversarial model must remain subject to the execution mechanism. Its obligations address what the permitted call can do: the proof must cover every modelled outcome from the compliant active states compatible with the published view, including the changes allowed before execution.

\textbf{Formal policies and contracts for agents.} FORGE \citep{palumbo2026formal} enforces Datalog authorization rules over execution histories and externally supplied facts. Its enforcement argument relies on complete mediation and faithful policy-relevant observations. The guarantee concerns enforcement of the supplied policy, leaving its relationship to state safety to a separate argument.

Alignment Contracts \citep{david2026alignment} give a decidable contract language for permitted effects, budgets and disclosures. Their mechanized soundness theorem establishes that the event trace produced by a monitor satisfies the contract, independently of the agent's policy. Applying the theorem requires correct event extraction and complete mediation, which the work explicitly distinguishes from its mechanized results. It also separates formal enforcement from heuristic mitigation of attacks involving the contents of permitted messages. The guarantee concerns the events represented in the contract language, with unmodelled disclosures and payload semantics outside its scope.

When authorization rules are supplied separately from a desired state constraint, their author must establish that every permitted call preserves that constraint. A policy can be too restrictive, rejecting safe calls, or insufficiently restrictive, allowing violations despite being enforced correctly. PC2 asks the contract author to specify what must remain true of the state, through $\clause$, together with models of the tools' possible effects. From these objects, the published view and the allowed control steps, it derives the obligation for each requested call (\cref{def:obligation,def:delayed-obligation}). The proof must establish enabledness and preservation of $\clause$ at every compliant active state compatible with the view and every active state admitted by its control window. The agent can therefore choose any strategy for carrying out the task, provided its calls carry proofs of these obligations. No separate permission policy needs to be designed and justified: each accepted proof establishes the preservation condition itself. This condition is exact relative to the chosen models and the states admitted by the view and control window: every rule that safely admits calls on the basis of the view implies it (\cref{prop:wp}). Conservative models and difficulties in finding proofs can still prevent safe calls from executing (\cref{sec:coherence,sec:provers}).

Proof-carrying provides a further form of generality. The Airlock checks a supplied proof of the constructed obligation, so its enforcement does not require a decision procedure for the obligation itself. The contract's constraints and tool models can therefore give rise to undecidable obligations while retaining an independently checkable justification for every authorized call. Under the assumptions of \cref{thm:commitment-final}, these checks and the automatic control steps preserve $\clause$ throughout execution once the contract has started. Fidelity and adequacy connect this execution guarantee to world safety in \cref{sec:risks}, and \cref{sec:composition,sec:responsibilities} develop reusable module guarantees and the evidence and responsibilities supporting their application. Goal-gated settlement further allows parties to agree on a completion procedure whose authorized execution requires no subsequent client approval.

\textbf{Proofs, evidence and resource contracts.} Type-Checked Compliance \citep{rashie2026typechecked} proposes translating financial policies into Lean and checking proposed actions before invoking an API. Like PC2, it places a formal check between the proposed action and its execution. It does not develop the transition-based argument relating accepted obligations to all possible intermediate states and then to safety in the world. Proof-Carrying Agent Actions \citep{wang2026proofcarrying} uses ``proof'' for structured evidence attached to an action, including approval receipts, recorded assumptions and material for replay. Such evidence supports accountability, but is different from a checked derivation of an obligation over possible tool outcomes.

Agent Behavioral Contracts \citep{bhardwaj2026abc} develops a probabilistic account of compliance, drift and recovery, with composition results under assumptions on these processes. Resource-bounded Agent Contracts \citep{ye2026agentcontracts} organizes budgets, deadlines and delegation, and explicitly distinguishes stopping future calls from preventing an overrun within a call. PC2's execution guarantee includes the states reached within each call and during control, rather than only the conditions checked when a new call is requested. Under its stated assumptions, preservation of the contract holds at every instant for arbitrary requests and proposed proofs.

\textbf{Runtime enforcement and shields.} Execution monitoring studies which properties can be enforced by observing a program and stopping unacceptable behaviour. Schneider's characterization \citep{schneider2000enforceable} identifies safety as necessary for this form of enforcement: violations must be witnessed by an execution prefix. Shield synthesis \citep{bloem2015shield} and shielded reinforcement learning \citep{alshiekh2018shielding} construct a component that prevents unsafe actions independently of the agent's policy. Shielding has also been developed for continuous state and action spaces \citep{kim2024continuousshield}. In PC2, the Airlock checks a supplied proof rather than deciding the truth of the obligation by a general safety-analysis procedure. This permits predicates and state spaces beyond those supported by a particular synthesis algorithm. Relative to the chosen tool models, the obligation asks exactly for admissibility throughout the control window of every compliant active state with the published view: like the maximally permissive supervisor of supervisory control \citep{ramadge1987supervisory}, it restricts the agent only where the contract requires (\cref{prop:wp}). It imposes no further restriction on how the agent chooses its actions or constructs its proof. Obtaining that proof is the task discussed in \cref{sec:provers}.

\textbf{Proof-carrying code and authentication.} Proof-carrying code \citep{necula1997pcc} places proof construction with an untrusted code producer and proof checking with the consumer. Proof-carrying authentication \citep{appel1999pca} applies the same division to requests for access under a logical policy. PC2 adopts this established separation of responsibilities for tool calls. Its certificate concerns the effects of a particular requested operation, including nondeterministic outcomes and facts hidden by the view, while delayed obligations keep the certificate applicable as control runs. This makes the proof gate part of a continuing interaction with a changing system. The acting model and the prover can be replaced without changing the logical criterion for accepting a request.

\textbf{Reasoning transparency.} Inspecting an agent's reasoning can help identify deception and unexpected behaviour, but the usefulness of this evidence depends on how much of the reasoning is visible and faithfully reported \citep{shah2026transparency}. PC2 asks for a different kind of evidence: a proof that the requested call satisfies its obligation. That proof need not describe how the agent chose the call, and it can be produced by a separate prover. Its validity therefore does not depend on the transparency of the acting model. Monitoring can still help identify risks or omissions in the contract, while enforcement of its formal constraints rests on the proof and execution assumptions.

\textbf{From verification to world safety.} A proof about a mathematical description must be connected to the routines and services that execute it. Refinement mappings \citep{abadi1991refinement} provide a foundation for relating descriptions at different levels of abstraction. Verification-guided development and differential testing, as used for Cedar \citep{disselkoen2024cedar}, can establish properties of a model and test whether an implementation agrees with it. Smith \citep{smith1985limits} examines the role of these interpretations in claims about real systems, and empirical studies of verified systems locate failures in specifications and interactions with unverified components \citep{fonseca2017empirical}. PC2 makes the relevant connections explicit through coverage, fidelity and adequacy. This separates the proof that execution preserves $\clause$ from the evidence that compliant states describe safe situations in the world.

The structure follows the satisfaction argument of requirements engineering: properties of the environment together with a system specification must imply the intended requirement \citep{zave1997four}. Assume--guarantee contracts \citep{benveniste2018contracts} explain how a component's guarantees depend on its environment and can support composition. In PC2, capability modules organize such reusable guarantees around tool models and the correspondence between states and the world. Assurance cases then organize the claims and evidence supporting an application of the theorem. They distinguish the argument for safety from confidence in its premises \citep{hawkins2011new}, and examine objections that could defeat that argument \citep{bloomfield2020assurance}. \Cref{sec:responsibilities} uses this perspective to make the dependencies of PC2's guarantees visible, before \cref{sec:signature} considers responsibility for what remains assumed.

\textbf{Measurements and agreed meaning.} Blockchain oracles face a related problem when external facts are used by an executable contract \citep{eskandari2021sok}. Attestation and authenticated access can establish the source and processing of data \citep{zhang2016towncrier,zhang2020deco}, while its interpretation still depends on what that source represents. Similarly, the Airlock's record of an expert verdict establishes which procedure supplied an answer, without by itself establishing the answer's accuracy. Goodhart's law \citep{manheim2018goodhart} and incomplete-contract accounts of AI alignment \citep{hadfieldmenell2019incomplete} explain why an optimized measure or a formal agreement may miss the intended objective. These concerns do not preclude proving adequacy for specified $\rho$ and $\Safe$. They motivate examining whether those choices express the intended protection. The distinction between engineering a system to satisfy a model and assessing a model against an existing system \citep{lee2016fundamental} also informs the different conditions attached to custody, formal interfaces and external behaviour in \cref{sec:composition}.

\textbf{Guaranteed safe AI.} The guaranteed safe AI framework \citep{dalrymple2024gsai} combines a world model, a safety specification and a verifier. PC2 develops this approach for interactions mediated by tools: tool relations describe the relevant possible changes, the contract specifies constraints, and proofs are checked before requested operations run. The formalization can concentrate on the quantities needed by the contract and can use guarantees supplied by capability modules. It does not require a model of the agent's internal reasoning. The coverage and fidelity assumptions nevertheless include every relevant change during execution, including outside changes, which is why restricting access and specifying external behaviour matter.

\textbf{Agents, code and theorem proving.} ReAct \citep{yao2023react} and Toolformer \citep{schick2023toolformer} develop agents that combine reasoning with tool use, while CodeAct \citep{wang2024codeact} uses code as the action interface. PC2 retains the ability to choose actions and adapt to observations, adding an obligation to requests that affect the protected system. LeanDojo \citep{yang2023leandojo}, learned provers \citep{ren2025deepseekproverv2,lin2025goedelproverv2,wang2025kimina} and autoformalization \citep{wu2022autoformalization} supply relevant methods for constructing proofs and specifications. Their application to the obligations and agreements considered here is the practical question developed in \cref{sec:residual}.

\textbf{Commitment and cooperation.} Program equilibrium \citep{tennenholtz2004program} studies how commitments expressed through programs can support cooperation. PC2 provides a mechanism for enforcing agreed constraints on an agent's interactions and for invoking authorized settlement without a further client veto. This addresses one part of the principal--agent problem \citep{jensen1976theory}, in which a principal delegates work without being able to rely on the delegate's incentives. It also connects to the risks and opportunities of agents acting for different principals \citep{hammond2025multiagent}. Choreographic and negotiation languages \citep{gopinathan2026pact,wyse2026commitment} study how interactions and agreements are organized. PC2 concerns their subsequent enforcement, provided the agreement has been expressed as a contract with the required guarantees.

\textbf{Confidential computing.} Confidential virtual machines, including AMD SEV-SNP \citep{amd2020sevsnp}, provide mechanisms for protecting execution and attesting a running configuration to other parties. Work on confidential LLM inference \citep{chrapek2025confidential} studies their use and cost for the models involved. These mechanisms can support the Airlock's host and custody requirements and allow parties to identify the code to which they entrust an agreement. Their concrete configuration must provide the isolation and other execution properties identified in \cref{sec:establishing-assumptions,sec:execution-duration}.

\section{Conclusion}\label{sec:conclusion}

PC2 binds an AI agent to a contract it cannot break. The agent keeps its usual loop, and each of its tool calls carries a proof, checked before the call runs. Under the stated assumptions, from the moment the contract starts, the clauses hold at every instant for arbitrary requests and proposed proofs (\cref{thm:commitment-final}). Since the Airlock isolates the agent, none of these assumptions depends on it: the checker, the Airlock, the contract and the tools are assessed once, and the guarantees hold for every agent, adversarial or not (\cref{cor:every-agent}). For the standard contract, settlement requires proof of the goal, an authorized settlement request requires no further client approval, and the contract otherwise ends as agreed (\cref{thm:completion}). What remains is named. Whether the guarantees reach the world depends on the fidelity of the tools and on the adequacy of the contract, for which we identify concrete deployment conditions (\cref{sec:risks,sec:assurance}) and discuss responsibility for remaining assumptions (\cref{sec:signature}). Whether PC2 is useful depends on whether contracts can be written and proofs found at acceptable cost, which formalizer and prover agents address and a benchmark can measure (\cref{sec:residual}). PC2 guarantees safety under its assumptions, and termination by a deadline that the contract states as a clause is part of it (\cref{cor:every-agent}). Other forms of progress, such as settlement once the goal is reached, require additional conditions (\cref{sec:execution-duration}). This paper treats one acting agent; several agents acting on shared resources, and subcontracts whose obligations survive at any depth, are the subject of a second paper.

\bibliographystyle{pc2reportnat}
\PCBibliography{references}

\newpage
\appendix
\PCLayer{Appendices}
\PCAppendixSetup
\crefalias{section}{appendix}\crefalias{subsection}{subappendix}\crefalias{subsubsection}{subsubappendix}

\crefname{equation}{}{}
\Crefname{equation}{Eq.}{Eqs.}
\crefname{figure}{Figure}{Figures}
\crefname{table}{Table}{Tables}
\crefname{section}{Section}{Sections}
\crefname{theorem}{Theorem}{Theorems}
\crefname{proposition}{Proposition}{Propositions}
\crefname{fact}{Fact}{Facts}
\crefname{lemma}{Lemma}{Lemmas}
\crefname{corollary}{Corollary}{Corollaries}
\crefname{example}{Example}{Examples}
\crefname{remark}{Remark}{Remarks}
\crefname{algorithm}{Algorithm}{Algorithms}
\crefname{enumi}{}{}

\crefname{appendix}{Appendix}{Appendices}

\numberwithin{equation}{section}
\numberwithin{lemma}{section}
\numberwithin{proposition}{section}
\numberwithin{theorem}{section}
\numberwithin{corollary}{section}
\numberwithin{definition}{section}
\numberwithin{algorithm}{section}
\numberwithin{remark}{section}

\section*{Appendices}
\markboth{Appendices}{}

The appendices provide implementation details, worked examples and further developments of the construction. They can be read independently alongside the relevant sections of the main text.
\begin{itemize}[leftmargin=2em,labelsep=.6em,itemsep=3pt,topsep=4pt]
\item[\PCAppendixRef{app:wp}] \textbf{The obligation as a weakest precondition.} Why the obligation asks for exactly the safety of a call, and what this means for policy engines.
\item[\PCAppendixRef{app:proof-check-program}] \textbf{Proof checking.} The auxiliary program that checks requests and makes an accepted proof available for one tool invocation.
\item[\PCAppendixRef{sec:ctrl-design}] \textbf{Control design.} A construction that maintains compliance by idling, correcting the state or terminating the contract, illustrated with a compute budget.
\item[\PCAppendixRef{app:intervals}] \textbf{Time between operations.} How bounded gaps can be included in the operation models, with a corresponding convention for recorded observations.
\item[\PCAppendixRef{app:examples}] \textbf{Correspondence in practice.} Payment and booking examples show how state records support guarantees about safety and completion in the world.
\item[\PCAppendixRef{app:no-isolation}] \textbf{Failure without isolation.} The ways in which actions outside the proof gate can invalidate coverage, fidelity or adequacy.
\item[\PCAppendixRef{app:isolation}] \textbf{A formal account of isolation.} Conditions on concurrent steps and a proof that the agent introduces no effects beyond reference executions.
\end{itemize}

\section{The obligation as a weakest precondition}\label{app:wp}

\paragraph{Weakest preconditions} For a relation $R$ on states and a predicate $Q$, we use the following relational weakest precondition, which requires $R$ to admit some outcome and all its outcomes to satisfy $Q$ \citep{dijkstra1975guarded}:
\[
\mathrm{wp}(R,Q)(s)\;:\iff\;\bigl(\exists s'.\ R(s,s')\bigr)\wedge\bigl(\forall s'.\ R(s,s')\implies Q(s')\bigr).
\]
The outcome is chosen against the requester, and the first conjunct excludes miracles: a relation with no outcome would otherwise establish any $Q$. Admissibility is a weakest precondition, $\mathcal A_\tool(i,\cdot)=\mathrm{wp}(\tran_\tool(i,\cdot,\cdot),\clause)$. Here an outcome of $\tran_\tool$ is a state admitted during an invocation, rather than necessarily a state at which it returns. The existential condition therefore establishes enabledness in the model, not termination of $\exec_\tool$, and the universal condition requires $\clause$ throughout the call when coverage holds. In proof-carrying code, the producer proves a verification condition computed as the weakest precondition of its program for the consumer's safety policy \citep{necula1997pcc}; in PC2, the agent proves that of one call.

\paragraph{What the agent knows} The agent does not know the state, only the view $v$ on which its request is based, the contract, and that the state complies with it, which \cref{thm:commitment-final} makes true. For a call that may start after up to $n$ control steps, the states it cannot rule out are
\[
K_n(v)\;:=\;\bigl\{\,s'\in\Sigma_{\mathrm{act}}\;:\;\exists s\in\Sigma_{\mathrm{act}}.\ \obs(s)=v\wedge\clause(s)\wedge\tran^{\leqslant n}_\toolctrl(s,s')\,\bigr\},
\]
and, by \cref{def:delayed-obligation},
\begin{equation}\label{eq:obl-wp}
\Obl^{\leqslant n}_\tool(i,v)\iff\forall s'\in K_n(v).\ \mathrm{wp}\bigl(\tran_\tool(i,\cdot,\cdot),\clause\bigr)(s').
\end{equation}
The obligation states that the agent knows the call to be admissible, in the sense of the logic of knowledge, where a fact is known when it holds at every state considered possible \citep{fagin1995knowledge}.

\paragraph{The most permissive gate} A \emph{gate} specifies, from a request and the view $v$ on which it is based, whether the call may run: it is a predicate $G(\tool,i,n,v)$. A gate is \emph{safe} if every call it admits is admissible at every state from which it may start: $G(\tool,i,n,v)$ implies $\mathcal A_\tool(i,s')$ for every $s'\in K_n(v)$.

\begin{proposition}[The most permissive safe gate]\label{prop:wp}
The obligation is a safe gate, and every safe gate $G$ satisfies $G(\tool,i,n,v)\implies\Obl^{\leqslant n}_\tool(i,v)$ for all $\tool$, $i$, $n$ and $v$. In particular, if the obligations of a contract form an undecidable predicate, every decidable safe gate rejects some call whose obligation holds.
\end{proposition}

\begin{proof}
By \cref{eq:obl-wp}, $\Obl^{\leqslant n}_\tool(i,v)$ holds exactly when $\mathcal A_\tool(i,s')$ holds for every $s'\in K_n(v)$. This is what safety of the obligation requires, and what safety of $G$ yields when $G(\tool,i,n,v)$ holds. A decidable safe gate that admitted every call whose obligation holds would coincide with the obligation, which would then be decidable.
\end{proof}

The proposition is relative to the models and to the view: models that admit fewer outcomes while remaining enabled, views that show more and shorter windows all weaken the obligation, and a gate that also used earlier views could admit more calls, which a contract obtains by putting the relevant history in the state. It concerns one call, while the contract must hold for the rest of the execution. The two agree because the control step keeps $\clause$ from every compliant active state (\cref{ass:initial-control}): a call that keeps $\clause$ leaves a state from which the contract can be kept, so a safe gate needs no look-ahead. In the terms of supervisory control, $\clause$ is controllable, and the obligation is the maximally permissive supervisor, which excludes exactly the calls that fail enabledness or preservation at some state in $K_n(v)$ \citep{ramadge1987supervisory}.

\paragraph{Policy engines} A policy engine admits calls by a fixed rule that it decides itself: a decidable gate, whose author must show that it is safe. By \cref{prop:wp}, a safe rule admits only calls whose obligation holds, and if its safety is proved, each call it admits comes with a proof of its obligation: evaluate the rule on the request, and apply the proof of safety. An agent constrained by a safe policy can therefore run under PC2, which admits, in addition, every other call whose obligation is proved: the agent may follow any strategy whose calls keep the contract. A fixed rule rejects safe calls wherever it is stricter than the obligation, as it must be wherever the obligation is undecidable. In the travel contract, a rule that caps each of three nights at 600 EUR keeps the budget of 2000 EUR, and rejects a first night at 900 EUR followed by two at 500 EUR, although these keep the budget too. The exact rule is the obligation itself: a policy engine can enforce it only where it is decidable, and only after someone has derived it from the contract, while PC2 derives it and checks a proof of it.

\section{Details on the proof-checking program}\label{app:proof-check-program}

\begin{algorithm}[t]
\caption{Details on proof checking.}
\label{alg:proof-checking}
\begin{algorithmic}[1]
  \State {\bfseries shared variables} $r \gets \mathsf{none}$; $h \gets \mathsf{none}$
  \Statex
  \State {\bfseries def} $\mathsf{check}(\tool,i,k_0,n,\pi,k_{\mathsf{last}},v)$ \Comment{runs in its own process}
  \State \qquad $k \gets \mathsf{len}(v)-1$ \Comment{index of the last published view}
  \State \qquad {\bfseries if} $\tool \in {\cal T}$ \; {\bfseries and} $\; k_{\mathsf{last}} < k_0 \leq k \leq k_0 +n$
  \Statex \PCAlgorithmNote{Checks that $i:I_\tool$, so we can build $\Obl_{\tool}^{\leqslant n}(i,v[k_0])$}
  \State \qquad\qquad  {\bfseries if} $\mathsf{type\_check}(i, I_\tool)$
  \Statex \PCAlgorithmNote{Checks that $\pi$ is a valid proof for $\Obl_{\tool}^{\leqslant n}(i,v[k_0])$}
  \State \qquad\qquad\qquad {\bfseries if} $\mathsf{kernel\_check}(\pi, \Obl_{\tool}^{\leqslant n}(i,v[k_0]))$
  \State \qquad\qquad\qquad\qquad $r \gets \mathsf{correct}$; $\mathsf{return}$
  \State \qquad $r \gets \mathsf{incorrect}$
  \Statex
  \State {\bfseries def} $\mathsf{start\_proof\_check}(\tool,i,k_0,n,\pi,k_{\mathsf{last}},v)$
  \State \qquad {\bfseries if} $r = \mathsf{computing}$ {\bfseries then} $\mathsf{kill\_process}(h)$
  \State \qquad $r \gets \mathsf{computing}$ \Comment{any earlier verdict is dropped}
  \State \qquad $h \gets \mathsf{new\_process}(\mathsf{check},(\tool,i,k_0,n,\pi,k_{\mathsf{last}},v))$
  \Statex
  \State {\bfseries def} $\mathsf{is\_proof\_correct}()$
  \State \qquad $\mathsf{return} ~ r=\mathsf{correct}$
  \Statex
  \State {\bfseries def} $\mathsf{consume\_proof}()$ \Comment{an accepted proof authorizes one call}
  \State \qquad $r \gets \mathsf{none}$
\end{algorithmic}
\end{algorithm}

\Cref{alg:proof-checking} gives the routines used by \cref{alg:pc2} to check a request while control continues. The check concerns a particular tool, input and recorded view. Its correctness therefore requires both a sound checker, as in \cref{ass:sound-checking}, and the connection between the checked statement and the call required by \PCAssumptionRef[(iii)]{ass:execution-integrity}.

\textbf{Forming the obligation.}
For a request $(\tool,i,k_0,n,\pi)$, the Airlock retrieves the unchanged view $v[k_0]$ and checks that $\tool\in\mathcal T$ and $i\in I_\tool$. It then forms the statement $\Obl_\tool^{\leqslant n}(i,v[k_0])$ from the contract's agreed definitions. The submitted proof must establish this statement under the permitted axioms. Checking a statement supplied by the requester would not suffice, since it might express a different condition.

This requirement also concerns how a statement is represented. Source text acquires its meaning from the notations and definitions in scope. For example, a macro that replaces the written obligation with \texttt{True} would make \texttt{trivial} a valid proof of the resulting statement, while establishing nothing about the requested call. Agents have exploited changes to notation in a Lean submission system in this way \citep{paglieri2026swarm}. The issue is the connection between the intended statement and the one checked, a distinction discussed by \citet{pollack1998believe,wiedijk2012pollack}.

In Lean, the obligation can instead be formed directly as an expression in the kernel's language, using the fixed contract definitions and an encoding of the recorded view. That encoding must faithfully represent the view. A kernel expression does not depend on submitted notation or macros for its meaning. Lean's export format records such expressions and their dependencies, and \texttt{comparator} checks a submitted solution against an independently supplied statement and permitted axioms \citep{leanfro2025lean4export,leanfro2025comparator}. These facilities support the check, provided the statement against which the solution is compared is formed independently of untrusted source.

\textbf{From a verdict to a call.}
Each check uses a fixed request and recorded view, and its verdict belongs only to that request. In the pseudocode, $\mathsf{kill\_process}$ cancels the previous check before its replacement starts, so the cancelled check cannot subsequently publish a verdict. Replacing a request also discards any earlier verdict. The execution loop checks that the current request remains within its control window before running it, and $\mathsf{consume\_proof}$ prevents its reuse.

The invocation must use the same tool and input that appear in the checked obligation. For an input containing executable code, this also requires the runtime behaviour of that code to agree with the definitions used in the proof, a requirement on the compiler and runtime. Running the checker in a separate process allows control to continue while a proof is checked. The checker must also be prevented from altering the contract, recorded views or execution of a running routine. Process separation supports these requirements, but must be accompanied by the access restrictions needed to satisfy \cref{ass:execution-integrity,ass:isolation}.

\textbf{Independent checking.}
The same formal proof can be checked by several implementations. If acceptance requires all of them to accept the same statement under the agreed definitions and axioms, soundness of any one of them suffices for soundness of the combined check, provided their results are combined correctly. Lean provides export-based checking and, from version 4.35, an option to run several bundled checkers \citep{lean2026v435}. This offers a way to reduce reliance on one implementation of the checking rules.

Rocq also has a separately implemented checker through the MetaRocq project. Its \emph{safe checker} is written and verified in Rocq for PCUIC, a formal calculus underlying Rocq. Its correctness and completeness results concern that calculus under their stated assumptions, including strong normalization \citep{sozeau2020coq,sozeau2025checking}. Within this scope it provides another way to support \cref{ass:sound-checking}.

Using different proof assistants requires an additional connection between their statements. For example, recent work translates MiniF2F statements into Rocq, with human review of their semantic agreement \citep{viennot2025}. To use such a translation for a PC2 obligation, one also needs a proof of the translated statement and a justification that it expresses the intended obligation. Rechecking the same formal statement with another kernel does not require this translation.

\section{A possible design for the control \texorpdfstring{$\toolctrl$}{ctrl}}\label{sec:ctrl-design}

The control routine of \cref{sec:control-tool} keeps a contract compliant while the agent prepares calls and their proofs are checked. Depending on the state it reads, it waits, corrects the state or closes the contract. We now explain how these choices can establish the certificate required by \cref{ass:initial-control}, and illustrate the argument with a compute budget. The construction connects the clause to the interventions available to control: every compliant active state must leave at least one of these choices safe.

\textbf{Establishing the certificate.}
Fix the idle period $p$ and the region $C_{\mathrm{idle}}$ of \cref{sec:control-tool}. Outside $\mathit{Act}$, a compliant state belongs to $C_{\mathrm{idle}}$, so the model of one idle period admits only compliant states. In $\mathit{Act}\cap\mathit{Fix}$, the correcting routine must preserve $C$ throughout its execution and return to $C_{\mathrm{idle}}$. In $\mathit{End}$, the termination routine must preserve $C$ until the contract closes, including at closing. Each branch must also be enabled wherever it is selected. The model $\tran_\toolctrl$ covers the whole control period, from its initial reading until it is ready to hand over to the next invocation. The branch guarantees must therefore include the outside changes and costs over that whole period. Together, these conditions give the controller line of \cref{ass:initial-control}.

Initialization supplies its other role. From $\mathit{Ready}$, the initializer establishes $C$ before making the contract active, as when it sets the status to $\mathsf{ongoing}$ at the end of initialization. From the other inactive states it leaves the contract inactive. Providing an enabled model for both cases gives the initializer requirements of \cref{ass:initial-control}. Only the tests used by the routine, namely $\mathit{Ready}$, $\mathit{Act}$ and $\mathit{Fix}$, need to be decidable. The clause and $C_{\mathrm{idle}}$ can remain more general predicates, with their relationship to these tests established by the certificate.

This construction includes the simpler controllers discussed in \cref{sec:control-tool}. If outside changes preserve $C$, then $C_{\mathrm{idle}}=C$ and control can simply wait, as in the stability condition of rely--guarantee reasoning \citep{jones1983rely}. If no correction is supplied, control waits or terminates, as in the travel contract. When correction is available, returning well inside $C_{\mathrm{idle}}$ leaves room for further activity before another intervention is needed.

When a control invocation from a compliant active state returns without closing the contract, the state is either in $C_{\mathrm{idle}}$ after correction or in a state reached by one idle period from that region. This describes the states at the completion of control, as illustrated in \cref{fig:world-correct}(a). A requested tool may subsequently move the state elsewhere within $C$. The delayed obligations in panel~(b) use the full model $\tran_\toolctrl$, which must also include the intermediate states of control invocations.

\textbf{Example: a compute budget.}
A company lets an agent run cloud jobs within a budget $B$, with the contract to close by a deadline $d$. The state records the bill $b$, an upper bound $r$ on the combined charging rate of the running jobs, the time $t$ and the contract's status $\ell$. Outside changes can increase the bill at rate at most $r$ or stop jobs, but only an authorized call can add jobs or increase the rate. In this example we model readings and branch selection as instantaneous, while stopping the jobs can take up to $\delta$ units of time. The contract must reserve enough money and time for this shutdown.

To express that reserve throughout execution, the state also records a remaining shutdown allowance $h$. While jobs run normally, $h=\delta$. Once shutdown begins, $h$ counts down with elapsed time, and the model requires the jobs to have stopped by the time it reaches zero. If no jobs remain, $r=h=0$. We use the clause
\[
\clause(s)\quad:\iff\quad
b+rh\le B
\quad\wedge\quad
\bigl(\ell=\mathsf{ongoing}\ \implies\ t+h\le d\bigr).
\]
The term $rh$ reserves the maximum further charge during shutdown. Thus the clause protects the budget while retaining the means to stop, and the agent's calls cannot spend that reserve.

Control can idle for $p$ whenever
\[
b+r(p+h)\le B,
\qquad
t+p+h\le d.
\]
These decidable tests certify membership in $C_{\mathrm{idle}}$: throughout the idle period the bill can rise by at most $rp$, while enough money and time remain for shutdown. Use the complement of this tested region for $\mathit{Act}$. When only the budget test fails, there is still time to correct the state. A simple fix stops the running jobs and leaves the contract active. When the time test fails, control stops the jobs and closes the contract instead.

Both interventions preserve the clause. After $0\le u\le h$ time units during shutdown, its model gives $b'\le b+ru$, $r'\le r$, $h'\le h-u$ and $t'=t+u$, hence
\[
b'+r'h'\le b+ru+r(h-u)=b+rh\le B,
\qquad
t'+h'\le t+h\le d.
\]
At completion, $r'=h'=0$. In the correcting branch, the time test held before the intervention, so $t'+p\le t+h+p\le d$. The resulting state therefore belongs to $C_{\mathrm{idle}}$. In the terminating branch, shutdown finishes within the reserved time and the contract closes by $d$, with the bill still at most $B$. Initialization starts the contract only when the waiting tests have been established; otherwise it leaves it inactive.

The same accounting applies to a tool that launches more jobs: its obligation requires the clause, including the shutdown reserve, throughout the call. The call need not leave enough room to wait for another whole period, because control can intervene immediately afterwards. This distinguishes the protection supplied by the contract from the additional room for uninterrupted work supplied by $C_{\mathrm{idle}}$.

\section{Time between operations}\label{app:intervals}

\PCAssumptionRef[(i)]{ass:execution-integrity} takes the operations of the Airlock to be contiguous: each starts when the previous one ends. In an implementation, a routine may return before the host is ready to invoke the next one, and the system can continue to change during this interval. We account for a gap of at most $t_0$ by including it in the preceding operation and enlarging the tool models accordingly. Together with the sampling convention specified below, this lets us apply the results of \cref{sec:guarantees} to the extended operations. Their models and control certificate must cover the added intervals.

\textbf{Extended operations.}
Between the return of a routine and the invocation of the next one, no tool routine runs. The state can still change, by outside changes: time advances, prices move, mail arrives. Over a time of at most $t_0$, these are the changes that the model of $\mathsf{idle}$ admits (\cref{sec:model}): $\tran_{\mathsf{idle}}(t_0,z,s'')$ holds when outside changes can lead from $z$ to $s''$ within a time of at most $t_0$. In particular, $\tran_{\mathsf{idle}}(t_0,z,z)$ holds for every $z$, since that time may be zero.

For each tool $\tool\in\tools\cup\{\toolctrl\}$, let $\tool'$ be the corresponding extended tool, with the same input space. An operation of $\tool'$ executes the routine $\exec_\tool$ and includes the subsequent gap. Its model is
\begin{equation}\label{eq:extended-model}
\tran_{\tool'}(i,s,s'')\;:\iff\;\exists z.\ \tran_\tool(i,s,z)\wedge\tran_{\mathsf{idle}}(t_0,z,s'').
\end{equation}
The model of $\tool'$ admits every state that the model of $\tool$ admits, taking $z=s''$, and every state that outside changes can reach within $t_0$ from such a state, in particular from the state in which the routine returns. It is enabled exactly where the model of $\tool$ is. Let $\contract'$ be the contract $\contract$ with the tools $\tool'$, for $\tool\in\tools$, and the control tool $\toolctrl'$. Its state space, clauses, view and inactivity test are those of $\contract$.

An operation of $\tool'$ starts when the Airlock invokes $\exec_\tool$, as an operation of $\tool$ does, but it ends only when the Airlock invokes the next routine, or when the contract closes. It consists of the execution of the routine followed by the time between operations. The underlying tool routine is unchanged, but its extended operation includes the time spent by the host before the next invocation. If the routines run on the intervals $[\tau_1,\tau'_1],[\tau_2,\tau'_2],\dots$, with $\tau'_k\le\tau_{k+1}\le\tau'_k+t_0$, the operations of $\contract$ are these intervals, while those of $\contract'$ are $[\tau_1,\tau_2],[\tau_2,\tau_3],\dots$. The extended operations are therefore contiguous, as required by \PCAssumptionRef[(i)]{ass:execution-integrity}.

\textbf{Readings at operation boundaries.}
To use \cref{lem:window}, a recorded view must describe the state at a boundary between the operations counted by its window. We therefore use the following sampling convention for the extended operations: the reading associated with a completed operation is taken at the handoff to the next routine, and views and the inactivity test use that reading. Sampling and dispatch are treated as a single boundary event in this abstraction. This convention, together with the bound $t_0$, is a condition of the construction. It connects the readings to the extended operations as required by \PCAssumptionRef[(ii)]{ass:execution-integrity}. Together with the exclusion of intervening requested calls in \PCAssumptionRef[(iii)]{ass:execution-integrity}, this ensures that a call based on an earlier recorded view is separated from its reading only by complete operations of $\toolctrl'$, so the control-window argument applies.

\textbf{Guarantees for the extended operations.}
With this sampling convention, we apply the definitions and results of \cref{sec:model,sec:guarantees} to $\contract'$. In \cref{ass:transition-coverage}, $\effect_{\tool'}$ now describes the states reached throughout the extended operation, including the gap. The Airlock forms the obligations of $\contract'$ from the models $\tran_{\tool'}$, and their windows from $\tran_{\toolctrl'}$. Under \cref{ass:sound-checking,ass:transition-coverage,ass:initial-control,ass:execution-integrity} for $\contract'$, \cref{thm:commitment-final} gives the clauses at every instant from the start, including the time between operations, which now lies within operations of $\contract'$. Since $(\tool|_{\pre})'=\tool'|_{\pre}$, the settlement tool of a standard contract remains goal-gated, and since outside changes do not modify the label, \cref{thm:completion} applies too. For $t_0=0$, $\tran_{\mathsf{idle}}(0,\cdot,\cdot)$ is the identity, and $\contract'$ is $\contract$.

These guarantees require the assumptions to hold for $\contract'$. Coverage must account for the routine and the following gap, and the control certificate must establish initialization and compliance for the enlarged model $\tran_{\toolctrl'}$. For an idling branch of duration $p$, this requires accounting for up to $p+t_0$ of outside changes. The correcting and terminating routines must be checked against their enlarged models as well. The agent's obligations also use the enlarged models, so an accepted proof establishes compliance during the call and its following gap.

To apply the world-safety result of \cref{thm:world-final}, $\widehat{\effect}_{\tool'}$ must likewise describe the world and state throughout the extended operation, and \cref{ass:tool-fidelity} must hold over that whole interval. Preserving the correspondence during the added gap then connects the compliance guarantee to safety in the world through \cref{ass:contract-adequacy}.


\section{Two worked examples of correspondence}\label{app:examples}

This appendix works through two correspondences of the kind studied in \cref{sec:world}. In the payment ledger, the state bounds what happens in the world, as safety requires. In the booking confirmations, what the state records must be true of the world, as completion requires.

\subsection{A payment ledger}\label{ex:payment-ledger}

The structure needed is simple, and we describe it in words. The contract lets the agent pay merchants within a budget $B$, fees included, through a payment provider that issues \emph{payment instructions} with a spending limit, such as single-use virtual cards, pre-authorizations, or payment mandates with a maximum amount \citep{ap22025}.
\begin{itemize}[leftmargin=*,itemsep=1pt,topsep=2pt]
\item \emph{The state} contains a ledger kept by the Airlock: a list of entries, each with three fields, an identifier, a payee and a limit.
\item \emph{The world} contains the provider's records for the account: for each instruction it has issued, its identifier, payee and limit, and the charges posted against it.
\item \emph{The tool} $\mathsf{pay}(m,q)$ pays the merchant $m$ up to $q$. Its routine appends the entry $(j,m,q)$ to the ledger, under a fresh identifier $j$, and only then asks the provider to issue the instruction $(j,m,q)$. A retry reuses $j$, which the provider treats as an idempotency key. The model of $\mathsf{pay}$ admits the ledger unchanged or with the new entry.
\end{itemize}
For an identifier $j$, write $r_j(s)$ for the limit of its entry in the ledger, $u_j(w)$ for the limit the provider has issued under $j$, and $a_j(w)$ for the total charged under $j$, each zero if there is no such entry, instruction or charge. The clause bounds what the ledger has reserved, and safety bounds what has actually been charged:
\begin{equation}\label{eq:budget-quantities}
b(s):=\sum_j r_j(s),\qquad a(w):=\sum_j a_j(w),
\qquad \clause(s):=b(s)\leq B,\qquad
\Safe(w):=a(w)\leq B.
\end{equation}
The budget bound remains part of the clause at closing, so settlement or termination cannot discharge the obligation to respect it.
With this model, admissibility of $\mathsf{pay}(m,q)$ at a compliant state reduces to $b(s)+q\leq B$, as for the tools of \cref{sec:tool-families}. When the view reveals the reserved total, the ordinary obligation reduces to the same inequality. In the travel contract, $b(s)$ is the amount $\mathit{spent}$ in \cref{sec:travel}: it counts reservations, not receipts. The correspondence relates the provider's records to the ledger:
\begin{equation}\label{eq:budget-correspondence}
\corresponds(w,s)\quad\iff\quad
\forall j.\quad a_j(w)\leq u_j(w)\leq r_j(s).
\end{equation}
Every issued limit is covered by an entry of the ledger, and every charge stays within its limit.

Two conditions maintain this correspondence. The \emph{reservation discipline}: only the Airlock's routines ask the provider to issue instructions, each after recording its entry; entries are never removed or lowered, issued limits never change, and retries reuse their identifier. The \emph{service condition}: the provider lets the charges under an instruction, fees and retries included, reach at most its limit, and accepts no charge without an instruction. Both conditions apply at every instant, to all routines and outside interactions. Charges may arrive after a routine returns, so entries are retained to cover them.

\begin{proposition}[Preserving budget correspondence]\label{prop:budget-correspondence}
The reservation discipline and service condition preserve \cref{eq:budget-correspondence} whenever it holds initially. In particular, they imply the first line of \cref{ass:tool-fidelity}.
\end{proposition}

\begin{proof}
For an existing instruction, $u_j\leq r_j$ persists because its issued limit is fixed and its reservation never decreases. A new instruction satisfies the inequality because reservation precedes issuance. Before issuance, $u_j=0\leq r_j$. The service condition gives $a_j\leq u_j$ at every instant, so correspondence holds throughout, including between routines.
\end{proof}

Summing \cref{eq:budget-correspondence} gives $a(w)\leq b(s)$, so the clause implies $\Safe(w)$, establishing the safety implication of \cref{ass:contract-adequacy}. A new entry with limit $q$ requires $b(s)+q\leq B$ at every invocation state admitted by the view and control window. Under \cref{ass:sound-checking,ass:transition-coverage,ass:initial-control,ass:execution-integrity}, \cref{thm:commitment-final} preserves this budget clause. Together with \cref{prop:budget-correspondence}, and a start in which the empty ledger describes the account, as for a fresh account with no issued instructions or charges (\cref{ass:tool-fidelity}), this yields the world-safety guarantee. Recording expenditure only when a receipt arrives would instead leave issued instructions temporarily uncovered. The ledger reserves a quantity before it can be consumed, as the escrow method of transaction processing does for bounded quantities \citep{oneil1986escrow}, and records an action before performing it, as write-ahead logging does \citep{mohan1992aries}.

The reservation discipline concerns the routine, which can be verified, and the custody of the provider's credentials, which only the Airlock may use (\cref{sec:composition}). The service condition concerns the provider: it is a property of a service that the parties do not build, which the provider must support or the parties must accept. Where a provider allows charges to exceed a limit by a tolerance, the routine must record the limit plus that tolerance.

The guarantee can extend beyond closing. Suppose the hypotheses of \cref{thm:world-final} hold and the session closes at time $\tau_c$ with state $s_c$. If no new instructions can be issued afterwards, existing issued limits remain fixed, and the service condition continues to hold, then for every $\tau\geq\tau_c$,
\begin{equation}\label{eq:postclosing-budget}
a(w(\tau))\leq b(s_c)\leq B.
\end{equation}
The reservations at closing already cover the charges that may still arrive. This protection therefore survives without the Airlock continuing to run, provided the restrictions on issuance and charging remain effective.

\subsection{Booking confirmations}\label{app:booking}

The ledger construction controls which instructions can produce charges. A booking confirmation raises a different question: what does a service's response establish, and for how long?

\begin{example}[Booking confirmations]\label{ex:booking-confirmations}
Let $R(s)$ be a finite set of recorded confirmations $(b,d)$, where $b$ identifies a particular reservation and $d$ is its expiry time. Write $\mathsf{Issued}(w,b,d)$ when the world's history contains the agreed service's confirmation of this reservation until $d$, and $\mathsf{Held}(w,b)$ when the hotel actually holds it. With $t_{\worlds}(w)$ denoting time in the world, consider
\begin{equation}\label{eq:booking-correspondence}
\corresponds_{\mathsf{book}}(w,s)\quad\iff\quad
\forall (b,d)\in R(s).\quad
\mathsf{Issued}(w,b,d)\ \land\
\bigl(t_{\worlds}(w)<d\implies\mathsf{Held}(w,b)\bigr).
\end{equation}
The first conjunct states what the service reported. The second connects that report to a reservation still held before expiry. After expiry, the record continues to describe a past confirmation, but no longer asserts that the reservation is held.

The \emph{recording condition} requires entries to be added only from authenticated confirmations of the agreed service, with the reservation and expiry recorded correctly. Existing entries may be retained or removed, and any replacement must meet the same condition. The \emph{reservation condition} requires every confirmed reservation to be held until its stated expiry, including during other operations and between routines. This condition concerns both the service issuing the confirmation and the hotel holding the reservation. A confirmation with an expiry is a lease: a grant that holds until a stated time \citep{gray1989leases}.
\end{example}

\begin{proposition}[Preserving booking correspondence]\label{prop:booking-correspondence}
The recording and reservation conditions in \cref{ex:booking-confirmations} preserve \cref{eq:booking-correspondence} whenever it holds initially, provided the world's history retains issued confirmations. They therefore establish the booking component of the first line of \cref{ass:tool-fidelity}.
\end{proposition}

\begin{proof}
Consider any later instant and an entry $(b,d)$ still in the record. It was present initially or added from a correctly recorded confirmation. In either case, $\mathsf{Issued}(w,b,d)$ holds, since history retains that event. If $t_{\worlds}(w)<d$, the reservation condition gives $\mathsf{Held}(w,b)$. Both conjuncts therefore hold throughout execution, including between routines.
\end{proof}

Authentication supports the recording condition, but does not establish the reservation condition. It establishes where a report came from, not that the report is true, a distinction familiar from authenticated data feeds and blockchain oracles \citep{zhang2016towncrier,zhang2020deco,eskandari2021sok}. The service may have sent an authentic confirmation even though the hotel later cancels the booking. Conversely, the reservation condition need not prevent cancellation after expiry. For a goal that requires the reservation still to be held, the contract must also establish that the confirmation has not expired when settlement begins. For instance, a further correspondence claim $t_{\worlds}(w)\leq t(s)$, where $t(s)$ is a conservative time bound, makes $t(s)<d$ sufficient. This bound must cover delays before invocation: an earlier clock reading alone does not provide it. Together with \cref{eq:booking-correspondence}, it supplies the connection from a recorded confirmation to $\mathsf{Held}$ needed by \cref{cor:world-completion}.

For the travel contract, this gives the following goal in the world.

\begin{example}[Settlement of a travel contract]\label{ex:world-completion}
For the travel contract, let $\mathit{Nights}$ be the requested nights, and let $\mathsf{RoomHeld}(w,n)$ mean that a hotel holds a reservation for the traveller on night $n$. A possible world goal is
\begin{equation}\label{eq:travel-world-goal}
\goal_{\worlds}(w):=
\bigl(\forall n\in\mathit{Nights}.\ \mathsf{RoomHeld}(w,n)\bigr)
\land\mathsf{VisaGranted}(w),
\end{equation}
where the second conjunct means that the traveller's visa has been granted. Under its hypotheses, \cref{cor:world-completion} ensures that these conditions hold when settlement begins. A goal expressed through recorded confirmations supports this conclusion only if correspondence connects those records to reservations still valid at that instant.
\end{example}

\section{What goes wrong without isolation}\label{app:no-isolation}

The examples of \cref{app:examples} show that relating the contract's state to the world requires conditions on the routines and on the services they use. In the payment example, the argument relies on every issued instruction being backed by a ledger entry and on the provider respecting its spending limit. To apply that argument to a system used by an agent, we must also consider the access that the agent has to the account and the ledger. If it can issue a payment directly or erase an entry, the fact that the Airlock's payment routine follows the prescribed procedure is no longer sufficient.

We are interested in systems where this access can be restricted by design. For instance, keeping the payment credentials and ledger under the Airlock's control can ensure that instructions are issued only by its routines. The tool model $\tran_\tool$ can then describe the state changes caused by these operations without also including arbitrary modifications of the ledger by the agent. The obligation for a proposed payment checks whether the modeled changes respect the budget, while the provider's spending limits connect the recorded amounts to actual charges. \Cref{sec:composition} develops capability modules, which combine access to a resource with guarantees about the operations performed on it. This appendix examines why protecting this access matters, by following what an action outside the proof gate does to the assumptions of \cref{thm:world-final}, as illustrated in \cref{fig:mediation}.

\textbf{Ungated interactions.}
Consider an agent that can invoke a routine in two ways. It can submit a request to the Airlock, which checks the corresponding obligation before executing the call, or invoke the routine directly without providing a proof. We call the second form of access \emph{ungated}. To see how it can affect the guarantees, suppose that an invocation authorized by the Airlock starts from a compliant active state $s$ with $\corresponds(w,s)$. While it runs, an ungated action makes the world unsafe, so that $\neg\Safe(w')$. The running invocation may be control, as in \cref{fig:mediation}, or a requested tool. In either case, $\widehat{\effect}_\tool$ includes this change, since it describes everything that can happen to the world and state during the invocation.

If the state remains unchanged, it still satisfies $\clause(s)$. Under \cref{ass:contract-adequacy}, it can therefore no longer describe the unsafe world:
\begin{equation}\label{eq:unsafe-bypass}
\clause(s)\land\neg\Safe(w')
\quad\Longrightarrow\quad \neg\corresponds(w',s).
\end{equation}
This violates fidelity, \cref{ass:tool-fidelity}, because correspondence held at the start of the enabled invocation and has been lost while it runs. Detecting the change afterwards cannot restore a guarantee that was required to hold at every instant.

If instead the state changes to $s'$ and $\corresponds(w',s')$ holds, adequacy requires $\neg\clause(s')$. The question is then whether the tool model $\tran_\tool$ includes this state. Omitting it violates coverage, \cref{ass:transition-coverage}, since it is a possible effect of the running invocation. Including it means that the invocation is not admissible from $s$. For control this contradicts \cref{ass:initial-control}, and for a requested tool it contradicts the admissibility established from its accepted obligation under the assumptions of \cref{thm:commitment-final}. Finally, if we allow $\corresponds(w',s')$ to hold while keeping $\clause(s')$ true, the unsafe world is described by a compliant state and adequacy itself fails. The three branches of \cref{fig:mediation} illustrate these different ways in which an ungated action can invalidate the assumptions.

\textbf{Restricting effects.}
We could try to account for ungated access by enlarging $\tran_\tool$ to cover all the state changes it makes possible. Suppose, however, that these changes include some $s'$ with $\effect_\tool(i,s,s')$ and $\neg\clause(s')$. For any model that includes this outcome, $\mathcal A_\tool(i,s)$ is false, because admissibility requires every state described by the model to satisfy $\clause$. If this happens for every requested call throughout execution, the proof gate rejects all of them. Provided control remains admissible as required by \cref{ass:initial-control} and the other hypotheses of \cref{thm:world-final} hold, the system still remains safe. It can continue running control, but cannot execute any requested action to carry out the task. Thus covering all possible behaviour can preserve the safety guarantee while making the system unusable for its intended purpose.

A separate difficulty arises when the ungated action changes the world without changing the state. Such an action appears in $\widehat{\effect}_\tool$, even if it introduces no additional state into $\effect_\tool$. As \cref{eq:unsafe-bypass} shows, a compliant state cannot remain a faithful description of an unsafe world under adequacy. Enlarging $\tran_\tool$ cannot resolve this failure of fidelity. This is why we consider restrictions on the agent's actual access, alongside conditions on the models used by the checker. \Cref{sec:mediation} states the corresponding restriction as \cref{ass:isolation}.

\section{A formal account of isolation}\label{app:isolation}

\Cref{lem:isolation} allows coverage and fidelity established for reference executions to apply when an agent is present. The reason is that, during an invocation, the agent cannot change the protected system or introduce new possibilities for its evolution. This appendix gives a stepwise formulation of that argument. It uses the method of proving non-interference through conditions on individual steps \citep{goguen1984unwinding,rushby1992noninterference}.

We fix $\tool\in\tools\cup\{\toolctrl\}$, an input $i$ and an invocation of its routine $R=\exec_\tool(i)$. The routine runs alongside the environment, the host and the agent. The environment includes external services and the passage of time. The host includes the Airlock's loop and checker. As in \cref{sec:mediation}, reference executions suppress the agent's steps while leaving the other processes present.

\begin{definition}[Configurations and invocation prefixes]\label{def:ref-execution}
A configuration is a tuple $c=(w,s,r,u,a)$, where $w\in\worlds$, $s\in\ss$, and $r$, $u$ and $a$ are the local states of the routine, host and agent, respectively. The routine state $r$ contains its fixed input and records whether it has returned. For each actor $b\in\{R,\mathsf{env},\mathsf{host},\mathsf{ag}\}$, let $\to_b$ be its step relation on configurations. An invocation prefix is a finite sequence $c_0\to_{b_1}c_1\cdots\to_{b_m}c_m$ starting with the routine in its initial state on $i$, in which the routine has not returned at any $c_j$ with $j<m$. The protected part of a configuration is
\[
\kappa_R(w,s,r,u,a):=(w,s,r).
\]
A reference prefix is an invocation prefix containing no step of $\mathsf{ag}$. Write $c\rightsquigarrow d$ for a finite sequence, possibly empty, of steps of $R$, $\mathsf{env}$ and $\mathsf{host}$ from $c$ to $d$ within the same invocation, with no step after the routine has returned.
\end{definition}

A prefix may end before the routine returns, or at its return. This is necessary because the effect relations describe every state reached during execution. When the step model represents the executions of \cref{sec:world}, each world--state pair reached at any instant must be represented by a finite invocation prefix. In particular, a step that advances time must allow prefixes ending at intermediate instants if the world or state changes during that step. Sampling only the entry and return of a routine would not describe the effects used in the main results.

For an agent $\alpha$, whose processes determine $\to_{\mathsf{ag}}$, let $\widehat{\effect}^{\alpha}_\tool(i,w,s,w',s')$ hold when an invocation prefix starting with world $w$ and state $s$ ends with world $w'$ and state $s'$. Define $\widehat{\effect}^{\bullet}_\tool$ in the same way using reference prefixes. The state effects are their projections, as in \cref{eq:world-projection}. The initial configurations and the step relations of the other actors are the same in both cases.

\textbf{Conditions on steps.}
The following conditions give a sufficient stepwise formulation of isolation for this model. They are required for each tool and input, for configurations of the same invocation and steps taken before the routine returns.
\begin{enumerate}[label=(NI\arabic*),leftmargin=2.8em,itemsep=2pt,topsep=2pt]
\item \emph{Preservation by the agent.} If $c\to_{\mathsf{ag}}c'$, then $\kappa_R(c')=\kappa_R(c)$.
\item \emph{Independence of the protected evolution.} If $c\to_b c'$ with $b\ne\mathsf{ag}$ and $\kappa_R(d)=\kappa_R(c)$, then there is a configuration $d'$ such that $d\rightsquigarrow d'$ and $\kappa_R(d')=\kappa_R(c')$.
\end{enumerate}

(NI1) expresses the restriction on the agent's access in \cref{ass:isolation}: its computation and submission of requests leave the running routine, world and state unchanged. (NI2) expresses the independence needed to remove that computation. When two configurations agree on the protected quantities, a change produced by the other actors from one must also be possible from the other without an agent step. The matching sequence may be empty when the original step leaves the protected quantities unchanged. Thus, for example, a checker's private progress need not be reproduced step for step.

These conditions must be established for the interfaces and execution environment that enforce isolation. They do not require knowing how the agent chooses its requests or proofs. Host and service actions remain in reference executions, and their effects must still be covered by the tool models. Execution integrity, \cref{ass:execution-integrity}, continues to govern the scheduling of invocations and the connection between requests, views and calls.

\begin{lemma}[Isolation, stepwise formulation]\label{lem:ref-isolation}
For any agent $\alpha$ satisfying (NI1) and (NI2), every invocation prefix has a reference prefix starting from the same configuration and ending in a configuration with the same protected part. Consequently, for all inputs, worlds and states,
\[
\widehat{\effect}^{\alpha}_\tool(i,w,s,w',s')
\implies\widehat{\effect}^{\bullet}_\tool(i,w,s,w',s'),
\qquad
\effect^{\alpha}_\tool(i,s,s')
\implies\effect^{\bullet}_\tool(i,s,s').
\]
\end{lemma}

\begin{proof}
Consider an invocation prefix $c_0\to_{b_1}c_1\cdots\to_{b_m}c_m$. We construct a reference prefix starting at $c_0$ whose endpoint $d_j$ satisfies $\kappa_R(d_j)=\kappa_R(c_j)$, by induction on $j$.

For $j=0$, the empty prefix suffices. Suppose the prefix has been constructed up to $j<m$. If $b_{j+1}=\mathsf{ag}$, (NI1) gives $\kappa_R(c_{j+1})=\kappa_R(c_j)$, so the same reference prefix serves, with $d_{j+1}=d_j$. Otherwise, (NI2) gives a sequence $d_j\rightsquigarrow d_{j+1}$ with $\kappa_R(d_{j+1})=\kappa_R(c_{j+1})$, which we append. Since $j<m$, the routine has not returned at $c_j$ or at its matching configuration $d_j$. The appended sequence takes no step after return, so the concatenation is a reference prefix.

At $j=m$, the constructed prefix starts at the same invocation and has the same final world and state. It therefore witnesses the implication for $\widehat{\effect}$. Projection on states gives the implication for $\effect$.
\end{proof}

Thus coverage and both parts of fidelity for reference executions also hold for actual executions, as required by \cref{lem:isolation}.

\end{document}